\documentclass[11pt]{article}
\usepackage[utf8x]{inputenc}
\usepackage[a4paper, left=3cm, right=3cm, top=3cm, bottom=3cm]{geometry}%

\usepackage{xr}
\usepackage{times}
\usepackage{dsfont}
\usepackage{amsmath,amsfonts,amssymb}
\usepackage{amsthm}
\usepackage{mathtools}
\usepackage{color}
\usepackage{natbib}
\RequirePackage{fix-cm}

\usepackage{mathrsfs}
\usepackage{relsize}
\usepackage{algorithmic}
\usepackage{algorithm}
\usepackage{appendix}

\usepackage{booktabs}

\usepackage{tikz} % Required for flow chart
\usetikzlibrary{shapes,arrows,calc,decorations.pathreplacing}
\usetikzlibrary{arrows.meta}

\definecolor{myred}{rgb}{0.77, 0.0, 0.1}
\definecolor{newgreen}{RGB}{0,153,0}
\definecolor{myturq}{rgb}{0.1, 0.7, 0.7}

\usepackage{hyperref}

\hypersetup{
    colorlinks,
    linkcolor=myred, % myred, red, blue ; or magenta or blue
    citecolor=blue, % blue or newgreen
    linktoc=all % default : section, alternatives : all, page, none
  }

\renewcommand{\leq}{\leqslant}
\renewcommand{\geq}{\geqslant}

\newcommand{\wt}{\widetilde}
\newcommand{\wh}{\widehat}
\newcommand{\pp}{\; : \; }

\newcommand{\N}{\mathbf N}

\newcommand{\R}{\mathbf R}

\newcommand{\F}{\mathscr F}

\newcommand{\Y}{\mathcal Y}

\newcommand{\G}{\mathscr{G}}

\usepackage{xspace}
\newcommand{\ie}{\textit{i.e.}\@\xspace} 

\newcommand{\iid}{i.i.d.\@\xspace}

\newcommand{\E}{\mathbb E}
\renewcommand{\P}{\mathbb P}

\newcommand{\cond}{\,|\,}

\newcommand{\indic}[1]{\mathbf{1}( #1 )}
\newcommand{\probas}{\mathcal{P}}
\newcommand{\expdist}{\mathsf{Exp}}

\newcommand{\dirichletdist}{\mathsf{Dir}}
\newcommand{\uniformdist}{\mathcal{U}}%{\mathsf{Unif}}
\newcommand{\prior}{\pi}

\newcommand{\cumloss}{L}
\newcommand{\risk}{R}

\newcommand{\predspace}{\widehat{\Y}}
\newcommand{\pred}{\widehat{y}}
\newcommand{\predrule}{f}

\newcommand{\Experts}{\mathcal{E}}

\newcommand{\dataset}{\mathscr D}

\newcommand{\MP}{\mathop{\mathsf{MP}}} % Mondrian process
\newcommand{\splits}{\Sigma} % splits of the tree T
\newcommand{\asplit}{\sigma} % a split 
\newcommand{\node}{\mathbf{v}} %\eta
\newcommand{\othernode}{\mathbf{w}}
\newcommand{\nodes}{\mathcal N} % all nodes of T
\newcommand{\inodes}{\mathcal N^{\circ}} % interior nodes of T
\newcommand{\leaves}{\mathcal L} % leaves of T

\newcommand{\leaf}{\node}%{\mathbf{l}} %\phi

\newcommand{\cell}{C}
\renewcommand{\root}{\epsilon} % root of a tree
\newcommand{\cellrange}{R}%{\cell^X}
\newcommand{\birth}{\tau} % birth time
\newcommand{\globaltree}{\mathbf{T}} % global tree
\newcommand{\tree}{\mathcal{T}} % tree or subtree
\newcommand{\mondrian}{\Pi} % Mondrian realization
\newcommand{\parent}[1]{\mathtt{parent} (#1)} % parent of a node
\newcommand{\prednode}{\pred}
\newcommand{\pathpoint}{\mathtt{path}} % path of a point in a tree
\newcommand{\wbar}{\overline{w}} % averaged weights in CTW

\newcommand{\wpred}{\widehat{w}}

\newcommand{\nbtrees}{M} % number of trees in the forest
\newcommand{\idtree}{m} % index of the tree

\newcommand{\ExtendCell}{\mathtt{NodeUpdate}}

\newcommand{\MondrianUpdate}{\mathtt{AmfUpdate}}
\newcommand{\Predict}{\mathtt{AmfPredict}}
\newcommand{\MondrianForestPredict}{\mathtt{MondrianPredict}}
\newcommand{\MondrianForestUpdate}{\mathtt{MondrianUpdate}}

\newcommand{\SampleMondrian}{\mathtt{SampleMondrian}}

\usepackage{listings}

\usepackage{todonotes}

\definecolor{codegray}{rgb}{0.96,0.96,0.96}
\definecolor{codepurple}{rgb}{0.58,0,0.82}
\definecolor{codeblue}{rgb}{0.112,0.148,0.195}
\definecolor{codered}{rgb}{0.67,0.29,0.15}
\definecolor{codegreen}{rgb}{0.42,0.52,0.33}

\definecolor{lightblue}{rgb}{0.145,0.6666,1} % Defines the color used for content box 

\newcommand\pythonstyle{\lstset{
  language=Python,
  basicstyle=\ttfamily\footnotesize,
  breaklines=true,
  frame=lines,
  backgroundcolor=\color{codegray},
  showstringspaces=false
}}

\lstnewenvironment{python}[1][]{
\pythonstyle
\lstset{#1}}{}

\newtheorem{proposition}{Proposition}%[chapter]
\newtheorem{theorem}{Theorem}
\newtheorem{lemma}{Lemma}
\newtheorem{corollary}{Corollary}
\newtheorem{definition}{Definition}
\newtheorem{remark}{Remark}

\title{AMF: Aggregated Mondrian Forests for Online Learning}

\author{
  Jaouad Mourtada%
  \thanks{CMAP, UMR 7641, \'Ecole Polytechnique CNRS, Paris, France}\\ 
  \and 
  St\'ephane Ga\"iffas%
  \thanks{LPSM, UMR 8001, Universit\'e de Paris, Paris, France}
  \thanks{DMA, CNRS UMR 8553, Ecole normale sup\'erieure, Paris, France}
  \and
  Erwan Scornet%
  \footnotemark[1]
}

\begin{document}

\maketitle

\begin{abstract}
  Random Forests (RF) is one of the algorithms of choice in many supervised learning applications, be it classification or regression.
  The appeal of such tree-ensemble methods comes from a combination of several characteristics: a remarkable accuracy in a variety of tasks, a small number of parameters to tune, robustness with respect to features scaling, a reasonable computational cost for training and prediction, and their suitability in high-dimensional settings.
  The most commonly used RF variants however are ``offline'' algorithms, which require the availability of the whole dataset at once.
  In this paper, we introduce AMF, an online random forest algorithm based on Mondrian Forests. 
  Using a variant of the Context Tree Weighting algorithm, we show that it is possible to efficiently perform an exact aggregation over all prunings of the trees; in particular, this enables to obtain a truly online parameter-free algorithm which is competitive with the optimal pruning of the Mondrian tree, and thus adaptive to the unknown regularity of the regression function. Numerical experiments show that AMF is competitive with respect to several strong baselines on a large number of datasets for multi-class classification.


\end{abstract}

% \keywords{Online regression trees, Online learning, Adaptive regression, Nonparametric methods.}

\section{Introduction}
\label{sec:introduction}

In this paper, we consider the \emph{online supervised learning} problem in which we assume that the dataset is not fixed in advance. 
In this scenario, we are given an \iid sequence $(x_1, y_1),$ $ (x_2, y_2), \dots$ of $[0, 1]^d \times \Y$-valued random variables that come sequentially, such that each $(x_t, y_t)$ has the same distribution as a generic pair $(x, y)$. 
Our aim is to design an \emph{online} algorithm that can be updated ``on the fly'' given new sample points, that is, at each step $t \geq 1$, a \emph{randomized prediction function}
\begin{equation*}
  \widehat \predrule_t ( \cdot, \boldsymbol \mondrian_t , \dataset_t) : [0, 1]^d \to \predspace \, ,
\end{equation*}
where $\dataset_t = \{ (x_1, y_1) , \dots , (x_t, y_t) \}$ is the dataset available at step $t$, where $\boldsymbol \mondrian_t$ is a random variable that accounts for the randomization procedure and $\predspace$ is a prediction space.
In the rest of the paper, we omit the explicit dependence in $\dataset_t$. 

This paper introduces the AMF algorithm (Aggregated Mondrian Forests), the main contributions of the paper and the main advantages of the AMF algorithm, are as follows:
\begin{itemize}
  \item AMF maintains and updates at each step (each time a new sample is available) a fixed number of decision trees in an \emph{online} fashion. The predictions of each tree is computed as a weighted average of all the predictions given by \emph{all} the prunings of this tree. These predictions are then averaged over all trees to obtain the final prediction of AMF.
  This makes the algorithm \emph{purely online} and therefore better than repeated calls to batch methods on ever increasing samples. An open source implementation of AMF is available in the \texttt{onelearn} Python package, available on \texttt{GitHub} and as a \texttt{PyPi} repository, together with a documentation which explains, among other things, how the experiments from the paper can be reproduced%
  \footnote{The source code of \texttt{onelearn} is available at \url{https://github.com/onelearn/onelearn} and it can be easily installed by typing \texttt{pip install onelearn} in a terminal. The documentation of \texttt{onelearn} is available at \url{https://onelearn.readthedocs.io}.}.

  \item The online training of AMF and the computations involved in its predictions are \emph{exact}, in the sense that no approximation is required.
  We are able to \emph{compute exactly the posterior distribution} thanks to a particular choice of prior, combined with an adaptation of Context Tree Weighting \citep{willems1995context-basic,willems1998context-extensions,helmbold1997pruning,catoni2004statistical}, commonly used in lossless compression to aggregate all subtrees of a prespecified tree, which is both computationally efficient and theoretically sound.
  Our approach is, therefore, drastically
  different from Bayesian trees \citep{chipman1998bayesiancart,denison1998bayesiancart,taddy2011dynamictrees} and from BART \citep{chipman2010bart} which implement MCMC methods to approximate posterior distributions on trees. It departs also from hierarchical Bayesian smoothing involved in~\citet{lakshminarayanan2014mondrianforests} for instance, which requires also approximations.

  \item This paper provides strong theoretical guarantees for AMF, that are \emph{valid for any dimension~$d$}, and \emph{minimax optimal}, while  previous theoretical guarantees for Random Forest type of algorithms propose only suboptimal convergence rates, see for instance~\citep{wager2015adaptive,duroux2016impact}. In a batch setting, adaptive minimax rates are obtained in~\citet{mourtada2018mondrian} in arbitrary dimension for the batch Mondrian Forests algorithm. This paper provides similar results in an online setting, together with control on the online regret.
\end{itemize}

\subsection{Random Forests}
 
%We consider non-linear prediction rules $(\widehat \predrule_t)_{t \geq 1}$ that are \emph{random forests}, defined as the  averaging of $M \geq 1$ randomized decision trees.
We let $\widehat \predrule_t(x, \mondrian_t^{(1)}), \dots, \widehat  \predrule_t( x, \mondrian_t^{(\nbtrees)})$ be randomized tree predictors at a point $x \in [0, 1]^d$ at time $t$, associated to the \emph{random tree partitions} $(\mondrian_t^{(m)})_{1 \leq m \leq M}$  of $[0, 1]^d$, where $\mondrian_t^{(1)}, \hdots , \mondrian_t^{(M)}$ are \iid.
Setting $\boldsymbol \mondrian_t^{(M)} = (\mondrian_t^{(1)}, \dots, \mondrian_t^{(\nbtrees)})$, the \emph{random forest estimate} $\widehat \predrule_t^{(\nbtrees)}(x, \boldsymbol \mondrian_t^{(M)})$ is then defined by
\begin{equation}
\label{eq_def_RF}
\widehat \predrule_t^{(M)}(x, \boldsymbol \mondrian_t^{(M)}) = \frac 1\nbtrees \sum_{\idtree=1}^\nbtrees \widehat \predrule_t(x, \mondrian_t^{(\idtree)}),
\end{equation}
namely taking the average over all tree predictions $\widehat \predrule_t(x, \mondrian_t^{(\idtree)})$.
The online training of each tree can be done in parallel, since they are fully independent of each other, and each of them follow the exact same randomized construction.
Therefore, we describe only the construction of a single tree (and its associated random partition and prediction function) and omit from now on the dependence on $m=1, \ldots, M$.
An illustration of the decision functions of $M=10$ trees and the corresponding forest is provided in Figure~\ref{fig:forest-effect}.

\begin{figure}[htbp]
  \centering
  \includegraphics[width=0.95\textwidth]{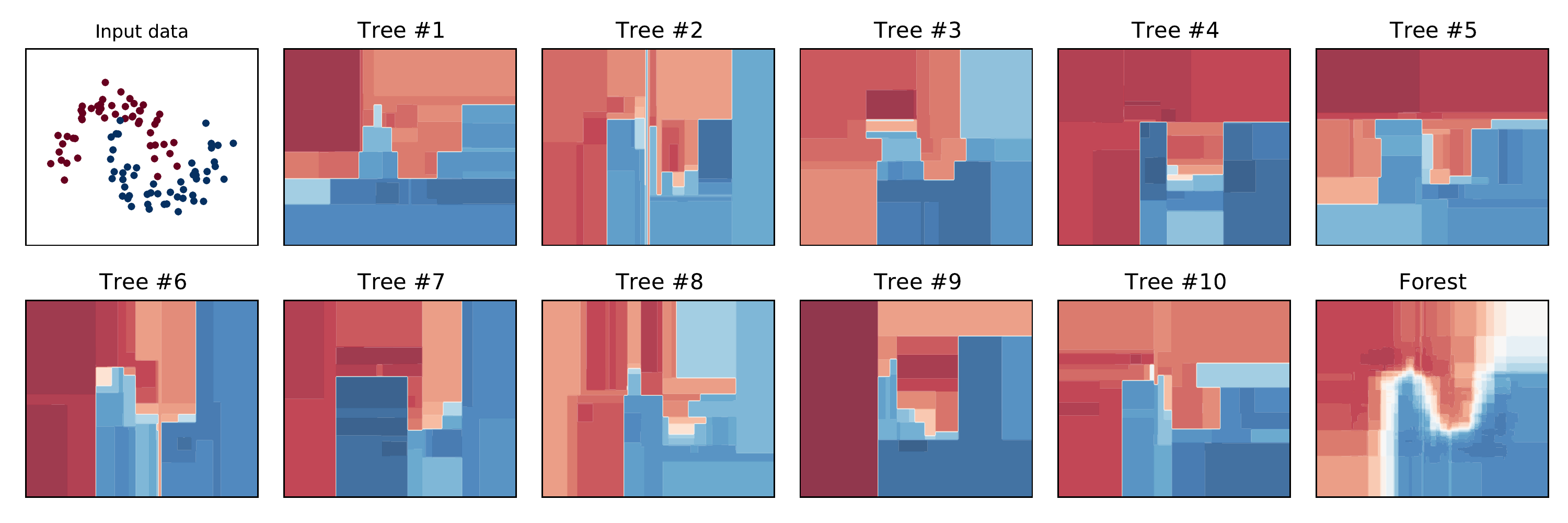}
  \caption{Decision functions of 10 trees and of the corresponding forest. Each tree is grown in parallel, following the same randomized mechanism.}
  \label{fig:forest-effect}
\end{figure}

Random tree partitions are given by $\mondrian_t = (\tree_t, \splits_t)$, where $\tree_t$ is a binary tree and $\Sigma_t$ contains information about each node in $\tree_t$ (such as splits).
These objects are introduced in Section~\ref{subsec:tree} and Section~\ref{subsec:mondrian}, in which, for simplicity, we assume that $t$ is fixed, and remove the explicit dependence on $t$.

\subsection{Random tree partitions}
\label{subsec:tree}

Let $\cell = \prod_{j=1}^d [a_j, b_j] \subseteq [0, 1]^d$ be a hyper-rectangular box.
A \emph{tree partition} (or \emph{$k$d tree}, \emph{guillotine partition}) of $\cell$ is a pair $(\tree, \splits)$, where $\tree$ is a \emph{finite ordered binary tree} and $\splits$ is a \emph{family of splits} at the interior nodes of $\tree$.

\paragraph{Finite ordered binary trees.} A finite ordered binary tree $\tree$ is represented as a finite subset of the set $\{ 0, 1 \}^* = \bigcup_{n \geq 0} \{ 0, 1 \}^n$ of all finite words on $\{ 0, 1\}$.
The set $\{ 0, 1\}^*$ is endowed with a tree structure (and called the complete binary tree): the empty word $\root$ is the root, and for any $\node \in \{ 0, 1\}^*$, the left (resp. right) child of $\node$ is $\node 0$ (resp. $\node 1$), obtained by adding a $0$ (resp. $1$) at the end of $\node$.
We denote by $\inodes (\tree) = \{ \node \in \tree : \node 0, \node 1 \in \tree \}$ the set of its \emph{interior nodes} and by $\leaves (\tree) = \{ \node \in \tree : \node 0, \node 1 \not\in \tree \}$ the set of its leaves, which are disjoint by definition. 

\paragraph{Family of splits.} Each split $\sigma_\node = (j_\node, s_\node)$ in the family $\splits = (\asplit_\node)_{\node \in \inodes (\tree)}$ of splits is characterized by its split dimension $j_\node \in \{ 1, \dots, d \}$ and its threshold $s_\node \in [0, 1]$.
%As described in Section~\ref{sub:online-mondrian} below, we will actually store in $\sigma_\node$ more information about each $\node \in \inodes (\tree)$.

\medskip
\noindent
One can associate to $(\tree, \splits)$ a partition $(\cell_{\leaf})_{\leaf \in \leaves  (\tree)}$ of $[0, 1]^d$ as follows.
For each node $\node \in \tree$, its \emph{cell} $\cell_\node$ is 
a hyper-rectangular region $\cell_\node \subseteq [0, 1]^d$ defined recursively: the cell associated to the root $\root$ of $\tree$ is $[0, 1]^d$, and, for each $\node \in \inodes(\tree)$, we define
\begin{equation}
  \label{eq:cell-split-definition}
  \cell_{\node 0} := \{ x \in \cell_\node : x_{j_\node} \leq s_{j_\node}  \} \quad \text{and} \quad \cell_{\node 1} := \cell_\node \setminus \cell_{\node 0}.
\end{equation}
Then, the leaf cells $(\cell_{\leaf})_{\leaf \in \leaves (\tree)}$ form a partition of $[0, 1]^d$ by construction.
We consider a random partition given by the Mondrian process~\citep{roy2009mondrianprocess}, following the construction of Mondrian forests~\citep{lakshminarayanan2014mondrianforests}.

  % On can therefore easily associate to $(\tree, \splits)$ a partition $(\cell_{\leaf})_{\leaf \in \leaves (\tree)}$ of $[0, 1]^d$. 
% \end{definition}
%
% \begin{definition}%[Refinement]
%  [Nested tree partitions]
%  \label{def:refinement}
%  A tree partition $\mondrian' = (C, \tree', \splits')$ of $C$ is a \emph{refinement} of the tree partition $\mondrian = (C, \tree, \splits)$ if $\tree$ is a subtree of $\tree'$ (\ie $\tree \subseteq \tree'$) and, for every $\node \in \tree$, $\asplit_\node = \asplit'_\node$.
%  A \emph{nested tree partition} of $C$ is a family $(\mondrian_\lambda)_{\lambda \geq 0}$ of tree partitions of $C$ such that, for every $\lambda, \lambda' \in \R^+$ with $\lambda \leq \lambda'$, $\mondrian_{\lambda'}$ is a refinement of $\mondrian_\lambda$.
%  Such a family can be described as follows: let $\tree$ be the (possibly infinite or complete) tree $\tree = \bigcup_{\lambda \geq 0} \tree_\lambda \subseteq \{0, 1\}^*$.
%  For each $\node \in \tree$, let $\birth_\node = \inf \{ \lambda \geq 0 \mid \node \in \tree_\lambda \} < \infty$ denote the \emph{birth time} of the node $\node$.
%  Additionally, let $\asplit_\node$ be the value of the split $\asplit_{\node, t}$ in $\mondrian_\lambda$ for $t > \birth_\node$ (which does not depend on $\lambda$ by the refinement property).
%  Then, $(\mondrian_\lambda)_{\lambda \geq 0}$ is characterized by $C$, $\tree$, $\splits = (\asplit_\node)_{\node \in T}$ and $\births = (\birth_\node)_{\node \in T}$, hence we identify it with $(C, T, \splits, \births)$.
% \end{definition}

\subsection{Mondrian random partitions}
\label{subsec:mondrian}

Mondrian random partitions are a specific family of random tree partitions.
An infinite Mondrian partition $\mondrian$ of $[0, 1]^d$ can be sampled from the infinite Mondrian process, denoted $\MP$ from now on, using the procedure $\SampleMondrian([0, 1]^d, \tau=0)$ described in Algorithm~\ref{alg:mondrian}.
If $C = \prod_{j=1}^d C^j$ with intervals $C^j = [a_j, b_j]$, we denote $|C^j| = b_j - a_j$ and $| C | = \sum_{j=1}^d |C^j|$. 
We denote by $\expdist(\lambda)$ the exponential distribution with intensity $\lambda > 0$ and by $\uniformdist ([a, b])$ the uniform distribution on a finite interval $[a, b]$.

\begin{algorithm}[htbp]
  \small
   \caption{$\SampleMondrian (\cell_\node, \tau_\node)$: sample a Mondrian starting from a cell $C_\node$ and time $\tau_\node$}
  \label{alg:mondrian} 
\begin{algorithmic}[1]
  \STATE \textbf{Inputs:} The cell $\cell_\node = \prod_{1\leq j \leq d} C_\node^j$ and
  creation time $\tau_\node$ of a node $\node$
  \STATE Sample a random variable $E \sim \expdist(|\cell_\node|)$ and put $\birth_{\node 0} = \birth_{\node 1} = \tau_\node + E$
  % \IF{$\tau + E_\cell \leq \lambda$}
  \STATE Sample a split coordinate $j_\node \in \{ 1,\dots, d \}$ with $\P (j_\node = j) 
  = |\cell_\node^j| / |\cell_\node|$
  \STATE Sample a split threshold $s_\node$ conditionally on $j_\node$ as $s_\node | j_\node \sim 
  \uniformdist (\cell_\node^{j_\node})$
  \STATE The split $(j_\node, s_\node)$ defines children cells $\cell_{\node 0}$ 
   and $C_{\node 1}$ following Equation~\eqref{eq:cell-split-definition}.
  \STATE \textbf{return} $\SampleMondrian (\cell_{\node 0}, \tau_{\node 0}) \cup \SampleMondrian (\cell_{\node 1}, \tau_{\node 1})$
\end{algorithmic}
\end{algorithm}

The call to $\SampleMondrian([0, 1]^d, \tau=0)$ corresponds to a call starting at the root node $\node = \root$, since $\cell_\root = [0, 1]^d$ and the birth time of $\root$ is $\birth_\root = 0$.
This random partition is built by iteratively splitting cells at some random time, which depends on the linear dimension $\cell_\node$ of the input cell $\cell_\node$.
The split coordinate $j_\node$ is chosen at random, with a probability of sampling $j$ which is proportional to the side length $|\cell_\node^j| / |\cell_\node|$ of the cell, and the split threshold is sampled 
uniformly in $\cell_\node^{j}$.
The number of recursions in this procedure is infinite, the Mondrian process $\MP$ is a distribution on infinite tree partitions of $[0, 1]^d$, see~\cite{roy2009mondrianprocess} and~\cite{roy2011phd} for a rigorous construction.

The \emph{birth times} $\birth_\node$ are not used in Algorithm~\ref{alg:mondrian} but will be used to define \emph{time prunings} of a Mondrian partition in Section~\ref{sub:regret-bounds} below, a notion which is necessary to prove that AMF has adaptation capabilities to the optimal time pruning.
Moreover, the birth times are used in the practical implementation of AMF 
described in Section~\ref{sec:practical_implementation}, since it is required to build \emph{restricted Mondrian partitions}, following~\cite{lakshminarayanan2014mondrianforests}.
Figure~\ref{fig:mondrian} below shows an illustration of a truncated Mondrian tree (where nodes with birth times larger than $\lambda$ have been removed), and its corresponding partition. 

\begin{figure}[h!]
	\centering    
	\begin{tikzpicture}[scale=0.4]

  %% times
  \pgfmathsetmacro{\TR}{6.5}  
  \pgfmathsetmacro{\Tl}{4.3}  
  \pgfmathsetmacro{\Tr}{2.7}  
  \pgfmathsetmacro{\Tlr}{1.4}  
  \pgfmathsetmacro{\TL}{0.6}  
  \pgfmathsetmacro{\TE}{-0.4}
  
    \coordinate (BG) at (0,0) ;
    \coordinate (BD) at (0,1) ;
    \coordinate (HG) at (1,0) ;
    \coordinate (HD) at (1,1) ;
    \coordinate (A1) at (4,0) ;
    \coordinate (B1) at (4,10) ;
    \coordinate (C1) at (4,5) ;
    \coordinate (A2) at (0,3) ;
    \coordinate (B2) at (4,3) ;
    \coordinate (C2) at (2,3) ;
    \coordinate (A3) at (4,6) ;
    \coordinate (B3) at (10,6) ;
    \coordinate (C3) at (7,6) ;
    \coordinate (A4) at (0,8) ;
    \coordinate (B4) at (4,8) ;
    \coordinate (C4) at (2,8) ;

   % \draw (0,10.5) node {\textcolor{white}{.}} ;
    
    \draw [thick] (0,0) rectangle (10,10) ; %very thick, line width=5pt
    \draw [thick,fill=lightblue!30] (0,0) rectangle (B2) ;
    \draw [thick,fill=yellow!30] (A4) rectangle (B2) ;
    \draw [thick,fill=red!30] (B1) rectangle (A4) ;
    \draw [thick,fill=green!30] (B1) rectangle (B3) ;
    \draw [thick,fill=magenta!30] (A1) rectangle (B3) ;
    
    \draw (A1) -- (B1) ; % node [midway, right] {$1.3$} ;
    \draw (A2) -- (B2) ; % node [midway, below] {$2.3$} ;
    \draw (A3) -- (B3) ;
    \draw (A4) -- (B4) ;
    \draw [dotted, fill=gray!0] (C1) circle (0.6) node {${\scriptstyle 1.3}$} ; %dashed, fill=gray!20
    \draw [dotted, fill=gray!0] (C2) circle (0.6) node {${\scriptstyle 2.3}$} ;
    \draw [dotted,fill=gray!0] (C3) circle (0.6) node {${\scriptstyle 2.7}$} ;
    \draw [dotted,fill=gray!0] (C4) circle (0.6) node {${\scriptstyle 3.2}$} ;
    % \draw (S1) node[below]{\small{$\tau_1 = \tau_2 = 1$}} ;
%    \draw (1, -2) node[below]{\small{$\tau_3 = 2$}} ;
%    \draw (S5) node[below]{\small{$\tau_4 = \tau_5 = 4$}} ;
%    \draw (2.7, -3.3) node[right]{{$\bullet$\!--- comparison expert}} ;
%    \draw[->] (0,0.5) -- (6.1, 0.5) node[above]{\small{rounds $t$}} ;
    
    \coordinate (R) at (18,\TR %9.3
    ) ;
    \coordinate (N0) at (15,\Tl %7.1
    ) ;
    \coordinate (N1) at (21,\Tr %7.1
    ) ;
    \coordinate (N00) at (13.5,\TL %4.9
    ) ;
    \coordinate (N01) at (16.5,\Tlr %4.9
    ) ;
    \coordinate (N10) at (19.5,\TL %4.9
    ) ;
    \coordinate (N11) at (22.5,\TL %4.9
    ) ;
    \coordinate (N010) at (15.5,\TL %2.7
    ) ;
    \coordinate (N011) at (17.5,\TL %2.7
    ) ;

    %% tree structure
    \draw (18,10) node {$\bullet$} -- (R) ; 
    \draw (15,\TR) -- (21, \TR) ;
    \draw (15,\TR) -- (N0) ;
    \draw (21,\TR) -- (N1) ;
    \draw (13.5,\Tl) -- (16.5, \Tl) ;
    \draw (13.5,\Tl) -- (N00) ;
    \draw (16.5,\Tl) -- (N01) ;
    \draw (15.5,\Tlr)  -- (17.5, \Tlr) ;
    \draw (15.5,\Tlr) -- (N010) ;
    \draw (17.5,\Tlr) -- (N011) ;
    \draw (19.5,\Tr) -- (22.5, \Tr) ;
    \draw (19.5,\Tr) -- (N10) ;
    \draw (22.5,\Tr) -- (N11) ;

    %% time axis   
    \coordinate (T) at (26,10) ; % t=0 (échelle: 10 = 3.7) 
    \coordinate (TR) at (26,\TR) ; % t=1.3 
    \coordinate (T0) at (26,\Tl) ; % t=2.3
    \coordinate (T1) at (26,\Tr) ; % t=2.7
    \coordinate (T01) at (26,\Tlr) ; % t=3.2
    \coordinate (TL) at (26,\TL) ;
    \coordinate (TE) at (26,\TE ) ;

    \draw[>=stealth,->] (T) -- (TE) ;
    \draw (T) node {$-$} ;
    \draw (TR) node {$-$} ;
    \draw (T0) node {$-$} ;
    \draw (T1) node {$-$} ;
    \draw (T01) node {$-$} ;
    \draw (TL) node {$-$} ;

    \draw (T) node[right] {$
      {\scriptstyle
        %\tau =
        0
      }
      $} ;
    \draw (TR) node[right] {${\scriptstyle 1.3}$} ;
    \draw (T0) node[right] {${\scriptstyle 2.3}$} ;
    \draw (T1) node[right] {${\scriptstyle 2.7}$} ;
    \draw (T01) node[right] {${\scriptstyle 3.2}$} ;
    \draw (TL) node[right] {${\scriptstyle \lambda = 3.4}$} ;
    \draw (TE) node[right] {{\small time}} ;    
    
    \draw [dotted] (18,10) -- (T) ;
    \draw [dotted] (R) -- (TR) ;
    \draw [dotted] (N0) -- (T0) ;
    \draw [dotted] (N1) -- (T1) ;
    \draw [dotted] (N01) -- (T01) ;
    \draw [dashed] (N00) -- (TL) ; % dashed ?

    %% splits on the tree
    \draw [fill=gray!0] (R) circle (0.16) ;
    \draw [fill=gray!0] (N0) circle (0.16) ;
    \draw [fill=gray!0] (N1) circle (0.16) ;
    \draw [fill=gray!0] (N01) circle (0.16) ;        

    %% leaves
    \draw [fill=lightblue!30] (N00) circle (0.25) ;
    \draw [fill=magenta!30] (N10) circle (0.25) ;
    \draw [fill=green!30] (N11) circle (0.25) ;
    \draw [fill=yellow!30] (N010) circle (0.25) ;
    \draw [fill=orange!30] (N011) circle (0.25) ;
    
    % \draw (R) node {$\circ$} ;
    % \draw [dotted,fill=gray!10] (R) circle (0.6) node {${\scriptstyle 1.3}$} ;
    % \draw [dotted,fill=gray!10] (N0) circle (0.6) node {${\scriptstyle 2.3}$} ;
    % \draw [dotted,fill=gray!10] (N1) circle (0.6) node {${\scriptstyle 2.7}$} ;
    % \draw [dotted,fill=gray!10] (N01) circle (0.6) node {${\scriptstyle 3.2}$} ;

  \end{tikzpicture}

%%% Local Variables:
%%% mode: latex
%%% TeX-master: "paper-mondrian-online"
%%% End:
	\caption{A Mondrian partition (left) with corresponding tree structure (right), which shows the evolution of the tree over time.
		The creation times $\birth_{\node}$ are indicated on the vertical axis, while the splits are denoted with bullets ($\circ$).
	}  
	\label{fig:mondrian}
\end{figure}
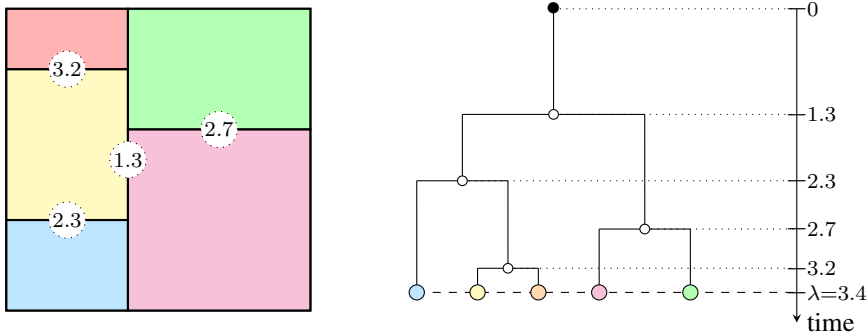

\subsection{Aggregation with exponential weights and prediction functions}
\label{sub:agg-ctw}

The prediction function of a tree in AMF is an aggregation of the predictions given by all finite subtrees of the infinite Mondrian partition $\MP$.
This aggregation step is performed in a purely online fashion, using an aggregation algorithm based on exponential weights, with a branching process prior over the subtrees.
This weighting scheme gives more importance to subtrees with a good predictive performance.

\paragraph{Node and subtree prediction.}

Let us assume that the realization of an infinite Mondrian partition 
$\Pi = (\tree^\mondrian, \Sigma^\mondrian) \sim \MP$ is available \emph{between} steps $t-1$ and $t$ (in the sense that the $t-1$-th sample has been revealed but not the $t$-th).
This partition is denoted $\Pi_t$, and is, by construction, independent of the observation $(x_t, y_t)$.
We will argue in Section~\ref{subsec:PartitionUpdate} that it suffices to store a finite partition $\Pi_t$, and show how to update it.
The definition of the prediction function used in AMF requires the notion of \emph{node} and \emph{subtree prediction}.
Given $\mondrian = (\tree^\mondrian, \Sigma^\mondrian) \sim \MP$, we define
\begin{equation*}
\pred_{\node, t} = h ( (y_s)_{1\leq s \leq t-1 \pp x_s \in \cell_\node} )
\quad \mbox{ and } \quad
L_{\node, t} = \sum_{1 \leq s \leq t \pp x_s \in \cell_\node} 
\ell (\pred_{\node, s}, y_s)
\end{equation*}
for each node $\node \in \tree^\mondrian$ (which defines a cell $\cell_\node \subseteq [0, 1]^d$ following Equation~\eqref{eq:cell-split-definition}) and each $t \geq 1$, where
 $h : \bigcup_{t \geq 0} \Y^t \to \predspace$ is a prediction algorithm used in each cell, with $\predspace$ its prediction space and $\ell : \predspace \times \Y \to \R$ a generic loss function.
The prediction between steps $t-1$ and $t$ of a finite subtree $\tree \subset \tree^\mondrian$ associated to some features vector $x \in [0, 1]^d$ is defined by
\begin{equation}
  \label{eq:expert-subtree}
 \pred_{\tree, t} (x) = \pred_{\node_{\tree} (x), t},
\end{equation}
where $\node_{\tree}(x)$ is the leaf of $\tree$ that contains $x$. 
Once again, note that the prediction $\pred_{\tree, t}$ uses samples from steps $1, \ldots, t-1$ but not $(x_t, y_t)$.
We define the cumulative loss of $\tree$ at step $t$ as
\begin{equation*}
L_{t} (\tree) = \sum_{s=1}^t \ell (\pred_{\tree, s} (x_s), y_s).
\end{equation*}
Note that the loss term $\ell (\pred_{\tree, s} (x_s), y_s)$ is evaluated at the sample $(x_s, y_s)$, which is \emph{not} used by the tree predictor $\pred_{\tree, s}$.
For regression problems, we use empirical mean forecasters
\begin{equation}
\label{eq:reg-predictor}
  \pred_{\node, t} = \frac{1}{n_{\node, t-1}} 
  \sum_{1 \leq s \leq t - 1\pp x_s \in \cell_\node} y_s,
\end{equation}
where $n_{\node, t} = | \{ 1 \leq s \leq t \pp x_s \in \cell_\node \} |$, and where we simply put $\pred_{\node, t} = 0$ if $\node$ is empty (namely, $C_\node$ contains no data point).
The loss is the \emph{quadratic loss} $\ell (\pred, y) = (\pred - y)^2$ for any $y \in \Y$ and $\pred \in \predspace$ where $\predspace = \Y = \R$.

For multi-class classification, we have labels $y_t \in \Y$ where $\Y$ is a finite set of label modalities (such as $\Y = \{ 1, \ldots, K \}$) and predictions are in $\predspace = \probas(\Y)$, the set of probability distributions on $\Y$. 
We use the \emph{Krichevsky-Trofimov} \textup(KT\textup) forecaster \textup(see \citealp{tjalkens1993sequential}\textup) in each node $\node$, which predicts
\begin{equation}
  \label{eq:kt-predictor}
  \pred_{\node, t} (y) = \frac{n_{\node, t-1} (y) + 1 / 2}{t -1 + |\Y | / 2},
\end{equation}
for any $y \in \Y$, where $n_{\node, t} (y) = | \{ 1 \leq s \leq t : x_s \in \cell_\node, y_s = y \} |$.        
For an empty $\node$, we use the uniform distribution on $\Y$. 
We consider the \emph{logarithmic loss} (also called \emph{cross-entropy} or \emph{self-information} loss) $\ell (\pred, y) = - \log \pred( y )$, where $\pred (y) = \pred (\{ y \}) \in [0, 1]$.

\begin{remark}
  The Krichevsky-Trofimov forecaster coincides with the
  %%% JAO: changement ici
  % exponential weights algorithm under the logarithmic loss \textup(with $\eta = 1$\textup) on $\probas (\Y)$ with a prior equal to the Dirichlet distribution $\dirichletdist (\frac{1}{2}, \dots, \frac{1}{2})$\textup, namely the \emph{Jeffreys prior} on the multinomial model $(\Y, \probas (\Y))$.
  Bayes predictive posterior with a prior on $\probas (\Y)$ equal to the Dirichlet distribution $\dirichletdist (\frac{1}{2}, \dots, \frac{1}{2})$\textup, namely the \emph{Jeffreys prior} on the multinomial model $\probas (\Y)$.
\end{remark}

\paragraph{The prediction function of AMF.}

Let $t\geq 1$ and $x \in [0, 1]^d$. 
The prediction function $\widehat f_t$ of AMF at step $t$ is given by
\begin{equation}
  \label{eq:exact-aggregation}
 \widehat f_t (x) = \frac{\sum_{\tree} \pi (\tree) e^{-\eta L_{t-1} (\tree)} \pred_{\tree, t} (x)}{\sum_{\tree} \pi (\tree) e^{-\eta L_{t-1} (\tree)}},
\end{equation}
where the sum is over all subtrees $\tree$ of $\tree^\mondrian$ and where the \emph{prior} $\pi$ on subtrees is the probability distribution defined by
\begin{equation}
\label{eq:ctw-prior}
\pi (\tree) = 2^{- | \tree |},
\end{equation}
where $|\tree|$ is the number of nodes in $\tree$ and $\eta > 0$ is a parameter called \emph{learning rate}.
% \end{definition}
Note that $\pi$ is the distribution of the branching process with branching probability $1 / 2$ at each node of $\tree^\mondrian$, with exactly two children when it branches; this branching process gives finite subtrees almost surely.
The learning rate $\eta$ can be optimally tuned following theoretical guarantees from Section~\ref{sec:theory}, see in particular Corollaries~\ref{cor:regret-best-pruning-log} and~\ref{cor:regret-best-pruning-square}.
This aggregation procedure is a \emph{non-greedy way to prune trees}: the weights do not depend only on the quality of one single split but rather on the performance of each subsequent split. An example of aggregated trees is provided in Figure~\ref{fig:algorithm-prediction}.

% However, the proper choice of the prior in Equation~\eqref{eq:ctw-prior} allows us to prove that $\widehat f_t$ can actually be computed very efficiently, at almost no memory cost, as stated in Proposition~\ref{prop:algorithm-implements-aggregation} below, where we prove that the AMF algorithm described in Section~\ref{sub:online-mondrian} below allows to compute  $\widehat f_t$ exactly and efficiently.

\begin{figure}[h!]
\begin{center}
\input{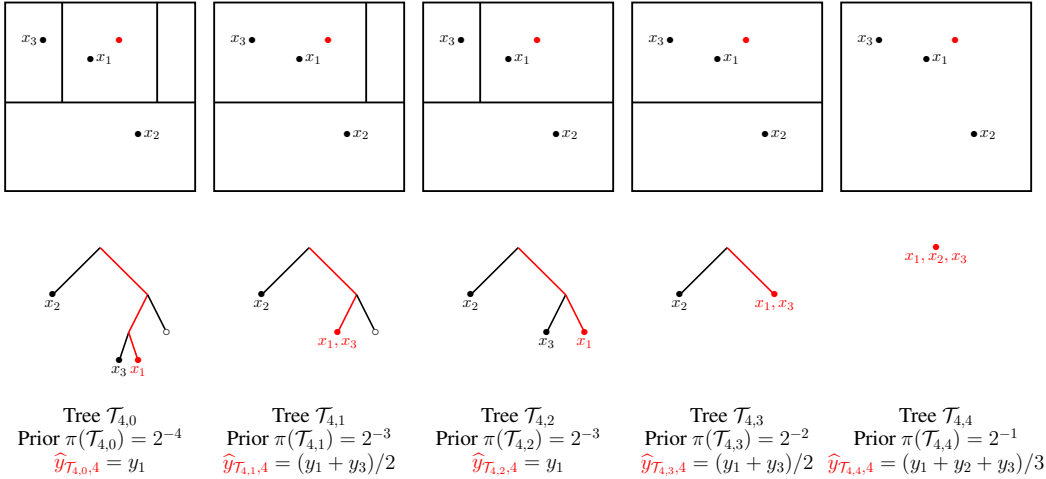}
\caption{Illustration of all subtrees involved in the prediction of a tree in AMF at the red query point between steps $t=4$ and $t=5$. The tree $\tree_{4} = \tree_{4, 0}$ is the one obtained with samples up to $t=4$ and subtrees $\tree_{4,j}$ for $j = 1, \hdots, 4$ are pruning of this tree. A tree in AMF produces a prediction given by Equation~\eqref{eq:exact-aggregation}, which is a convex combination of all five subtrees, weighted by their past performances through the aggregation weights.}
\label{fig:algorithm-prediction}  
\end{center}
\end{figure}

Let us stress that computing $\widehat f_t$ from Equation~\eqref{eq:exact-aggregation} seems computationally infeasible in practice as tree grows, since it involves a sum over all subtrees of $\tree^\mondrian$.
Besides, it requires to keep in memory one weight $e^{-\eta L_{t-1} (\tree)}$ for all subtrees $\tree$, which seems prohibitive as well.
Indeed, the number of subtrees of the minimal tree that separates $n$ points is exponential in the number of nodes, and hence \emph{exponential in $n$}.
However, it turns out that one can compute exactly and very efficiently $\widehat f_t$ using the prior choice from Equation~\eqref{eq:ctw-prior} together with an adaptation of Context Tree Weighting \citep{willems1995context-basic,willems1998context-extensions,helmbold1997pruning,catoni2004statistical}. This will be detailed in Section~\ref{sec:algorithm} below.

\subsection{Related previous works}

Introduced by \citet{breiman2001randomforests}, Random Forests (RF) is one of the algorithms of choice in many supervised learning applications.
The appeal of these methods comes from their remarkable accuracy in a variety of tasks, their reasonable computational cost at training and prediction time, and their suitability in high-dimensional settings \citep[][]{diaz2006gene,chen2012random}.

\paragraph{Online Random Forests.} % (fold)

Most commonly used RF algorithms, such as the original random forest procedure \citep{breiman2001randomforests}, extra-trees \citep{geurts2006extremely}, or conditional inference forest \citep{hothorn2010party} are batch algorithms, that require the whole dataset to be available at once. Several online random forests variants have been proposed to overcome this issue and handle data that come sequentially. 
\citet{utgoff1989incremental} was the first to extend Quinlan's ID3 batch decision tree algorithm \citep[see][]{quinlan1986induction} to an online setting. Later on, 
\citet{domingos2000hoeffdingtree} introduce Hoeffding Trees that can be easily updated:  since observations are available sequentially, a cell is split when $(i)$ enough observations have fallen into this cell, $(ii)$ the best split in the cell is statistically relevant (a generic Hoeffding inequality being used to assess the quality of the best split).
Since random forests are known to exhibit better empirical performances than individual decision trees, online random forests have been proposed \citep[see, e.g.,][]{saffari2009online-rf,denil2013online}. 
These procedures aggregate several trees by computing the mean of the tree predictions (regression setting) or the majority vote among trees (classification setting). The tree construction differs from one forest to another but share similarities with Hoeffding trees: a cell is to be split if $(i)$ and $(ii)$ are verified.

\paragraph{Mondrian Forests.} % (fold)

One forest of particular interest for this paper is the Mondrian Forest~\citep{lakshminarayanan2014mondrianforests} based on the Mondrian process~\citep{roy2009mondrianprocess}. 
Their construction differs from the construction described above since each new observation modifies the tree structure: instead of waiting for enough observations to fall into a cell in order to split it, the properties of the Mondrian process allow to update the Mondrian tree partition each time a sample is collected. 
Once a Mondrian tree is built, its prediction function uses a hierarchical prior on all subtrees and the average of predictions on all subtrees is computed with respect to this hierarchical prior using an approximation algorithm.

\paragraph{Context Tree Weighting.} % (fold)

The Context Tree Weighting algorithm has been applied to regression trees by \citet{blanchard1999progressive} in the case of a fixed-design tree, in which splits are prespecified.
This requires to split the dataset into two parts (using the first part to select the best splits and the second to compute the posterior distribution) and to have access to the whole dataset, since the tree structure needs to be fixed in advance.

\subsection{Organization of the paper} % (fold)

Section~\ref{sec:algorithm} provides a precise construction of the AMF algorithm.
A theoretical analysis of AMF is given in Section~\ref{sec:theory}, where
we establish regret bounds for AMF together with a minimax adaptive upper bound.
Section~\ref{sec:practical_implementation} introduces a modification of AMF which is used in all the numerical experiments of the paper, together with a guarantee and a discussion on its computational complexity.
Numerical experiments are provided in Section~\ref{sec:num-experiments}, in order to provide numerical insights on AMF, and to compare it with several strong baselines on many datasets.
Our conclusions are provided in Section~\ref{sec:conclusion}, while proofs are gathered in Section~\ref{sec:proofs}.

\section{AMF: Aggregated Mondrian Forests}
\label{sec:algorithm}

In an online setting, the number of sample points increases over time, allowing one to capture more details on the distribution of $y$ conditionally on $x$.
This means that the complexity of the decision trees should increase as samples are revealed as well.
We will therefore need to consider not just an individual, fixed tree partition $\mondrian$, but a sequence $(\mondrian_t)_{t \geq 1}$, indexed by ``time steps'' $t$ corresponding to the availability time of each new sample. 
Furthermore, AMF uses the aggregated prediction function from Equation~\eqref{eq:exact-aggregation}, independently within each tree $\mondrian_t^{(1)}, \ldots, \mondrian_t^{(M)}$ from the forest, see Equation~\eqref{eq_def_RF}.
When a new sample point $(x_t, y_t)$ becomes available, the algorithm does two things, in the following order:
\begin{itemize}
  \item \emph{Partition update \textup(Section~\ref{subsec:PartitionUpdate}\textup).} Using $x_t$, update the decision tree structure from 
  $\mondrian_{t} = (\tree_{t}, \Sigma_{t})$ to $\mondrian_{t+1} = (\tree_{t+1}, 
  \Sigma_{t+1})$, \ie sample new splits in order to ensure that each leaf in the tree \emph{contains at most one point} among $\{ x_1, \dots, x_t \}$. This update uses the recursive properties of Mondrian partitions;

  \item \emph{Prediction function update \textup(Section~\ref{subsec:PredictionUpdate}\textup).} Using $(x_t, y_t)$, update the prediction functions $\pred_{\node, t}$ and weights $w_{\node, t}$ and $\wbar_{\node, t}$ defined in Table~\ref{tab:notations}, that will allow us to compute $\widehat f_t$ from Equation~\eqref{eq:exact-aggregation}, thanks to a variant of Context Tree Weighting. These updates are \emph{local}, since they are performed only along the path of nodes leading to the leaf containing $x_t$, and are efficient. 
    %This update is efficient and enables
    % an efficient collapsed update enabling
    %the computation of $\widehat f_t$ from Equation~\ref{eq:exact-aggregation}, which aggregates the decision functions of all the prunings of the tree, thanks to a variant of CTW.
\end{itemize}

\begin{table}
  \caption{Notations and definitions used in AMF. The subscript $t$ used in these notations correspond to the \emph{step} or \emph{time}. For instance, $\pred_{\node, t}$ is the prediction of node $\node$ between times $t-1$ and $t$, namely when $t-1$ samples has been revealed, but not the $t$-th yet, while the cumulative loss $L_{\node, t}$ is computed once sample $(x_t, y_t)$ is revealed.}
  \label{tab:notations} 
  \fbox{
  \begin{tabular}{ll}
    \textbf{Notation or formula} & \textbf{Description} \\ \hline
    $\node \in \{ 0, 1\}^*$ & A node  \\
    $\tree \subset \{0,1\}^*$ & A tree \\
    $\node 0$~ (\textrm{resp.} ~ $\node 1$) & The left (resp. right) child of $\node$ \\
    $\tree_{\node}$ & A subtree rooted at $\node$ \\
    $\leaves (\tree)$ & The set of leaves of $\tree$ \\
    $\inodes (\tree)$ & The set of the interior nodes of $\tree$ \\
    $(\cell_{\leaf})_{\leaf \in \leaves (\tree)}$ & The cells of the partition defined by $\tree$ \\
    $\pred_{\node, t}$ & Prediction of a node $\node$ between $t-1$ and $t$ \\
    $L_{\node, t} = \sum_{1 \leq s \leq t \pp X_s \in \cell_\node} \ell (\pred_{\node, s}, y_s)$ & Cumulative loss of the node $\node$ at step $t$ \\
    $w_{\node, t} = \exp (-\eta \cumloss_{\node, t-1})$ & Weight stored in node $\node$ between $t-1$ and $t$ \\
    $\wbar_{\node, t} = \sum_{\tree_\node} 2^{-|\tree_\node |}\prod_{\node' \in \leaves(\tree_\node)} w_{\node', t}$ & Average weight stored in node $\node$  between $t-1$ and $t$ 
  \end{tabular}
}
\end{table}

Both updates can be implemented on the fly in a \emph{purely sequential manner}.
Training over a sequence $(x_1, y_1), \ldots, (x_t, y_t)$ means using each sample \emph{once} for training, and both updates are \emph{exact} and do not rely on an approximate sampling scheme.
Partition and prediction updates are precisely described in Algorithm~\ref{alg:mondrian-update} below and illustrated in Figure~\ref{fig:algorithm-unrestricted}.
The exact procedure computing the prediction after both updates have been performed is detailed in Algorithm~\ref{alg:predict} in Section~\ref{subsec:AmfPrediction}.
In order to ease the reading of this technical part of the paper, we gather in Table~1 %\ref{tab:notations} 
notations that are used in this Section.

\subsection{Partition update}
\label{subsec:PartitionUpdate}

Before seeing the point $(x_t, y_t)$, the algorithm maintains a partition $\mondrian_{t} = (\tree_t, \splits_t)$, which corresponds to the minimal subtree of the infinite Mondrian partition $\mondrian \sim \MP$ that separates all distinct sample points in $\{ x_1, \dots, x_{t-1} \}$.
This corresponds to the tree obtained from the infinite tree $\mondrian$ by removing all splits of ``empty'' cells (that do not contain any point among $\{ x_1, \dots, x_{t-1} \}$).
As $x_t$ becomes available, this tree is updated as follows (see
%%JAO: j'ai remplace pour gagner une ligne
%% this corresponds to
lines~2--11 in Algorithm~\ref{alg:mondrian-update} below):
\begin{itemize}
\item find the leaf in $\mondrian_{t}$ that contains $x_t$; it contains at most one point among $\{ x_1, \dots, x_{t-1} \}$;
\item if the leaf contains no point $x_s \neq x_t$, then let $\mondrian_{t+1} = \mondrian_{t}$. Otherwise, let $x_s$ be the unique point among $\{ x_1, \dots, x_{t-1} \}$ (distinct from $x_t$) in this cell. Splits of the cell containing $\{ x_s, x_t \}$ are successively sampled (following the recursive definition of the Mondrian distribution, see Section~\ref{subsec:mondrian}), until a split separates $x_s$ and $x_t$.
\end{itemize}

\subsection{Prediction function update}
\label{subsec:PredictionUpdate}

The algorithm maintains weights $w_{\node, t}$ and $\wbar_{\node, t}$ and predictions
 $\pred_{\node, t}$ in order to compute the aggregation over the tree structure (lines 12--18 in Algorithm \ref{alg:mondrian-update}).
 Namely, after round $t-1$ (after seeing sample $(x_{t-1}, y_{t-1}$)), each node $\node \in \tree_{t}$ has the following quantities in memory:
\begin{itemize}
\item the weight $w_{\node, t} = \exp (-\eta \cumloss_{\node, t-1})$, where
  $\cumloss_{\node, t} := \sum_{1\leq s \leq t \pp x_s \in \cell_\node } \ell (\pred_{\node, s}, y_s)$;
\item the averaged weight $\wbar_{\node, t} = \sum_{\tree_\node} 2^{-|\tree_\node |}\prod_{\node' \in \leaves(\tree_\node)} w_{\node', t}$, where the sum ranges over all subtrees $\tree_\node$ rooted at $\node$;
\item the forecast $\pred_{\node, t}$ in node $\node$ at time $t$.
\end{itemize}
Now, given a new sample point $(x_t, y_t)$, the update is performed as follows: we find the leaf $\node_t = \node_{\mondrian_{t+1}} (x_t)$ containing $x_t$ in $\mondrian_{t+1}$ (the partition has been updated with $x_t$ already, since the partition update is performed \emph{before} the prediction function update).
Then, we update the values of $w_{\node, t}, \wbar_{\node, t}, \pred_{\node, t}$ for each $\node$ along an \emph{upwards} recursion from $\node_t$ to the root, while \emph{the values of nodes outside of the path are kept unchanged}:
\begin{itemize}
\item $w_{\node, t+1} = w_{\node, t} \exp(- \eta \ell (\pred_{\node, t}, y_t) )$;
\item if $\node = \node_t$ then $\wbar_{\node, t+1} = w_{\node, t+1}$, otherwise
  \begin{equation*}
    \label{eq:update-avg-weights}
    \wbar_{\node, t+1} = \frac{1}{2} w_{\node, t+1} + \frac{1}{2} \wbar_{\node 0, t+1} \wbar_{\node 1, t+1};
  \end{equation*}  
\item $\pred_{\node, t+1} = h ( (y_s)_{1\leq s \leq t \pp x_s \in \cell_\node} )$ using the prediction algorithm $h : \bigcup_{t \geq 0} \Y^t \to \predspace$, see Section~\ref{sub:agg-ctw}. Note that the prediction algorithms given in Equation~\eqref{eq:reg-predictor} and \eqref{eq:kt-predictor} can be updated online using $y_t$ only and do not require to look back at the sequence $y_1, \ldots, y_{t-1}$.
\end{itemize}
The \emph{partition update} and \emph{prediction function update} correspond to the $\MondrianUpdate(x, y)$ procedure described in Algorithm~\ref{alg:mondrian-update} below.
\begin{algorithm}[htbp]
  \small
  \caption{$\MondrianUpdate(x, y):$ update AMF with a new sample
   $(x, y) \in [0, 1]^d \times \Y$}
  \label{alg:mondrian-update}   
  \begin{algorithmic}[1]
    \STATE \textbf{Input:} a new sample $(x, y) \in [0, 1]^d \times \Y$
    % \STATE \textbf{Parameters:} A finite tree partition $\mondrian$ of $C$, a finite set $\X \subset C$,
    % splits $\splits = (\sigma_{\node})_{\node \in T}$,
    % birth times $\births = (\birth_{\node})_{\node \in T}$,
    % ranges $\cellrange = (\cellrange_{\node})_{\node \in T}$,

    % weights $W = (w_{\node})_{\node \in \tree}$, averaged weights $\overline W = (\wbar_{\node})_{\node \in \tree}$, predictions $\wh Y = (\pred_{\node})_{\node \in \tree}$, a feature point $x \in \R^d$ and its label $y$.
    % $\R^d \setminus \X$
    \STATE Let $\node(x)$ be the leaf such that $x \in C_{\node(x)}$ and put $\node = \node(x)$
    \WHILE{$\cell_\node$ contains some $x' \neq x$}
    \STATE Use Lines 1--5 from Algorithm~\ref{alg:mondrian} to split $\cell_\node$ and obtain children cells $C_{\node 0}$ and $C_{\node 1}$
    % \STATE Sample $J \in \{ 1, \dots, d\}$ with $\P (J=j) = |C_\node^j| / |C_\node|$, and $S | J \sim \uniformdist(C_\node^J)$.
    % \STATE Split the cell $\cell_\node$ along coordinate $J$ with threshold $S$, and set $\node$ to be the child of the current node that contains $x$.
    \IF {$\{ x, x' \} \subset \cell_{\node a}$ for some $a \in \{ 0, 1\}$}
    \STATE Put $\node = \node a$, $(w_{\node a}, \wbar_{\node a}, \pred_{\node a}) = (w_{\node}, \wbar_{\node}, \pred_{\node})$ and $(w_{\node (1-a)}, \wbar_{\node (1-a)}, \pred_{\node (1-a)}) = (1, 1,  h (\varnothing))$ ($h (\varnothing)$ is the default initial prediction described in Section~\ref{sub:agg-ctw})
    \ELSE 
    \STATE Let $a \in \{ 0, 1\}$ be such that $x \in \cell_{\node a}$ and $x' \in \cell_{\node (1 - a)}$. Put $\node = \node a$ and $(w_{\node a}, \wbar_{\node a}, \pred_{\node a}) = (1, 1,  h (\varnothing))$ and $(w_{\node (1-a)}, \wbar_{\node (1-a)}, \pred_{\node (1-a)}) = (w_{\node}, \wbar_{\node}, \pred_{\node})$
    \ENDIF
    % \STATE For each newly created leaf $\node'$ in the previous line, set $(w_{\node'}, \wbar_{\node'}, \pred_{\node'})$ to the value of its parent if $x' \in \cell_{\node'}$, otherwise set  $(w_{\node'}, \wbar_{\node'}, \pred_{\node'}) = (1, 1, \pred)$ where $\pred = h (\varnothing)$ is the default initial prediction (which is described for regression in Examples~\ref{example:regression} and classification in Example~\ref{example:classification})
    \ENDWHILE
    \STATE Put $x_\node = x$ (memorize the fact that $\node$ contains $x$)
    % \STATE Add $x$ to $\X$.
    \STATE Let $\mathtt{continueUp} \gets \mathtt{true}$
    \WHILE {$\mathtt{continueUp}$}
    \STATE Set $w_\node = w_\node \exp (- \eta \ell(\pred_\node, y))$
    \STATE Set $\wbar_\node = w_\node$ if $\node$ is a leaf and $\wbar_\node = \frac{1}{2} w_\node + \frac{1}{2} \wbar_{\node 0} \wbar_{\node 1}$ otherwise
    \STATE Update $\pred_\node$ using $y$ (following Section~\ref{sub:agg-ctw})
    \STATE If $\node \neq \root$ let $\node = \parent{\node}$, otherwise let $\mathtt{continueUp} = \mathtt{false}$
    \ENDWHILE
  \end{algorithmic}
\end{algorithm}
Training AMF over a sequence $(x_1, y_1), \ldots, (x_t, y_t)$ means using successive calls to $\MondrianUpdate(x_1, y_1), \ldots, \MondrianUpdate(x_t, y_t)$.
The procedure $\MondrianUpdate$ $(x, y)$ maintains in memory the current state of the Mondrian partition $\Pi = (\tree, \Sigma)$.
The tree $\tree$ contains the parent and children relations between all nodes $\node \in \tree$, while each $\sigma_\node \in \Sigma$ can contain
\begin{equation}
  \label{eq:data_in_node_unrestricted}
  \sigma_\node = (j_\node, s_\node, \pred_{\node}, w_{\node}, \wbar_{\node}, x_\node),
\end{equation}
namely the split coordinate $j_\node \in \{ 1, \ldots, d \}$ and split threshold $s_\node \in [0, 1]$ (only if $\node \in \inodes(\tree)$), the prediction function $\pred_\node \in \predspace$, aggregation weights $w_{\node}, \wbar_{\node} \in (0, +\infty)$ and a vector $x_\node \in [0, 1]^d$ if $\node \in \leaves(\tree)$.
An illustration of Algorithm~\ref{alg:mondrian-update} is given in Figure~\ref{fig:algorithm-unrestricted} below.

\begin{figure}[h!]
  \begin{minipage}{0.37\textwidth}
  	\begin{tikzpicture}[scale=0.25, every node/.style={scale=0.55}]

  %%% partition
  
  \pgfmathsetmacro{\xmin}{0} % first square
  \pgfmathsetmacro{\xmax}{10}
  \pgfmathsetmacro{\ymin}{0}
  \pgfmathsetmacro{\ymax}{10}

  % \pgfmathsetmacro{\xtrans}{15} % translation between two squares  

  \pgfmathsetmacro{\xa}{2} % point a
  \pgfmathsetmacro{\ya}{8}
  \pgfmathsetmacro{\xb}{1.3} % point b
  \pgfmathsetmacro{\yb}{6.1}
  \pgfmathsetmacro{\xc}{4.5} % point c
  \pgfmathsetmacro{\yc}{7}
  \pgfmathsetmacro{\xd}{5} % point d
  \pgfmathsetmacro{\yd}{4}
  \pgfmathsetmacro{\xe}{7} % point e
  \pgfmathsetmacro{\ye}{3}

  \pgfmathsetmacro{\sy}{4.7} % root split (y)
  \pgfmathsetmacro{\sxl}{4.2} % split left node (x)
  \pgfmathsetmacro{\sxr}{8}
  \pgfmathsetmacro{\sylr}{1.5}
  \pgfmathsetmacro{\sxrl}{3}
  \pgfmathsetmacro{\sxlrr}{6.3}
  \pgfmathsetmacro{\syrll}{6.8}

  %% squares
  \draw [line width=0.7pt] (\xmin,\ymin) rectangle (\xmax,\ymax) ; % left square

  %% splits
  \draw [line width=0.7pt] (\xmin,\sy) -- (\xmax, \sy) ; % first split
  % \draw [red,line width=0.7pt] (\sxl,\ymin) -- (\sxl, \sy) ; % second split (only unrestricted)
  \draw [line width=0.7pt] (\sxr,\sy) -- (\sxr, \ymax) ; % third
  % \draw [red,line width=0.7pt] (\sxl,\sylr) -- (\xmax, \sylr) ; % fourth  
  \draw [line width=0.7pt] (\sxrl,\sy) -- (\sxrl, \ymax) ; % fifth
  % \draw [red,line width=0.7pt] (\sxlrr,\sylr) -- (\sxlrr, \sy) ; % sixth  
  \draw [line width=0.7pt] (\xmin,\syrll) -- (\sxrl, \syrll) ; % seventh

  %% points
  \draw (\xa,\ya) node {$\bullet$} node [left] {$x_3$} ;
  \draw (\xb,\yb) node {$\bullet$} node [below] {$x_4$} ;
  \draw (\xc,\yc) node {$\bullet$} node [right] {$x_1$} ;
  % \draw [red] (\xd,\yd) node {$\bullet$} ;  
  % \draw [red] (\xd,\yd) node [right] {$x_5$} ;  
  \draw (\xe,\ye) node {$\bullet$} ;
  \draw (\xe,\ye) node [right] {$x_2$} ;

  %%% tree representation
  \pgfmathsetmacro{\trx}{16.5} % tree root
  \pgfmathsetmacro{\try}{10}

  \pgfmathsetmacro{\ix}{2.5} % horizontal increment at level 1  
  \pgfmathsetmacro{\ixx}{1}  
  \pgfmathsetmacro{\ixxx}{0.5}  

  \pgfmathsetmacro{\iy}{-2.5} % vertical increment
  \pgfmathsetmacro{\iyy}{-2}
  \pgfmathsetmacro{\iyyy}{-1.5}

  \coordinate (E) at (\trx,\try) ; % root  
  \coordinate (L) at (\trx-\ix,\try+\iy) ;  
  \coordinate (R) at (\trx+\ix,\try+\iy) ;  
  \coordinate (RL) at (\trx+\ix-\ixx,\try+\iy+\iyy) ;
  \coordinate (RR) at (\trx+\ix+\ixx,\try+\iy+\iyy) ;
  \coordinate (RLL) at (\trx+\ix-\ixx-\ixxx,\try+\iy+\iyy+\iyyy) ;
  \coordinate (RLR) at (\trx+\ix-\ixx+\ixxx,\try+\iy+\iyy+\iyyy) ;
  \coordinate (RLLL) at (\trx+\ix-\ixx-\ixxx-\ixxx,\try+\iy+\iyy+\iyyy+\iyyy) ;
  \coordinate (RLLR) at (\trx+\ix-\ixx-\ixxx+\ixxx,\try+\iy+\iyy+\iyyy+\iyyy) ;
  \coordinate (LL) at (\trx-\ix-\ixx,\try+\iy+\iyy) ;  
  \coordinate (LR) at (\trx-\ix+\ixx,\try+\iy+\iyy) ;
  \coordinate (LRL) at (\trx-\ix+\ixx-\ixxx,\try+\iy+\iyy+\iyyy) ;
  \coordinate (LRR) at (\trx-\ix+\ixx+\ixxx,\try+\iy+\iyy+\iyyy) ;
  \coordinate (LRRL) at (\trx-\ix+\ixx+\ixxx-\ixxx,\try+\iy+\iyy+\iyyy+\iyyy) ;
  \coordinate (LRRR) at (\trx-\ix+\ixx+\ixxx+\ixxx,\try+\iy+\iyy+\iyyy+\iyyy) ;

%  \draw (E) node {$\bullet$} ;  
  \draw [line width=0.6pt] (E) -- (L) ;
  \draw [line width=0.6pt] (E) -- (R) ;
  \draw [line width=0.6pt] (R) -- (RL) ;
  \draw [line width=0.6pt] (R) -- (RR) node {\textcolor{white}{$\bullet$}} node {$\circ$} ;
  \draw [line width=0.6pt] (RL) -- (RLL) ;
  \draw [line width=0.6pt] (RL) -- (RLR) node {$\bullet$} node [below] {$x_1$} ;
  \draw [line width=0.6pt] (RLL) -- (RLLL) node {$\bullet$} node [below] {$x_4$} ;
  \draw [line width=0.6pt] (RLL) -- (RLLR) node {$\bullet$} node [below] {$x_3$} ;
  % \draw [red,line width=0.6pt] (L) -- (LL) node {\textcolor{white}{$\bullet$}} node {$\circ$} ;  
  % \draw [red,line width=1.5pt] (L) -- (LR) ;  
  % \draw [red,line width=0.6pt] (LR) -- (LRL) node {\textcolor{white}{$\bullet$}} node {$\circ$} ;  
  % \draw [red,line width=1.5pt] (LR) -- (LRR) ;  
  % \draw [red,line width=1.5pt] (LRR) -- (LRRL) node {$\bullet$} node [below] {$x_5$} ;  
  % \draw [red,line width=0.6pt] (LRR) -- (LRRR) ;  
  \draw (L) node {$\bullet$} node [below] {$x_2$} ;

  \draw (\trx,\ymin+0.7) node {Tree $\tree_t$ $(t=5)$ %partition    
    before seeing $x_t$};

\end{tikzpicture}    
  
%%% Local Variables:
%%% mode: latex
%%% TeX-master: "paper-mondrian-online"
%%% End:
  \end{minipage}
  \begin{minipage}{0.53\textwidth}
  	\begin{tikzpicture}[scale=0.25, every node/.style={scale=0.55}]

  %%% partition
  
  \pgfmathsetmacro{\xmin}{0} % first square
  \pgfmathsetmacro{\xmax}{10}
  \pgfmathsetmacro{\ymin}{0}
  \pgfmathsetmacro{\ymax}{10}

  % \pgfmathsetmacro{\xtrans}{15} % translation between two squares  

  \pgfmathsetmacro{\xa}{2} % point a
  \pgfmathsetmacro{\ya}{8}
  \pgfmathsetmacro{\xb}{1.3} % point b
  \pgfmathsetmacro{\yb}{6.1}
  \pgfmathsetmacro{\xc}{4.5} % point c
  \pgfmathsetmacro{\yc}{7}
  \pgfmathsetmacro{\xd}{5} % point d
  \pgfmathsetmacro{\yd}{4}
  \pgfmathsetmacro{\xe}{7} % point e
  \pgfmathsetmacro{\ye}{3}

  \pgfmathsetmacro{\sy}{4.7} % root split (y)
  \pgfmathsetmacro{\sxl}{4.2} % split left node (x)
  \pgfmathsetmacro{\sxr}{8}
  \pgfmathsetmacro{\sylr}{1.5}
  \pgfmathsetmacro{\sxrl}{3}
  \pgfmathsetmacro{\sxlrr}{6.3}
  \pgfmathsetmacro{\syrll}{6.8}

  %% squares
  \draw [line width=0.7pt] (\xmin,\ymin) rectangle (\xmax,\ymax) ; % left square

  %% splits
  \draw [line width=0.7pt] (\xmin,\sy) -- (\xmax, \sy) ; % first split
  \draw [red,line width=0.7pt] (\sxl,\ymin) -- (\sxl, \sy) ; % second split (only unrestricted)
  \draw [line width=0.7pt] (\sxr,\sy) -- (\sxr, \ymax) ; % third
  \draw [red,line width=0.7pt] (\sxl,\sylr) -- (\xmax, \sylr) ; % fourth
  \draw [line width=0.7pt] (\sxrl,\sy) -- (\sxrl, \ymax) ; % fifth
  \draw [red,line width=0.7pt] (\sxlrr,\sylr) -- (\sxlrr, \sy) ; % sixth
  \draw [line width=0.7pt] (\xmin,\syrll) -- (\sxrl, \syrll) ; % seventh

  %% points
  \draw (\xa,\ya) node {$\bullet$} node [left] {$x_3$} ;
  \draw (\xb,\yb) node {$\bullet$} node [below] {$x_4$} ;
  \draw (\xc,\yc) node {$\bullet$} node [right] {$x_1$} ;
  \draw [red] (\xd,\yd) node {$\bullet$} ;
  \draw [red] (\xd,\yd) node [right] {$x_5$} ;
  \draw (\xe,\ye) node {$\bullet$} ;
  \draw (\xe,\ye) node [right] {$x_2$} ;

  %%% tree representation
  \pgfmathsetmacro{\trx}{16.5} % tree root
  \pgfmathsetmacro{\try}{10}

  \pgfmathsetmacro{\ix}{2.5} % horizontal increment at level 1  
  \pgfmathsetmacro{\ixx}{1}  
  \pgfmathsetmacro{\ixxx}{0.5}  

  \pgfmathsetmacro{\iy}{-2.5} % vertical increment
  \pgfmathsetmacro{\iyy}{-2}
  \pgfmathsetmacro{\iyyy}{-1.5}

  \coordinate (E) at (\trx,\try) ; % root  
  \coordinate (L) at (\trx-\ix,\try+\iy) ;  
  \coordinate (R) at (\trx+\ix,\try+\iy) ;  
  \coordinate (RL) at (\trx+\ix-\ixx,\try+\iy+\iyy) ;
  \coordinate (RR) at (\trx+\ix+\ixx,\try+\iy+\iyy) ;
  \coordinate (RLL) at (\trx+\ix-\ixx-\ixxx,\try+\iy+\iyy+\iyyy) ;
  \coordinate (RLR) at (\trx+\ix-\ixx+\ixxx,\try+\iy+\iyy+\iyyy) ;
  \coordinate (RLLL) at (\trx+\ix-\ixx-\ixxx-\ixxx,\try+\iy+\iyy+\iyyy+\iyyy) ;
  \coordinate (RLLR) at (\trx+\ix-\ixx-\ixxx+\ixxx,\try+\iy+\iyy+\iyyy+\iyyy) ;
  \coordinate (LL) at (\trx-\ix-\ixx,\try+\iy+\iyy) ;  
  \coordinate (LR) at (\trx-\ix+\ixx,\try+\iy+\iyy) ;
  \coordinate (LRL) at (\trx-\ix+\ixx-\ixxx,\try+\iy+\iyy+\iyyy) ;
  \coordinate (LRR) at (\trx-\ix+\ixx+\ixxx,\try+\iy+\iyy+\iyyy) ;
  \coordinate (LRRL) at (\trx-\ix+\ixx+\ixxx-\ixxx,\try+\iy+\iyy+\iyyy+\iyyy) ;
  \coordinate (LRRR) at (\trx-\ix+\ixx+\ixxx+\ixxx,\try+\iy+\iyy+\iyyy+\iyyy) ;

%  \draw (E) node {$\bullet$} ;  
  \draw [line width=1.5pt] (E) -- (L) ;
  \draw [line width=0.6pt] (E) -- (R) ;
  \draw [line width=0.6pt] (R) -- (RL) ;
  \draw [line width=0.6pt] (R) -- (RR) node {\textcolor{white}{$\bullet$}} node {$\circ$} ;
  \draw [line width=0.6pt] (RL) -- (RLL) ;
  \draw [line width=0.6pt] (RL) -- (RLR) node {$\bullet$} node [below] {$x_1$} ;
  \draw [line width=0.6pt] (RLL) -- (RLLL) node {$\bullet$} node [below] {$x_4$} ;
  \draw [line width=0.6pt] (RLL) -- (RLLR) node {$\bullet$} node [below] {$x_3$} ;
  \draw [red,line width=0.6pt] (L) -- (LL) node {\textcolor{white}{$\bullet$}} node {$\circ$} ;
  \draw [red,line width=1.5pt] (L) -- (LR) ;
  \draw [red,line width=0.6pt] (LR) -- (LRL) node {\textcolor{white}{$\bullet$}} node {$\circ$} ;
  \draw [red,line width=1.5pt] (LR) -- (LRR) ;
  \draw [red,line width=1.5pt] (LRR) -- (LRRL) node {$\bullet$} node [below] {$x_5$} ;
  \draw [red,line width=0.6pt] (LRR) -- (LRRR) ;
  \draw (LRRR) node {$\bullet$} node [below] {$x_2$} ;

  % \draw [newgreen,fill] (E) +(-2pt,-2pt) rectangle +(2pt,2pt) ;
  % \draw [newgreen,fill] (L) +(-2pt,-2pt) rectangle +(2pt,2pt) ;
  % \draw [newgreen,fill] (LR) +(-2pt,-2pt) rectangle +(2pt,2pt) ;
  % \draw [newgreen,fill] (LRR) +(-2pt,-2pt) rectangle +(2pt,2pt) ;
  % \draw [newgreen,fill] (LRRL) +(-1.8pt,-1.8pt) rectangle +(1.8pt,1.8pt) ;
  \draw (\trx,\ymin+0.7) node {Updated tree %partition
    \textcolor{red}{$\tree_{t+1}$} $(t=5)$};

  % \draw (LRRL)+(0.3,0) .. controls +(0.5,1) and +(1,-2) .. (LRp);
  % \draw (LR)+(0.3,0) .. controls +(-1.2,2.4) and +(-1,-1) .. (E)+(0.3,0);

  %%% updates
  \pgfmathsetmacro{\xaxis}{22.5}
  \pgfmathsetmacro{\xpp}{0.5}
  \pgfmathsetmacro{\yaxismin}{2}
  \pgfmathsetmacro{\yaxismax}{10}
  \pgfmathsetmacro{\yaxisa}{7}
  \pgfmathsetmacro{\yaxisb}{6}
  \pgfmathsetmacro{\yaxisc}{5}
  \pgfmathsetmacro{\yaxisd}{4}

  \coordinate (Amin) at (\xaxis,\yaxismin) ;
  \coordinate (Amax) at (\xaxis,\yaxismax) ;
  
  \draw[-{latex[length=2mm, width=1mm]},dashed] (Amin) -- (Amax) ;
  % \draw (Amax) node {$\mathbf{\uparrow}$} ;
  \draw (\xaxis+\xpp,\yaxisa) node [right] {Updates along the \textbf{path} of $x_t$:} ;
  \draw (\xaxis+\xpp,\yaxisb) node [right] {${w_{\node, t+1}} = w_{\node, t} \exp(-\ell (\pred_{\node, t}, y_t))$} ;
  \draw (\xaxis+\xpp,\yaxisc) node [right] {${w_{\node, t+1}} = \frac{1}{2} w_{\node, t+1} + \frac{1}{2} \wbar_{\node 0, t+1} \wbar_{\node 1, t+1}$} ;
  \draw (\xaxis+\xpp,\yaxisd) node [right] {${\pred_{\node, t+1}} = \cdots$} ;

\end{tikzpicture}    
  
%%% Local Variables:
%%% mode: latex
%%% TeX-master: "paper-mondrian-online"
%%% End:	
  \end{minipage}  
  \caption{Illustration of the $\MondrianUpdate(x_t, y_t)$ procedure from Algorithm~\ref{alg:mondrian-update}: update of the partition, weights and node predictions as a new data point $(x_t, y_t)$ for $t=5$ becomes available.
    \emph{Left}: tree partition $\mondrian_t$ before seeing $(x_t, y_t)$. \emph{Right}:
    update of the partition (in red) and new splits to separate $x_5$ from $x_2$.    
    Empty circles ($\circ$) denote empty leaves, while leaves containing a point are indicated by a filled circle ($\bullet$).
    The path of $x_t$ in the tree is indicated in bold.
    The updates of weights and predictions along the path are indicated, and are computed in an upwards recursion.}
  \label{fig:algorithm-unrestricted}  
\end{figure}
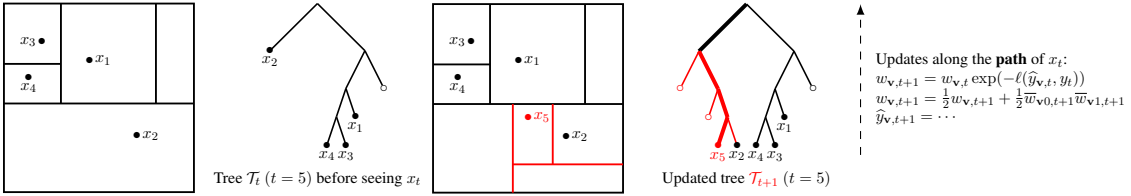
% Note that the partition update is also used in the prediction step, hence for clarity it may be isolated as a separate function.
% \jao{Est-ce que vous préférez la présentation ci-dessous, ou celle en termes de labeled trees (et en séparant la mise à jour de la partition de celle des poids/prédictions) pour la version avec extension de ranges ?}
%\jao{Continuer à parler des temps de splits ou pas ?}
% \jao{In Algorithm~\ref{alg:mondrian-update}, the first while loop (``downward loop'') does the update of the structure, while the second updates the weights and labels. The partition update is also used in the prediction step, hence for clarity it may be isolated as a separate function.}
% Note that, although the sum defining $\wbar_{\node, t}$ is over all subtrees
%The first while loop, which is a "downwards" loop, corresponds to the partition update, while the second corresponds to the weights and labels update.

\subsection{Prediction}
\label{subsec:AmfPrediction}

At any point in time, one can ask AMF to perform prediction for an arbitrary features vector $x \in [0, 1]^d$.
Let us assume that AMF did already $t$ training steps on the $M$ trees it contains and
let us recall that the prediction produced by AMF is the average of their predictions, see Equation~\eqref{eq_def_RF}, where the prediction $\widehat f_t(x, \Pi_t^{(m)})$ of each decision tree $m=1, \ldots, M$ is computed in parallel following Equation~\eqref{eq:exact-aggregation}.

The prediction of a decision tree is performed through a call to the procedure $\Predict(x)$ described in Algorithm~\ref{alg:predict} below.
% The tree contains the realization of a Mondrian partition $\Pi$ and each of its nodes has weights $ w_{\node}, \wbar_{\node}$ and a prediction function $\pred_\node$.
% The prediction for $x \in C$ is obtained as follows.
First, we perform a temporary partition update of $\mondrian$ using $x$, following Lines~2--10 of Algorithm~\ref{alg:mondrian-update}, so that we find or create a new leaf node $\node(x)$ such that $x \in C_{\node(x)}$.
Let us stress that this update of $\Pi$ using $x$ is discarded once the prediction for $x$ is produced, so that the decision function of AMF does not change after producing predictions. 
The prediction is then computed recursively, along an upwards recursion going from $\node(x)$ to the root $\root$, in the following way:
\begin{itemize}
\item if $\node = \node(x)$ we set $\wt y_{\node} = \pred_{\node}$;
\item if $\node \neq \node(x)$ (it is an interior node such that $x \in \cell_\node$), then 
assuming that $\node a$ ($a \in \{ 0 , 1\}$) is the child of $\node$ such that $x \in \cell_{\node a}$, we set
  \begin{equation*}
    \label{eq:update-avg-pred}
    \wt y_{\node}  = \frac{1}{2} \frac{w_{\node}}{\wbar_{\node}} \pred_{\node}
    + \frac{1}{2} \frac{\wbar_{\node 0} \wbar_{\node 1}}{\wbar_{\node}} \wt y_{\node a} \, .
  \end{equation*}
\end{itemize}
The prediction $\widehat f_t(x)$ of the tree is given by $\wt y_\root$, which is the last value obtained in this recursion.
Let us recall that this computes the aggregation with exponential weights of all the decision functions produced by all the prunings of the current Mondrian tree, as described in Equation~\eqref{eq:exact-aggregation} and stated in Proposition~\ref{prop:algorithm-implements-aggregation} above.
The prediction procedure is summarized in Algorithm~\ref{alg:predict} below.

\begin{algorithm}[htbp]
  \small
  \caption{$\Predict (x): $ predict the label of $x \in [0, 1]^d$}
  \label{alg:predict}
  \begin{algorithmic}[1]
  \STATE \textbf{Input:} a features vector $x \in [0, 1]^d$
    \STATE Follow Lines~2--10 of Algorithm~\ref{alg:mondrian-update} to do a temporary update of the current partition $\Pi$ using $x$ and let $\node(x)$ be the leaf such that $x \in C_{\node(x)}$
    \STATE Set $\wt y_\node = \pred_{\node(x)}$
    \WHILE{$\node \neq \root$}
    \STATE Let $(\node, \node a) = (\parent{\node}, \node)$ (for some $a\in \{ 0, 1\}$)
    \STATE Let
    $\wt y_\node = \frac{1}{2} \frac{w_\node}{\wbar_{\node}} \pred_{\node} + \frac{1}{2} \frac{\wbar_{\node (1-a)} \wbar_{\node a}}{\wbar_\node} \wt y_{\node a}$
    \ENDWHILE
    \STATE \textbf{Return} $\wt y_\root$
  \end{algorithmic}
\end{algorithm}

\subsection{AMF is efficient and exact} % (fold)
\label{sub:amf_does_what_is_claims}

As explained above, the aggregated predictions of all subtrees weighted by the prior $\pi$ from Equation~\eqref{eq:exact-aggregation} can be computed exactly using Algorithm~\ref{alg:predict}.
\begin{proposition}
  \label{prop:algorithm-implements-aggregation}
  Let $t\geq 1$ and $x \in [0, 1]^d$.
  The value $\widehat f_t(x)$ from Equation~\eqref{eq:exact-aggregation} can be computed exactly via the $\Predict$ procedure \textup(see Algorithms~\ref{alg:mondrian-update} and~\ref{alg:predict}\textup).
\end{proposition}
The proof of Proposition~\ref{prop:algorithm-implements-aggregation} is given in Section~\ref{sec:proofs}.
This prior choice enables to bypass the need to maintain one weight per subtree, and leads to the ``collapsed'' implementation described in Algorithms~\ref{alg:mondrian-update} and~\ref{alg:predict} that only requires to maintain one weight per node (which is exponentially smaller).
Note that this algorithm is \emph{exact}, in the sense that it does not require any approximation scheme. 
Moreover, this online algorithm corresponds to its batch counterpart, in the sense that there is no loss of information coming from the online (or streaming) setting versus the batch setting (where the whole dataset is available at once).

% \rev{You might want to discuss “Context Tree Switching” by Veness et al.~2011 which is known to improve over context tree weighting in some settings (see also “Skip Context Tree Switching” by Bellemare et al.~2014).}

The proof of Proposition~\ref{prop:algorithm-implements-aggregation} relies on some standard identities that enable to efficiently compute sums of products over tree structures in a recursive fashion (from~\citealp{helmbold1997pruning}), recalled in Lemma~\ref{lem:ctw-sum-prod} in Section~\ref{sec:proofs}.
Such identities are at the core of the \emph{Context Tree Weighting} algorithm (CTW), which our online algorithm implements (albeit over an evolving tree structure), and which constitutes an efficient way to perform Bayesian mixtures of contextual tree models under a branching process prior. The CTW algorithm, based on a sum-product factorization,
%simplification for the aggregation over tree structures under branching process priors,
is a state-of-the art algorithm used in lossless coding and compression.
% In order to be used in a streaming setting, where the structure of the tree is not fixed in advance, we adapt \emph{Context Tree Weighting} to the case when the ``global'' tree is growing over time --- more precisely, when both leaves and interior nodes of the current tree may be split.
We use a variant of the \emph{Tree Expert} algorithm \citep{helmbold1997pruning,cesabianchi2006PLG}, which is closely linked to CTW~\citep{willems1995context-basic,willems1998context-extensions,catoni2004statistical}.
We note that several extensions of CTW have been proposed in the framework of compression, to allow for bigger classes of models \citet{veness2012context} and leverage on the presence of irrelevant variables \citet{bellemare2014skip} just to name but a few.

\begin{table}
  \caption{Complexity of AMF versus Mondrian Forests.}
  \label{tab:complexity} 
  \centering
  \fbox{
    \begin{tabular}{cc}
      Algorithm & Complexity \\
      $\MondrianUpdate$ & $\Theta(\log n)$\\
      $\Predict$ & $\Theta(\log n)$\\
      $\MondrianForestUpdate$ & $\Theta( n)$\\
      $\MondrianForestPredict$ & $\Theta(\log n)$\\     
    \end{tabular}
  }
\end{table}

\begin{remark}
  \label{rem:complexity}
  The complexity of $\MondrianUpdate(x, y)$ is twice the depth of the tree at the moment it is called\textup, since it requires to follow a downwards path to a leaf\textup, and to go back upwards to the root.
  As explained in Proposition~\ref{prop:mondrian_depth_bound} from Section~\ref{sec:practical_implementation} below\textup, the depth of the Mondrian tree used in AMF is $\Theta(\log n)$ in expectation at step $n$ of training, which leads to a complexity $\Theta(\log n)$ both for Algorithms~\ref{alg:mondrian-update} and~\ref{alg:predict}\textup, where $\Theta(1)$ corresponds to the update complexity of a single node\textup, while the original MF algorithm uses an update with complexity that is linear in the number of leaves in the tree \textup(which is typically exponentially larger\textup).
  These complexities are summarized in Table~2. % \ref{tab:complexity}.
\end{remark}

\section{Theoretical guarantees}
\label{sec:theory}

In addition to being efficiently implementable in a streaming fashion, AMF is amenable to a thorough end-to-end theoretical analysis.
This relies on two main ingredients: $(i)$ a precise control of the geometric properties of the Mondrian partitions and $(ii)$ a regret analysis of the aggregation procedure (exponentially weighted aggregation of all finite prunings of the infinite Mondrian) which in turn yields excess risk bounds and adaptive minimax rates.
The guarantees provided below hold for a single tree in the Forest, but hold also for the average of several trees (used in by the forest) provided that the loss is convex.

\subsection{Related works on theoretical guarantees for Random Forests}

% \paragraph{Theoretical study of Random Forests.} % (fold)

% Several works study from a theoretical perspective the original Breiman's RF algorithm or its related variants, even if the data-dependent nature of the algorithm and its numerous components (sampling procedure, split selection, aggregation) make the theoretical analysis difficult. 

The consistency of stylized RF algorithms was first established by \citet{biau2008consistency_rf}, and later obtained for more sophisticated variants in \citet{denil2013online,scornet2015consistency_rf}.
Note that consistency results do not provide rates of convergence, and hence only offer limited guidance on how to properly tune the parameters of the algorithm.
Starting with \citet{biau2012analysis_rf, genuer2012variance_purf}, some recent work has thus sought to quantify the speed of convergence of some stylized variants of RF.
Minimax optimal nonparametric rates were first obtained by \citet{arlot2014purf_bias} in dimension~$1$ for the Purely Uniformly Random Forests (PURF) algorithm, in conjunction with suboptimal rates in arbitrary dimension (the number of features exceeds~$1$).

Several recent works \citep{wager2015adaptive,duroux2016impact} also established rates of convergence for variants of RF that essentially amount to some form of Median Forests, where each node contains at least a fixed fraction of observations of its parent. While valid in arbitrary dimension, the established rates are suboptimal. More recently, adaptive minimax optimal rates were obtained by~\citet{mourtada2018mondrian} in arbitrary dimension for the batch Mondrian Forests algorithm.
Our proposed online algorithm, AMF, also achieves minimax rates in an adaptive fashion, namely without knowing the smoothness of the regression function.

% myparagraph theoretical_study_of_random_forests (end)
As noted by \citet{rockova2017posterior}, the theoretical study of Bayesian methods on trees \citep{chipman1998bayesiancart,denison1998bayesiancart} or sum of trees \citep{chipman2010bart} is less developed.
\citet{rockova2017posterior} analyzes some variant of Bayesian regression trees and sum of trees; they obtain near minimax optimal posterior concentration rates.
Likewise,~\citet{linero2018bayesian} analyze Bayesian sums of soft decision trees models, and establish minimax rates of posterior concentration
  for the resulting SBART procedure.
While these frameworks differ from ours (herein results are posterior concentration rates as opposed to regret bounds and excess risk bounds, and the design is fixed), their approach differs from ours primarily in the chosen trade-off between computational complexity and adaptivity of the method:
these procedures involve approximate posterior sampling over large functional spaces through MCMC methods, and it is unclear whether the considered priors allow for reasonably efficient posterior computations.
In particular, the prior used in~\citet{rockova2017posterior} is taken over all subsets of variables, which is exponentially large in the number of features.

%   Indeed, on the one hand the aforementioned works establish a stronger form of adaptivity, since they can exploit a sparsity structure of the regression function, in the case where the latter only depends on a few important variables.
%   On the other hand, the corresponding procedures involve Bayesian inference over
%    large functional spaces (which is approximated through MCMC methods), and it is unclear whether the considered priors allow for reasonably efficient posterior computations.
% %
%   In principle, one could %also
%   achieve adaptation to the sparsity of the regression function in our setting, by  as in ; however, due to the high computational cost of such a procedure, we do not implement it and instead restrict to efficient algorithms. 

% OLD: {primarily in their theoretical focus: indeed, the considered procedures involves Bayesian inference over a very large space of functions, and it is not clear whether the considered prior allow for reasonably efficient posterior computations.}

% subsection subsection_name (end)

\subsection{Regret bounds}
\label{sub:regret-bounds}

For now, the sequence $(x_1, y_1), \dots, (x_n, y_n) \in [0, 1]^d \times \Y$ is arbitrary, and is in particular not required to be \iid
Let us recall that at step $t$, we have a realization $\Pi_t = (\tree_t, \Sigma_t)$ of a finite Mondrian tree, 
which is the minimal subtree of the infinite Mondrian partition $\mondrian = (\tree^\mondrian, \Sigma^\mondrian)$ that separates all distinct sample points in $\{ x_1, \dots, x_{t} \}$.
Let us recall also that $\pred_{\tree, t} : [0, 1]^d \to \predspace$ are the tree forecasters from Section~\ref{sub:agg-ctw}, where $\tree$ is some subtree of $\tree^\mondrian$.
We need the following
\begin{definition}
  Let $\eta>0$. A loss function $\ell : \predspace \times \Y \to \R$ is said to be $\eta$-\emph{exp-concave} if the function $\exp (-\eta \, \ell (\cdot, y)) : \predspace \to \R$ is concave for each $y \in \Y$.
\end{definition}
\noindent
The following loss functions are $\eta$-exp-concave:
\begin{itemize}
  \item The \emph{logarithmic loss} $\ell (\pred, y) = - \log \pred( y )$, with $\Y$ a finite set and $\predspace = \probas(\Y)$, with $\eta = 1$;
  \item The \emph{quadratic loss} $\ell (\pred, y) = (\pred - y)^2$ on $\Y = \predspace = [-B, B] \subset \R$, with $\eta = {1}/ ( 8B^2)$.
\end{itemize}

We start with Lemma~\ref{lem:regret-pruning}, which states that the prediction function used in AMF (see Equation~\eqref{eq:exact-aggregation}) satisfies a regret bound where the regret is computed with respect to any pruning $\tree$ of~$\tree^\mondrian$.
 \begin{lemma}
  \label{lem:regret-pruning}
  Consider a $\eta$-exp-concave loss function $\ell$. 
  Fix a realization $\mondrian = (\tree^\mondrian, \Sigma^\mondrian) \sim \MP$ and let $\tree \subset \tree^\mondrian$ be a finite subtree.
  For every sequence $(x_1, y_1), \dots, (x_n, y_n)$, the prediction functions $\wh f_1, \dots, \wh f_n$ based on $\mondrian$ and computed by AMF{} satisfy
  \begin{equation}
    \label{eq:regret-pruning}
    \sum_{t=1}^n \ell (\wh f_t (x_t) , y_t)
    - \sum_{t=1}^n \ell (\pred_{\tree, t} (x_t), y_t)
    \leq \frac{1}{\eta} |\tree| \log 2,
  \end{equation}
  where we recall that $|\tree|$ is the number of nodes in $\tree$.
\end{lemma}

Lemma~\ref{lem:regret-pruning} is a direct consequence of a standard regret bound for the exponential weights algorithm (see Lemma~\ref{lem:exp-regret} from Section~\ref{sec:proofs}), together with the fact that the Context Tree Weighting algorithm performed in Algorithms~\ref{alg:mondrian-update} and~\ref{alg:predict} computes it exactly, as stated in Proposition~\ref{prop:algorithm-implements-aggregation}.
By combining Lemma~\ref{lem:regret-pruning} with regret bounds for the online algorithms used in each node, both for the logarithmic loss and the quadratic loss, we obtain the following regret bounds with respect to any pruning $\tree$ of $\tree^\mondrian$.
\begin{corollary}[Classification]
  \label{cor:regret-best-pruning-log}
  % Let $\tree$ and $\mondrian_T$ be as in Lemma~\ref{lem:regret-pruning}.
  Fix $\mondrian = (\tree^\mondrian, \Sigma^\mondrian)$ as in Lemma~\ref{lem:regret-pruning} and consider the classification setting described in Section~\ref{sub:agg-ctw}. 
  For any finite subtree $\tree$ of $\tree^\mondrian$ and every sequence $(x_1, y_1), \dots, (x_n, y_n)$, the prediction functions $\wh f_1, \dots, \wh f_n$ based on $\mondrian$ computed by AMF{} with $\eta=1$ satisfy
  \begin{equation}
    \label{eq:regret-best-pruning-log}
    \sum_{t=1}^n \ell (\wh f_t (x_t) , y_t)
    - \sum_{t=1}^n \ell (g_{\tree} (x_t), y_t)
    \leq  |\tree| \log 2 + \frac{(|\tree|+1) (|\Y| - 1)}{4} \log (4n)
  \end{equation}
  for any function $g_{\tree} : [0, 1]^d \to \probas(\Y)$ which is constant on the leaves of $\tree$.
\end{corollary}

\begin{corollary}[Regression]
  \label{cor:regret-best-pruning-square}
  Fix $\mondrian = (\tree^\mondrian, \Sigma^\mondrian)$ as in Lemma~\ref{lem:regret-pruning} and consider the regression setting described 
  in Section~\ref{sub:agg-ctw} with $\Y = [-B, B]$.
  For every finite subtree $\tree$ of $\tree^\mondrian$ and every sequence $(x_1, y_1), \dots, (x_n, y_n)$, the prediction functions
   $\wh f_1, \dots, \wh f_n$ based on $\mondrian$ computed by AMF{} with $\eta = 1/(8B^2)$ satisfy
  \begin{equation}
    \label{eq:regret-best-pruning-square}
    \sum_{t=1}^n \ell (\wh f_t (x_t) , y_t)
    - \sum_{t=1}^n \ell (g_{\tree} (x_t), y_t)
    \leq 4 B^2 (|\tree| + 1) \log n
  \end{equation}
  for any function $g_{\tree} : [0, 1]^d \to \Y$ which is constant on the leaves of $\tree$.
\end{corollary}

The proofs of Corollaries~\ref{cor:regret-best-pruning-log} and~\ref{cor:regret-best-pruning-square}
are given in Section~\ref{sec:proofs}, and rely in particular on Lemmas~\ref{lem:regret-kt-log} and~\ref{lem:regret-average-square} that provide regret bounds for the online predictors $\pred_{\node, t}$ considered in the nodes.
Corollaries~\ref{cor:regret-best-pruning-log} and~\ref{cor:regret-best-pruning-square} that control the regret with respect to any pruning of $\tree^\mondrian$ imply in particular regret bounds with respect to any \emph{time pruning} of $\MP$. 

\begin{definition}[Time pruning]
  For $\lambda > 0$, the \emph{time pruning} $\mondrian_\lambda$ of $\mondrian$ at time $\lambda$ is obtained by removing any node $\node$ whose creation time $\tau_\node$ satisfies $\tau_\node > \lambda$. 
  We denote by $\MP(\lambda)$ the distribution of the tree partition $\mondrian_\lambda$ of $[0, 1]^d$.  
\end{definition}

The parameter $\lambda$ corresponds to a complexity parameter, allowing to choose a subtree of~$\tree^\mondrian$ where all leaves have a creation time not larger  than $\lambda$. 
We obtain the following regret bound for the regression setting (a similar statement holds for the classification setting), where the regret is with respect to any time pruning $\Pi_\lambda$ of $\Pi$.

\begin{corollary}
  \label{cor:regret-lambda-regression}
  Consider the same regression setting as in Corollary~\ref{cor:regret-best-pruning-square}.
  Then, AMF with $\eta = 1 / (8 B^2)$ satisfies
 \begin{equation}
    \label{eq:regret-lambda-regression}
    \E \bigg[ \sum_{t=1}^n \ell (\wh f_t (x_t), y_t) \bigg]
    \leq \E \bigg[ \inf_g \sum_{t=1}^n \ell (g (x_t), y_t) \bigg] + 8 B^2 (1+\lambda)^d \log n
    \, ,
  \end{equation}
  where the expectations on both sides are over the random sampling of the partition $\mondrian_\lambda \sim \MP (\lambda)$, and the infimum is taken over all functions $g : [0, 1]^d \to \R$ that are constant on the cells of $\mondrian_\lambda$.
\end{corollary}

Corollary~\ref{cor:regret-lambda-regression} controls the regret of AMF with respect to a time pruned Mondrian partition $\Pi_\lambda \sim \MP(\lambda)$ for \emph{any} $\lambda > 0$.
This result is one of the main ingredients allowing to prove that AMF is able to adapt to the unknown smoothness of the regression function, as stated in Theorem~\ref{thm:minimax-adaptive} below.
The proof of Corollary~\ref{cor:regret-lambda-regression} is given in Section~\ref{sec:proofs}, and mostly relies on a previous result from~\citet{mourtada2018mondrian} which proves that $\E [|\leaves(\mondrian_\lambda)|] = (1+\lambda)^d$ whenever $\Pi_\lambda \sim \MP(\lambda)$, where $|\leaves(\mondrian_\lambda)|$ stands for the number of leaves in the partition $\mondrian_\lambda$.
Let us pause for a minute and discuss the choice of the prior used in AMF, compared to what is done in literature with Bayesian approaches for instance.

\paragraph{Prior choice.}

The use of a branching process prior on prunings of large trees is common in literature on Bayesian regression trees.
Indeed, \citet{chipman1998bayesiancart} choose a branching process prior on subtrees, with a splitting probability of each node $\node$ of the form
\begin{equation}
  \label{eq:prior-chipman}
  \alpha (1 + d_\node)^{-\beta}
\end{equation}
for some $\alpha \in (0, 1)$ and $\beta \geq 0$, where  $d_\node$ is the depth of note $\node$.
Note that the locations of the splits themselves are also parameters of the Bayesian model, which enables more flexible estimation, but prevents efficient closed-form computations.
The same prior on subtrees is used in the BART algorithm~\citep{chipman2010bart}, which considers sums of trees.
We note that several values are proposed for the parameters $(\alpha, \beta)$, although there does not appear to be any definitive choice or criterion (\citealt{chipman1998bayesiancart} considers several examples with $\beta \in [\frac{1}{2}, 2]$, while \citealt{chipman2010bart} suggest $(\alpha, \beta) = (0.95, 2)$ for BART).
The prior $\pi$ considered in this paper (see Equation~\eqref{eq:ctw-prior}) has a splitting probability $1 / 2$ for each node, so that $(\alpha, \beta) = (1 / 2, 0)$.
One appeal of the regret bounds stated above is that it offers guidance on the choice of parameters.
Indeed, it follows as a by-product of our analysis that the regret of AMF (with any prior $\pi$) with respect to a subtree $\tree$ is $O(\log \pi(\tree)^{-1})$.
This suggests to choose $\pi$ in AMF as flat as possible, namely $(\alpha, \beta) = (1 / 2, 0)$.

% Slightly suboptimal regret bounds when using priors like the one of \citet{chipman1998bayesiancart}, with branching probability of $(1+depth_\node)^{-\alpha}$ : regret wrt subtree $\tree$ of order $\alpha \sum_{\node \in \inodes(T)} \log (1+depth_\node) + \sum_{\node \in \leaves(T)} \dots = O( L \log L)$ with $L = |T|$ (in the worst case, \ie unbalanced tree), whereas CTW leads to a regret of order $O(L)$.
% Hence, the prior considered by \citet{chipman1998bayesiancart}
% \erw{Possible d'en faire un Lemme pour montrer que le prior choisi précédemment a un meilleur regret?}
% \jao{BART aussi a ce prior sur les arbres (cf p.7). En fait, ils prennent la probabilité de couper une cellule dans le prior égale à $\alpha (1+d)^{-\beta}$ où $d$ est la profondeur, et proposent des choix heuristiques de $\alpha,\beta$.}

\subsection{Adaptive minimax rates through online to batch conversion}
\label{sub:adaptive-rates}

In this Section, we show how to turn the algorithm described in Section~\ref{sec:algorithm} into a supervised learning algorithm with generalization guarantees that entail as a by-product adaptive minimax rates for nonparametric estimation.
Namely, this section is concerned with bounds on the risk (expected prediction error on unseen data)
rather than the regret of the sequence of prediction functions that was studied in 
Section~\ref{sub:regret-bounds}.
Therefore, we assume in this Section that the sequence $(x_1, y_1), (x_2, y_2), \dots$ consists of \iid random variables in $[0,1]^d \times \Y$, such that each $(x_t, y_t)$ that comes sequentially is distributed as some generic pair $(x, y)$.
The quality of a \emph{prediction function} $g : [0,1]^d \to \Y$ is measured by its risk defined as
\begin{equation}    
\label{eq:risk}
\risk (g) = \E [ \ell (g(x), y)] \, .
\end{equation}

% \jao{On dirait que l'on ne parle pas de l'excès de risque ci-dessous dans le reste du papier ?}
% Denoting by $\risk^*$ the infimum of $\risk (g)$ over all prediction functions $g$, the quality of $g$ is measured by its (expected) excess risk
% \begin{equation}
% \label{eq:excessrisk-rule}
% \excessrisk(g) = \E [\risk (g)] - \risk^* .
% \end{equation}

\paragraph{Online to batch conversion.} % (fold)

Our supervised learning algorithm remains \emph{online} (it does not require the knowledge of a fixed number of points $n$ in advance).
It is also virtually parameter-free, the only parameter being the learning rate $\eta$ (set to $1$ for the log-loss).
In order to obtain a supervised learning algorithm with provable guarantees, we use \emph{online to batch} conversion from~\cite{cesabianchi2004online_to_batch}, which turns any regret bound for an online algorithm into an excess risk bound for the average or a randomization of the past values of the online algorithm.
As explained below, it enables to obtain fast rates for the excess risk, provided that the online procedure admits appropriate regret guarantees.

% Note that one has to randomize the \emph{functions}, and not just the predictions.
% But the information of any function is equivalent to storing all past labeled trees, which is costly.
% Trick to randomize anytime.

% Also, some care has to be taken in order to account for the fact that the randomized restricted partition is revealed gradually, and depends on past feature points $X_1, \dots, X_t$.
%% so that the information on the \emph{partition} structure is used even after the weights and labels are no longer updated.

\begin{lemma}[Online to batch conversion]
  \label{lem:online-to-batch}
  % where $\Y$ is an arbitrary measurable set
  % Assume that $\predspace$ is a convex subset of a vector space, and that $\ell : \predspace \times \Y \to \R^+$ is a loss function which is convex in its first argument.
  Assume that the loss function $\ell : \predspace \times \Y \to \R^+$ is measurable, with $\predspace$ a measurable space, and let $\G$ be a class of measurable functions $[0,1]^d \rightarrow \predspace$.
  Given $f_1, \dots, f_{n}$ where $f_t : ([0,1]^d \times \Y)^{t-1} \to \predspace^{[0,1]^d}$\textup, we denote
  $\widehat f_t = f_t ((x_1, y_1),$ $\dots, (x_{t-1}, y_{t-1}))$. 
  Let $\widetilde f_n = \widehat f_{I_n}$ with $I_n$ a random variable uniformly distributed on 
  $\{ 1, \dots, n \}$.
  Then, we have
  \begin{equation}
    \label{eq:online-to-batch}
    \E [ R (\widetilde f_n) ] - R (g) 
    =  \frac{1}{n} \E \Big[ \sum_{t=1}^n \big( \ell (\widehat f_t (x_t), y_t) - 
    \ell (g(x_t), y_t) \big) \Big],
  \end{equation}
  which entails that  the expected excess risk of $\widetilde f_n$ with respect to 
  any $g \in \G$ is equal to the expected per-round regret of $\widehat f_1, \dots, 
  \widehat f_n$ with respect to $g$.
\end{lemma}

Although this result is well-known~\citep{cesabianchi2004online_to_batch}, we provide for completeness a proof of this specific formulation in Section~\ref{sec:proofs}.
% \begin{remark}
%   Note that, as defined above, $\bar f_n$ does not depend on the last sample $(X_n, Y_n)$ (since $\bar f_1, \dots, \wh f_n$ are determined by $(X_1, Y_1), \dots, (X_{n-1}, Y_{n-1})$).
%   It is in fact possible to state the same result for $\bar f_{n+1}$ (which can be obtained from the sample), by replacing $n$ with $n+1$ and involving a fictitious sample $(X_{n+1}, Y_{n+1})$ with the same distribution as $(X_t, Y_t)$, which is only used in the analysis.
% \end{remark}
In our case, $\G$ will be the (random) family of functions that are constant on the leaves of some pruning of an infinite Mondrian partition $\Pi$, and $\widehat f_1, \dots, \widehat  f_n$ will be the sequence of prediction functions of AMF{}.
Note that, when conditioning on $\Pi$ which is used to define both the class $\G$ and the algorithm, both $\G$ and the maps
$f_1, \dots, f_n$ become deterministic, so that we can apply Lemma~\ref{lem:online-to-batch} conditionally on $\mondrian$.
% It should be noted that $\mondrian_\infty$, and hence $\F$ and $h_1, \dots, h_n$, is not revealed or sampled, since the algorithm only samples and uses the restrictions to the ranges of $\{ X_1, \dots, X_t \}$ $(1\leq t \leq n)$, and does so progressively.
In what follows, we denote by $\widetilde f_n$ the outcome of online to batch conversion applied to our online procedure.

% More precisely, fix a realization $\mondrian$ of the Mondrian process.
% The batch procedure $f_n$ is obtained as follows: select a random index $I_n$ uniformly in $\{1, \dots, I_n\}$, and predict according to the function $\wh f_{I_n}$.
% For the latter part, computing 
% \jao{Say that it is not necessary to explicitly sample the following splits for each $x$, and that we can take the expectation over the random sampling of the following splits ?
%   In this case, argue that the guarantee also follows for the averaged version by Jensen.
% }

\paragraph{Oracle inequality and minimax rates.}

Let us show now that $\widetilde f_n$ achieves adaptive minimax rates under nonparametric assumptions, which complements and improves previous results~\citep{mourtada2017mondrian,mourtada2018mondrian}.
Indeed, AMF addresses the practical issue of optimally tuning the complexity parameter $\lambda$ of Mondrian trees, while remaining a very efficient online procedure.
As the next result shows, the procedure $\widetilde f_n$, which is virtually parameter-free, performs at least almost as well as the Mondrian tree with the best $\lambda$ chosen with hindsight.
For the sake of conciseness, Theorems~\ref{thm:oracle-best-lambda} and~\ref{thm:minimax-adaptive} are stated only in the regression setting, although a similar result holds for the log-loss.

\begin{theorem}
  \label{thm:oracle-best-lambda}
  Consider the same setting as in Corollary~\ref{cor:regret-best-pruning-square}\textup, the only difference being the fact that the sequence $(x_1, y_1), \ldots, (x_n, y_n)$ is \iid and consider the online to batch conversion $\widetilde f_n$ from Lemma~\ref{lem:online-to-batch} applied to AMF.
  For every $\lambda > 0$ and every function $g_\lambda$ which is constant on the cells of a 
  random partition $\mondrian_\lambda \sim \MP (\lambda)$\textup, we have
  \begin{equation}
    \label{eq:oracle-best-lambda}
    \E [\risk(\widetilde f_n)] - \E [\risk (g_\lambda)] \leq 8 B^2 (1+\lambda)^d \frac{\log n}{n}\, .
  \end{equation}
\end{theorem}

The proof of Theorem~\ref{thm:oracle-best-lambda} is given in Section~\ref{sec:proofs}.
It provides an oracle bound which is distribution-free, since it requires no assumption on the joint distribution of $(x, y)$ apart from $\Y = [-B, B]$.
Combined with previous results on Mondrian partitions~\citep{mourtada2018mondrian}, which enable to control the approximation properties of Mondrian trees, Theorem~\ref{thm:oracle-best-lambda} implies that $\widetilde f_n$ is adaptive with respect to the smoothness of the regression function, as shown in 
Theorem~\ref{thm:minimax-adaptive}.

\begin{theorem}
  \label{thm:minimax-adaptive}
  Consider the same setting as in Theorem~\ref{thm:oracle-best-lambda} and assume that the regression function $f^*(\cdot) = \E [y | x = \cdot]$ is $\beta$-H\"older with $\beta \in (0, 1]$ unknown.
  Then\textup, we have
  \begin{equation}
    \label{eq:minimax-adaptive}
    \E [ ( \widetilde f_n(x) - f^*(x) )^2 ] = O \Big( \Big( \frac{\log n}{n} \Big)^{2 \beta / (d+ 2\beta)} \Big),
  \end{equation}
  which is \textup(up to the $\log n$ term\textup) the minimax optimal rate of estimation over the class of $\beta$-Hölder functions.
\end{theorem}

The proof of Theorem~\ref{thm:minimax-adaptive} is given in Section~\ref{sec:proofs}.
Theorem~\ref{thm:minimax-adaptive} states that the online to batch conversion $\widetilde f_n$ of AMF is adaptive to the unknown H\"older smoothness $\beta \in (0, 1]$ of the regression function since it achieves, up to the $\log n$ term, the minimax rate $n^{-2 \beta / (d + 2 \beta)}$, see~\cite{stone1982optimal}.
It would be theoretically possible to modify the procedure in order to ensure adaptivity to higher regularities (say, up to some order $\bar{\beta} \in \N \setminus \{ 0 \}$), by replacing the constant estimates inside each node by polynomials (of order $\bar{\beta} - 1$). 
However, this would lead to a numerically involved procedure, that is beyond the scope of the paper.
In addition, it is known that averaging can reduce the bias of individual randomized tree estimators for twice differentiable functions, see~\cite{arlot2014purf_bias} and~\cite{mourtada2018mondrian} for Mondrian Forests.
Such results cannot be applied to AMF, since its decision function involves a more complicated process of aggregation over all subtrees.

\section{Practical implementation of AMF}
\label{sec:practical_implementation}

This section describes a modification of AMF that we use in practice, in particular for all the numerical experiments performed in the paper.
Because of extra technicalities involved with the modified version described below, we are not able to provide theoretical guarantees similar to what is done in Section~\ref{sec:theory}.
Indeed, the procedure described in this section exhibits a more intricate behaviour: new splits may be inserted above previous splits, which affects the underlying tree structure as well as the underlying prior over subtrees.
This section mainly modifies the procedures described in Algorithms~\ref{alg:mondrian-update} and~\ref{alg:predict} so that splits are sampled only within the range of the features seen in each node, see Section~\ref{sec:restr-splits-range}, with motivations to do so described below. 
Moreover, we provide in Section~\ref{sub:complexity} a guarantee on the average computational complexity of AMF through a control of the expected depth of the Mondrian partition.

\subsection{Restriction to splits within the range of sample points}
\label{sec:restr-splits-range}

Algorithm~\ref{alg:mondrian-update} from Section~\ref{subsec:PredictionUpdate} samples successive splits on the whole domain $[0, 1]^d$.
In particular, when a new features vector $x_t$ is available, it samples splits of the leaf $\node_t$ containing $x_t$ until a split successfully separates $x_t$ from the other point $x_s \neq x_t$ contained in $\node_t$ (unless $\node_t$ was empty).
In the process, several splits outside of the box containing $x_s$ and $x_t$ can be performed.
These splits are somewhat superfluous, since they induce empty leaves and delay the split that separates these two points.
% \jao{especially important in large dimension}
Removing those splits is critical to the performance of the method, in particular when the ambient dimension of the features is not small.
In such cases, many splits may be needed to separate the feature points.
On the other hand, only keeping those splits that are necessary to separate the sample points may yield a more adaptive partition, which can better adapt to a possible low-dimensional structure of the distribution of $x$.

We describe below a modified algorithm that samples splits in the range of the features vectors seen in each cell, exactly as in the original Mondrian Forest algorithm~\citep{lakshminarayanan2014mondrianforests}.
In particular, each leaf will contain exactly one sample point by construction (possibly with repetition if $x_s = x_t$ for some $s \neq t$) and no empty leaves.
Formally, this procedure amounts to considering the \emph{restriction} of the 
Mondrian partition to the finite set of points $\{ x_1, \dots, x_t\}$~\citep{lakshminarayanan2014mondrianforests}, where it is shown that such a restricted Mondrian partition can be updated efficiently in an online fashion, thanks to properties of the Mondrian process.
This update exploits the creation time $\birth_\node$ of each node, as well as the range of the features vectors $R_\node$ seen inside each node (as opposed to only leaves).
Moreover, this procedure can possibly split an interior node and not only a leaf.
The algorithm considered here is a modification of the procedure $\mathtt{ExtendMondrianTree}(\tree, \lambda, (x_t, y_t))$ described in~\cite{lakshminarayanan2014mondrianforests}, where we use $\lambda = +\infty$ and where we perform the exponentially weighted aggregation of subtrees described in Sections~\ref{sec:introduction} and~\ref{sec:algorithm}.

\begin{figure}[h!]
  \centering
  \begin{tikzpicture}[scale=0.25, every node/.style={scale=0.55}]

  \pgfmathsetmacro{\xmin}{0} % first square
  \pgfmathsetmacro{\xmax}{10}
  \pgfmathsetmacro{\ymin}{0}
  \pgfmathsetmacro{\ymax}{10}

  \pgfmathsetmacro{\xtrans}{15} % translation between two squares

  \pgfmathsetmacro{\xa}{2} % point a
  \pgfmathsetmacro{\ya}{8}
  \pgfmathsetmacro{\xb}{1.3} % point b
  \pgfmathsetmacro{\yb}{6.1}
  \pgfmathsetmacro{\xc}{4.5} % point c
  \pgfmathsetmacro{\yc}{7}
  \pgfmathsetmacro{\xd}{5} % point d
  \pgfmathsetmacro{\yd}{4}
  \pgfmathsetmacro{\xe}{7} % point e
  \pgfmathsetmacro{\ye}{3}

  \pgfmathsetmacro{\sy}{4.7} % root split (y)
  \pgfmathsetmacro{\sxl}{4.2} % split left node (x)
  \pgfmathsetmacro{\sxr}{8}
  \pgfmathsetmacro{\sylr}{1.5}
  \pgfmathsetmacro{\sxrl}{3}
  \pgfmathsetmacro{\sxlrr}{6.3}
  \pgfmathsetmacro{\syrll}{6.8}

  %% squares
  \draw [line width=0.7pt] (\xmin,\ymin) rectangle (\xmax,\ymax) ; % left square
  \draw [line width=0.7pt] (\xmin+\xtrans, \ymin) rectangle (\xmax+\xtrans, \ymax) ; % right square

  %% splits
  \draw [line width=0.7pt] (\xmin,\sy) -- (\xmax, \sy) ; % first split
  \draw [line width=0.7pt] (\xmin+\xtrans,\sy) -- (\xmax+\xtrans, \sy) ; % first split (right square)
  \draw [blue,line width=0.7pt] (\sxl,\ymin) -- (\sxl, \sy) ; % second split (only unrestricted)
  \draw [blue,line width=0.7pt] (\sxr,\sy) -- (\sxr, \ymax) ; % third
  \draw [blue,line width=0.7pt] (\sxl,\sylr) -- (\xmax, \sylr) ; % fourth
  \draw [line width=0.7pt] (\sxrl,\sy) -- (\sxrl, \ymax) ; % fifth
  \draw [line width=0.7pt] (\sxrl+\xtrans,\sy) -- (\sxrl+\xtrans, \ymax) ;
  \draw [line width=0.7pt] (\sxlrr,\sylr) -- (\sxlrr, \sy) ; % sixth
  \draw [line width=0.7pt] (\sxlrr+\xtrans,\ymin) -- (\sxlrr+\xtrans, \sy) ;
  \draw [line width=0.7pt] (\xmin,\syrll) -- (\sxrl, \syrll) ; % seventh
  \draw [line width=0.7pt] (\xmin+\xtrans,\syrll) -- (\sxrl+\xtrans, \syrll) ;

  %% points
  \draw (\xa,\ya) node {$\mathbf{\bullet}$} ;
  \draw (\xb,\yb) node {$\mathbf{\bullet}$} ;
  \draw (\xc,\yc) node {$\mathbf{\bullet}$} ;
  \draw (\xd,\yd) node {$\mathbf{\bullet}$} ;
  \draw (\xe,\ye) node {$\mathbf{\bullet}$} ;

  \draw (\xa+\xtrans,\ya) node {$\mathbf{\bullet}$} ;  
  \draw (\xb+\xtrans,\yb) node {$\mathbf{\bullet}$} ;
  \draw (\xc+\xtrans,\yc) node {$\mathbf{\bullet}$} ;
  \draw (\xd+\xtrans,\yd) node {$\mathbf{\bullet}$} ;
  \draw (\xe+\xtrans,\ye) node {$\mathbf{\bullet}$} ;

      % fill global cell
      % \fill [black!3] (\Rl,\Rd) rectangle (\Rr,\Ru) ;
      
      % outer segments
      % \draw [line width=0.7pt, black!25] (\xa,\Rd) -- (\xa, \Ru) ;
      % \draw [line width=0.7pt, black!25] (\Rl,\xb) -- (\Rr, \xb) ;
      % \draw [loosely dotted] (\xa,0) -- (\xa, 10) ;
      % \draw (X) node {$\bullet$} ;  
      % \draw [line width=1pt, black!25] (\rl,\Rd) -- (\rl, \Ru) ;

      % fill cell      
      % \fill [yellow!10] (CLD) rectangle (CRU) ;

      % interior segments
      % \draw [blue!60, line width=0.8pt, dashed] (\rl, \xb) -- (X) ;
      % \draw [blue!60, line width=0.8pt, dashed] (\rr, \xb) -- (X) ;
      % \draw [blue!60, line width=0.8pt, dashed] (\xa, \rd) -- (X) ;
      % \draw [blue!60, line width=0.8pt, dashed] (\xa, \ru) -- (X) ;      

\end{tikzpicture}    
  
%%% Local Variables:
%%% mode: latex
%%% TeX-master: "paper-mondrian-online"
%%% End:
  \caption{Unrestricted (left) \emph{vs.} restricted (right) Mondrian partitions.
    Dots ($\bullet$) represent sample points.
    In both cases, cells containing one sample point are no longer split.
    In addition, the restricted Mondrian partition is obtained by removing  from the unrestricted partition all splits (in blue) that create empty leaves.}
  \label{fig:restricted-unrestricted}  
\end{figure}
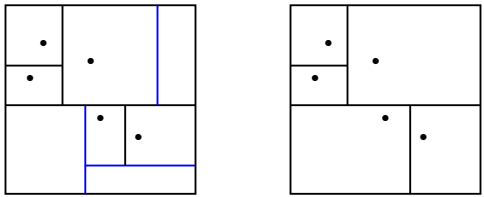

We call the former partition (from Section~\ref{subsec:PredictionUpdate}) an \emph{unrestricted} Mondrian partition, while the one described here will be referred to as a \emph{restricted} Mondrian partition.
The difference between the two is illustrated in Figure~\ref{fig:restricted-unrestricted}.
The tree $\tree$ contains, as before, parent and children relations between all nodes $\node \in \tree$, while each $\sigma_\node \in \Sigma$ contains
\begin{equation}
  \label{eq:data_in_node_restricted}
  \sigma_\node = (j_\node, s_\node, \birth_\node, \pred_{\node}, 
  w_{\node}, \wbar_{\node}, R_\node),
\end{equation}
which differs from Equation~\eqref{eq:data_in_node_unrestricted} since we keep in memory the creation time $\birth_\node$ of $\node$, and the range
\begin{equation*}
  R_\node = \prod_{j=1}^d [a_\node^j, b_\node^j]
\end{equation*}
of features vectors in $C_\node$ instead of $x_\node$ (a past sample point).
Another advantage of the restricted Mondrian partition is that the algorithm is 
\emph{range-free}, since it does not require to assume that all features vectors are in $[0, 1]^d$ (we simply use as initial root cell $\cell_\root = \R^d$).

Algorithms~\ref{alg:mondrian-update-experiments} and~\ref{alg:extend-split} below implement AMF with a restricted Mondrian partition, and are used instead of the previous Algorithm~\ref{alg:mondrian-update} in our numerical experiments.
These algorithms, together with Algorithm~\ref{alg:predict-experiments} below for prediction, maintain in memory, as in Section~\ref{sec:algorithm}, the current state of the Mondrian partition $\Pi = (\tree, \Sigma)$, which contains the tree structure $\tree$ (containing parent/child relationships between nodes) and data $\sigma_\node \in \Sigma$ for all nodes, see Equation~\eqref{eq:data_in_node_restricted}.
An illustration of Algorithms~\ref{alg:mondrian-update-experiments} and~\ref{alg:extend-split} is provided in Figure~\ref{fig:algorithm-restricted}.
We use the notation $x_+ = \max(x, 0)$ for any $x \in \R$.
\begin{algorithm}[h!]
  \caption{$\MondrianUpdate(x, y):$ update AMF with a new sample
   $(x, y) \in \R^d \times \Y$}
  \label{alg:mondrian-update-experiments}
  \small
  \begin{algorithmic}[1]
    \STATE \textbf{Input:} a new sample $(x, y) \in \R^d \times \Y$
    \IF{$\tree = \varnothing$}
    \STATE Put $\tree = \{ \root \}$ and $(\birth_\root, \pred_{\root}, w_{\root}, \wbar_{\root}, R_\root) = (0, h (\varnothing), 1, 1, \{ x \})$
    \ELSE
    \STATE {Call} $\ExtendCell(\root, x)$ from Algorithm~\ref{alg:extend-split}
    \ENDIF
    \STATE Let $\node(x)$ be the leaf such that $x \in \cell_{\node(x)}$ and put $\node = \node(x)$
    \STATE Let $\mathtt{continueUp} = \mathtt{true}$.  
    \WHILE{$\mathtt{continueUp}$}
    \STATE Set $w_\node = w_\node \exp (- \eta \ell(\pred_\node, y))$
    \STATE Set $\wbar_\node = w_\node$ if $\node$ is a leaf and $\wbar_\node = \frac{1}{2} w_\node + \frac{1}{2} \wbar_{\node 0} \wbar_{\node 1}$ otherwise
    \STATE Update $\pred_\node$ using $y$ (see Section~\ref{sub:agg-ctw})
    \STATE If $\node \neq \root$ let $\node = \parent{\node}$, otherwise let $\mathtt{continueUp} = \mathtt{false}$
    \ENDWHILE
  \end{algorithmic}
\end{algorithm}
In algorithm~\ref{alg:mondrian-update-experiments}, Line~3 initializes the tree the first time $\MondrianUpdate$ is called, otherwise the recursive procedure $\ExtendCell$ is used to update the restricted Mondrian partition, starting at the root~$\root$. Lines~7--14 perform the update of the aggregation weights in the same way as what we did in Section~\ref{subsec:PredictionUpdate}.

\begin{algorithm}[h!]
  \caption{$\ExtendCell(\node, x): $ update node $\node$ using $x$}
  \label{alg:extend-split} 
  \small
  \begin{algorithmic}[1]
    \STATE \textbf{Input:} a node $\node \in \tree$ from the current tree and a features vector $x \in \R^d$
    \STATE Let $\Delta_{j} = (x_j - b_\node^j)_+ + (a_\node^j - x_j)_+$ 
    and $\Delta = \sum_{j=1}^d \Delta_j$
    \STATE Sample $E \sim \expdist(\Delta)$ and put $E = + \infty$ if $\Delta = 0$ (namely $x \in \cellrange_\node$)
    \IF{$\node$ is a leaf \textbf{or} $\birth_\node + E < \birth_{\node 0}$}
    \STATE Sample a split coordinate $J \in \{ 1, \dots, d \}$ with $\P (J = j) 
    = \Delta_j / \Delta$
    \IF {$x_{J} < a_\node^{J}$}
    \STATE Put $a = 0$ and sample the split threshold
     $S | J \sim \uniformdist([x_{J}, a_\node^{J}])$
    \ELSE
    \STATE Put $a = 1$ and sample the split threshold
     $S | J \sim \uniformdist([b_\node^{J}, x_{J}])$
    \ENDIF
    \STATE Set $\tree_{\node (1-a)} = \tree_\node$, namely nodes $\node \node'$ are renamed as $\node (1 - a) \node'$ for any $\node' \in \tree_{\node}$
    \STATE Create nodes $\node 0$ and $\node 1$ and put $(\birth_{\node a}, \pred_{\node a}, w_{\node a}, \wbar_{\node a}, R_{\node a}) = (\tau_\node + E, h(\varnothing), 1, 1, \{ x\})$, put $\sigma_{\node (1 - a)} = \sigma_\node$ (see Equation~\eqref{eq:data_in_node_restricted}) but set $\tau_{\node (1 - a)} = \tau_\node + E$
    \STATE Put $a_\node^j = \min(a_\node^j, x_j)$ and $b_\node^j = \max(b_\node^j, x_j)$  
    \ELSE
    \STATE Put $a_\node^j = \min(a_\node^j, x_j)$ and $b_\node^j = \max(b_\node^j, x_j)$ 
    \STATE Let $a \in \{0, 1\}$ be such that $x \in \cell_{\node a}$ and call
     $\ExtendCell(\node a, x)$
    \ENDIF
  \end{algorithmic}
\end{algorithm}

In Algorithm~\ref{alg:extend-split}, Line~2 computes the range extension of $x$ with respect to $R_\node$.
In particular, if $x \in R_\node$, then no split will be performed and we go directly to Line~15. 
Otherwise, if $x$ is outside of $R_\node$, a split of $\node$ is performed whenever $\birth_\node + E < \birth_{\node 0}$ (a new node created at time $\birth_\node + E$ can be inserted before the creation time $\birth_{\node 0}$ of the current child $\node 0$ of $\node$).
In this case, we sample the split coordinate $j$ proportionally to $\Delta_j$ (coordinates with the largest extension are more likely to be used to split $\node$) and we sample the split threshold uniformly at random within the corresponding extension (Line~7 or Line~9).
Now, at Line~11, we move downwards the whole tree rooted at $\node$: any node at index $\node \node'$ for any $\node' \in \tree_\node$ is renamed as $\node (1-a) \node'$.
For instance, if $a=0$ (Line~7, the extension is on the left of the current range), the node $\node 0$ is renamed as $\node 10$, the node $\node 1$ as $\node 11$, etc.
Then, at Line~12, new nodes $\node 0$ and $\node 1$ are created, where $\node a$ is a new leaf containing~$x$ and $\node (1 - a)$ is a new node which is the root of the subtree we moved downwards at Line~11.
Line~12 also initializes $\sigma_{\node a}$ and copies $\sigma_\node$ into $\sigma_{\node (1 -a)}$.
The process performed in Lines 11--12 therefore simply inserts two new nodes below $\node$ (since we just split node $\node$): a leaf containing $x$, and another node rooting the tree that was rooted at $\node$ before the split.
Line~13 updates the range of $\node$ using $x$ and exits the procedure.
If no split is performed, Line~15 updates the range of $\node$ using $x$ and calls $\ExtendCell$ on the child of $\node$ containing $x$.

\begin{figure}[h!]
	\begin{minipage}{0.37\textwidth}
		\begin{tikzpicture}[scale=0.25, every node/.style={scale=0.55}]

  %%% partition  
  
  \pgfmathsetmacro{\xmin}{0} % first square
  \pgfmathsetmacro{\xmax}{10}
  \pgfmathsetmacro{\ymin}{0}
  \pgfmathsetmacro{\ymax}{10}

  % \pgfmathsetmacro{\xtrans}{15} % translation between two squares  

  \pgfmathsetmacro{\xa}{2} % point a
  \pgfmathsetmacro{\ya}{8}
  \pgfmathsetmacro{\xb}{1.3} % point b
  \pgfmathsetmacro{\yb}{6.1}
  \pgfmathsetmacro{\xc}{4.5} % point c
  \pgfmathsetmacro{\yc}{7}
  \pgfmathsetmacro{\xd}{5} % point d
  \pgfmathsetmacro{\yd}{4}
  \pgfmathsetmacro{\xe}{7} % point e
  \pgfmathsetmacro{\ye}{3}

  \pgfmathsetmacro{\sy}{4.7} % root split (y)
  \pgfmathsetmacro{\sxl}{4.2} % split left node (x)
  \pgfmathsetmacro{\sxr}{8}
  \pgfmathsetmacro{\sylr}{1.5}
  \pgfmathsetmacro{\sxrl}{3}
  \pgfmathsetmacro{\sxlrr}{6.3}
  \pgfmathsetmacro{\syrll}{6.8}

  %% squares
  \draw [line width=0.7pt] (\xmin,\ymin) rectangle (\xmax,\ymax) ; % left square

  %% splits
  \draw [line width=0.7pt] (\xmin,\sy) -- (\xmax, \sy) ; % first split
  %\draw [red,line width=0.7pt] (\sxrl,\sy) -- (\sxrl, \ymax) ; % fifth
  \draw [line width=0.7pt] (\sxlrr,\ymin) -- (\sxlrr, \sy) ; % sixth
  \draw [blue,line width=0.7pt] (\xmin,\syrll) -- (\xmax, \syrll) ; % seventh

  %% points
  \draw (\xa,\ya) node {$\bullet$} node [above] {$x_2$} ;
  \draw (\xb,\yb) node {$\bullet$} node [below] {$x_3$} ;
  % \draw [red] (\xc,\yc) node {$\bullet$} ;
  % \draw [red] (\xc,\yc) node [right] {$x_  
  \draw (\xd,\yd) node {$\bullet$} node [below] {$x_1$} ;
  \draw (\xe,\ye) node {$\bullet$} node [right] {$x_4$} ;
  % \draw (\xe,\ye) node [right] {$X_s$} ;

  %% ranges  
  % \draw [red,dashed] (\xa,\yb) rectangle (\xc,\ya) ;
  % \draw [blue,dashed] (\xa,\yb) rectangle (\xb,\ya) ;
  \draw [newgreen,dashed] (\xa,\yb) rectangle (\xb,\ya) ;
  
  %%% tree representation
  \pgfmathsetmacro{\trx}{16.5} % tree root
  \pgfmathsetmacro{\try}{10}

  \pgfmathsetmacro{\ix}{2.5} % horizontal increment at level 1  
  \pgfmathsetmacro{\ixx}{1}  
  \pgfmathsetmacro{\ixxx}{0.5}  

  \pgfmathsetmacro{\iy}{-2.5} % vertical increment
  \pgfmathsetmacro{\iyy}{-2}
  \pgfmathsetmacro{\iyyy}{-1.5}

  \coordinate (E) at (\trx,\try) ; % root  
  \coordinate (L) at (\trx-\ix,\try+\iy) ;  
  \coordinate (R) at (\trx+\ix,\try+\iy) ;  
  \coordinate (RL) at (\trx+\ix-\ixx,\try+\iy+\iyy) ;
  \coordinate (RR) at (\trx+\ix+\ixx,\try+\iy+\iyy) ;
  \coordinate (RLL) at (\trx+\ix-\ixx-\ixxx,\try+\iy+\iyy+\iyyy) ;
  \coordinate (RLR) at (\trx+\ix-\ixx+\ixxx,\try+\iy+\iyy+\iyyy) ;
  \coordinate (RLLL) at (\trx+\ix-\ixx-\ixxx-\ixxx,\try+\iy+\iyy+\iyyy+\iyyy) ;
  \coordinate (RLLR) at (\trx+\ix-\ixx-\ixxx+\ixxx,\try+\iy+\iyy+\iyyy+\iyyy) ;
  \coordinate (LL) at (\trx-\ix-\ixx,\try+\iy+\iyy) ;  
  \coordinate (LR) at (\trx-\ix+\ixx,\try+\iy+\iyy) ;
  \coordinate (LRp) at (\trx-\ix+\ixx+0.3,\try+\iy+\iyy) ;
  \coordinate (LRL) at (\trx-\ix+\ixx-\ixxx,\try+\iy+\iyy+\iyyy) ;
  \coordinate (LRR) at (\trx-\ix+\ixx+\ixxx,\try+\iy+\iyy+\iyyy) ;
  \coordinate (LRRL) at (\trx-\ix+\ixx+\ixxx-\ixxx,\try+\iy+\iyy+\iyyy+\iyyy) ;
  \coordinate (LRRR) at (\trx-\ix+\ixx+\ixxx+\ixxx,\try+\iy+\iyy+\iyyy+\iyyy) ;

%  \draw (E) node {$\bullet$} ;  
  \draw [line width=0.6pt] (E) -- (L) ;
  \draw [line width=0.6pt] (E) -- (R) ;
  % \draw [red,line width=0.6pt] (R) -- (RL) ;  
  % \draw [red,line width=1.5pt] (R) -- (RR) node {$\bullet$} node [below] {$X_5$} ;
  \draw [blue,line width=0.6pt] (R) -- (RL) ;
  \draw [blue,line width=0.6pt] (R) -- (RR) ;
  \draw [line width=0.6pt] (L) -- (LL) node {$\bullet$} node [below] {$x_1$} ;
  \draw [line width=0.6pt] (L) -- (LR) node {$\bullet$} node [below] {$x_4$} ;
  \draw (RL) node {$\bullet$} node [below] {$x_3$} ;
  \draw (RR) node {$\bullet$} node [below] {$x_2$} ;

  \draw [newgreen,fill] (R) +(-3pt,-3pt) rectangle +(3pt,3pt) ;
  \draw [newgreen] (R) node [above right] {$\node = 1$} ;

  \draw (\trx,\ymin+0.7) node {Tree partition $\tree_t$ $(t=5)$};

  % \draw (LRRL)+(0.3,0) .. controls +(0.5,1) and +(1,-2) .. (LRp);
  % \draw (LR)+(0.3,0) .. controls +(-1.2,2.4) and +(-1,-1) .. (E)+(0.3,0);

\end{tikzpicture}    
  
%%% Local Variables:
%%% mode: latex
%%% TeX-master: "paper-mondrian-online"
%%% End:		
	\end{minipage}
	\begin{minipage}{0.53\textwidth}
		\begin{tikzpicture}[scale=0.25, every node/.style={scale=0.55}]

  %%% partition  
  
  \pgfmathsetmacro{\xmin}{0} % first square
  \pgfmathsetmacro{\xmax}{10}
  \pgfmathsetmacro{\ymin}{0}
  \pgfmathsetmacro{\ymax}{10}

  % \pgfmathsetmacro{\xtrans}{15} % translation between two squares  

  \pgfmathsetmacro{\xa}{2} % point a
  \pgfmathsetmacro{\ya}{8}
  \pgfmathsetmacro{\xb}{1.3} % point b
  \pgfmathsetmacro{\yb}{6.1}
  \pgfmathsetmacro{\xc}{4.5} % point c
  \pgfmathsetmacro{\yc}{7}
  \pgfmathsetmacro{\xd}{5} % point d
  \pgfmathsetmacro{\yd}{4}
  \pgfmathsetmacro{\xe}{7} % point e
  \pgfmathsetmacro{\ye}{3}

  \pgfmathsetmacro{\sy}{4.7} % root split (y)
  \pgfmathsetmacro{\sxl}{4.2} % split left node (x)
  \pgfmathsetmacro{\sxr}{8}
  \pgfmathsetmacro{\sylr}{1.5}
  \pgfmathsetmacro{\sxrl}{3}
  \pgfmathsetmacro{\sxlrr}{6.3}
  \pgfmathsetmacro{\syrll}{6.8}

  %% squares
  \draw [line width=0.7pt] (\xmin,\ymin) rectangle (\xmax,\ymax) ; % left square

  %% splits
  \draw [line width=0.7pt] (\xmin,\sy) -- (\xmax, \sy) ; % first split
  %\draw [line width=0.7pt] (\sxl,\ymin) -- (\sxl, \sy) ; % second split (only unrestricted)
  %\draw [line width=0.7pt] (\sxr,\sy) -- (\sxr, \ymax) ; % third
  %\draw [line width=0.7pt] (\sxl,\sylr) -- (\xmax, \sylr) ; % fourth
  \draw [red,line width=0.7pt] (\sxrl,\sy) -- (\sxrl, \ymax) ; % fifth
  % \draw [line width=0.7pt] (\sxlrr,\sylr) -- (\sxlrr, \sy) ; % sixth
  \draw [line width=0.7pt] (\sxlrr,\ymin) -- (\sxlrr, \sy) ; % sixth
  \draw [blue,line width=0.7pt] (\xmin,\syrll) -- (\sxrl, \syrll) ; % seventh

  %% points
  \draw (\xa,\ya) node {$\bullet$} node [above] {$x_2$} ;
  \draw (\xb,\yb) node {$\bullet$} node [below] {$x_3$} ;
  \draw [red] (\xc,\yc) node {$\bullet$} ;
  \draw [red] (\xc,\yc) node [right] {$x_5$} ;
  \draw (\xd,\yd) node {$\bullet$} node [below] {$x_1$} ;
  \draw (\xe,\ye) node {$\bullet$} node [right] {$x_4$} ;
  % \draw (\xe,\ye) node [right] {$X_s$} ;

  %% ranges  
  % \draw [red,dashed] (\xa,\yb) rectangle (\xc,\ya) ;
  % \draw [blue,dashed] (\xa,\yb) rectangle (\xb,\ya) ;
  \draw [newgreen,dashed] (\xb,\yb) rectangle (\xc,\ya) ;

  %%% tree representation
  \pgfmathsetmacro{\trx}{16.5} % tree root
  \pgfmathsetmacro{\try}{10}

  \pgfmathsetmacro{\ix}{2.5} % horizontal increment at level 1  
  \pgfmathsetmacro{\ixx}{1}  
  \pgfmathsetmacro{\ixxx}{0.5}  

  \pgfmathsetmacro{\iy}{-2.5} % vertical increment
  \pgfmathsetmacro{\iyy}{-2}
  \pgfmathsetmacro{\iyyy}{-1.5}

  \coordinate (E) at (\trx,\try) ; % root  
  \coordinate (L) at (\trx-\ix,\try+\iy) ;  
  \coordinate (R) at (\trx+\ix,\try+\iy) ;  
  \coordinate (RL) at (\trx+\ix-\ixx,\try+\iy+\iyy) ;
  \coordinate (RR) at (\trx+\ix+\ixx,\try+\iy+\iyy) ;
  \coordinate (RLL) at (\trx+\ix-\ixx-\ixxx,\try+\iy+\iyy+\iyyy) ;
  \coordinate (RLR) at (\trx+\ix-\ixx+\ixxx,\try+\iy+\iyy+\iyyy) ;
  \coordinate (RLLL) at (\trx+\ix-\ixx-\ixxx-\ixxx,\try+\iy+\iyy+\iyyy+\iyyy) ;
  \coordinate (RLLR) at (\trx+\ix-\ixx-\ixxx+\ixxx,\try+\iy+\iyy+\iyyy+\iyyy) ;
  \coordinate (LL) at (\trx-\ix-\ixx,\try+\iy+\iyy) ;  
  \coordinate (LR) at (\trx-\ix+\ixx,\try+\iy+\iyy) ;
  \coordinate (LRp) at (\trx-\ix+\ixx+0.3,\try+\iy+\iyy) ;
  \coordinate (LRL) at (\trx-\ix+\ixx-\ixxx,\try+\iy+\iyy+\iyyy) ;
  \coordinate (LRR) at (\trx-\ix+\ixx+\ixxx,\try+\iy+\iyy+\iyyy) ;
  \coordinate (LRRL) at (\trx-\ix+\ixx+\ixxx-\ixxx,\try+\iy+\iyy+\iyyy+\iyyy) ;
  \coordinate (LRRR) at (\trx-\ix+\ixx+\ixxx+\ixxx,\try+\iy+\iyy+\iyyy+\iyyy) ;

%  \draw (E) node {$\bullet$} ;  
  \draw [line width=0.6pt] (E) -- (L) ;
  \draw [line width=1.5pt] (E) -- (R) ;
  \draw [red,line width=0.6pt] (R) -- (RL) ;
  \draw [red,line width=1.5pt] (R) -- (RR) node {$\bullet$} node [below] {$x_5$} ;
  \draw [blue,line width=0.6pt] (RL) -- (RLL) ;
  \draw [blue,line width=0.6pt] (RL) -- (RLR) ;
  % \draw [line width=0.6pt] (RLL) -- (RLLL) node {$\bullet$} ;
  % \draw [line width=0.6pt] (RLL) -- (RLLR) node {$\bullet$} ;
  \draw [line width=0.6pt] (L) -- (LL) node {$\bullet$} node [below] {$x_1$} ;
  \draw [line width=0.6pt] (L) -- (LR) node {$\bullet$} node [below] {$x_4$} ;
  % \draw [red,line width=0.6pt] (LR) -- (LRL) node {\textcolor{white}{$\bullet$}} node {$\circ$} ;
  %\draw [red,line width=1.5pt] (LR) -- (LRR) ;
  %\draw [red,line width=1.5pt] (LRR) -- (LRRL) node {$\bullet$} node [below] {$x_t$} ;
  %\draw [red,line width=0.6pt] (LRR) -- (LRRR) ;
  %\draw (LRRR) node {$\bullet$} node [below] {$X_s$} ;
  \draw (RLL) node {$\bullet$} node [below] {$x_3$} ;
  \draw (RLR) node {$\bullet$} node [below] {$x_2$} ;

  \draw [newgreen,fill] (R) +(-3pt,-3pt) rectangle +(3pt,3pt) ;
  \draw [newgreen] (R) node [above right] {$\node = 1$} ;

  \draw (\trx + 1,\ymin+0.7) node {Updated restricted tree partition
    \textcolor{red}{$\tree_{t+1}$}};

  % \draw (LRRL)+(0.3,0) .. controls +(0.5,1) and +(1,-2) .. (LRp);
  % \draw (LR)+(0.3,0) .. controls +(-1.2,2.4) and +(-1,-1) .. (E)+(0.3,0);

  %%% updates
  \pgfmathsetmacro{\xaxis}{22.5}
  \pgfmathsetmacro{\xpp}{0.5}
  \pgfmathsetmacro{\yaxismin}{5.5}
  \pgfmathsetmacro{\yaxismax}{10}
  \pgfmathsetmacro{\yaxisa}{9}
  \pgfmathsetmacro{\yaxisb}{8}
  \pgfmathsetmacro{\yaxisc}{7}
  \pgfmathsetmacro{\yaxisd}{6}

  \coordinate (Amin) at (\xaxis,\yaxismin) ;
  \coordinate (Amax) at (\xaxis,\yaxismax) ;
  
  \draw[-{latex[length=2mm, width=1mm]},dashed] (Amin) -- (Amax) ;
  % \draw (Amax) node {$\mathbf{\uparrow}$} ;
  \draw (\xaxis+\xpp,\yaxisa) node [right] {Updates along the \textbf{path} of $x_t$:} ;
  \draw (\xaxis+\xpp,\yaxisb) node [right] {${w_{\node, t+1}} = w_{\node, t} \exp(-\ell (\pred_{\node, t}, y_t))$} ;
  \draw (\xaxis+\xpp,\yaxisc) node [right] {${w_{\node, t+1}} = \frac{1}{2} w_{\node, t+1} + \frac{1}{2} \wbar_{\node 0, t+1} \wbar_{\node 1, t+1}$} ;
  \draw (\xaxis+\xpp,\yaxisd) node [right] {${\pred_{\node, t+1}} = \dots$} ;

\end{tikzpicture}    
  
%%% Local Variables:
%%% mode: latex
%%% TeX-master: "paper-mondrian-online"
%%% End:
	\end{minipage}  
  \caption{Illustration of the $\MondrianUpdate(x_t, y_t)$ procedure from Algorithms~\ref{alg:mondrian-update-experiments} and~\ref{alg:extend-split}: update of the partition, weights and node predictions as a new data point $(x_t, y_t)$ for $t=5$ becomes available.
    \emph{Left}: tree partition $\mondrian_t$ before seeing $(x_t, y_t)$. \emph{Right}:
    update of the partition using $(x_t, y_t)$. The path of $x_t$ in the tree is indicated in bold. 
    In green is the node $\node = 1$ and dashed lines indicates its range $R_1$. Since $x_5$ is outside of $R_1$ at $t=5$, the range is extended. 
    A new split (in red) is sampled in the extended range, since its creation time $\birth_1 + E$ is smaller than the one of the next split $\birth_{10}$ and two new nodes named $10$ and $11$ are inserted below $1$, while the previous nodes $10$ and $11$ are moved as $100$ and $101$.
    The weights and node predictions are then updated using an upwards path from the new leaf containing $x$ to the root (in bold). All leaves contain exactly one single point.
  }
  \label{fig:algorithm-restricted}
\end{figure}
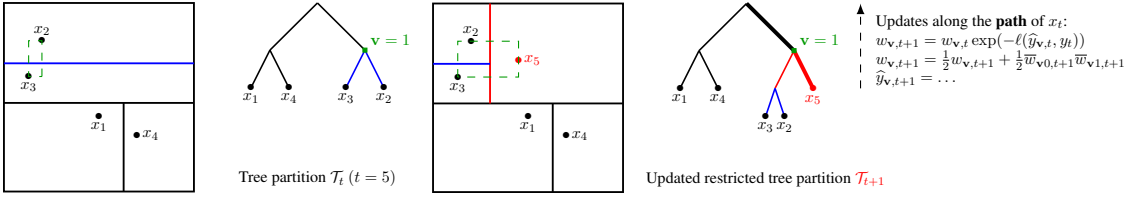

The prediction algorithm described in Algorithm~\ref{alg:predict-experiments} below is a modification of Algorithm~\ref{alg:predict}, where we use $\ExtendCell$ instead of Algorithm~\ref{alg:mondrian-update}.
Finally, the algorithm used in our experiments do not use the online to batch conversion from Section~\ref{sub:adaptive-rates}: it simply uses the current tree, namely the most recent updated Mondrian tree partition $\Pi_{t+1} = (\tree_{t+1}, \Sigma_{t+1})$ after calling $\MondrianUpdate(x_1, y_1), \ldots ,\MondrianUpdate(x_t, y_t)$, where $(x_t, y_t)$ is the last sample seen.

\begin{algorithm}[htbp]
  \small
  \caption{$\Predict (x)$: predict the label of $x \in [0, 1]^d$}
  % \label{alg:predict}
  \begin{algorithmic}[1]
  \STATE \textbf{Input:} a features vector $x \in [0, 1]^d$
    \STATE Call $\ExtendCell(\root, x)$ in order to obtain a temporary update of the current partition $\Pi$ using $x$ and let $\node(x)$ be the leaf such that $x \in C_{\node(x)}$
    \STATE Set $\wt y_\node = \pred_{\node(x)}$
    \WHILE{$\node \neq \root$}
    \STATE Let $(\node, \node a) = (\parent{\node}, \node)$ (for some $a\in \{ 0, 1\}$)
    \STATE Let
    $\wt y_\node = \frac{1}{2} \frac{w_\node}{\wbar_{\node}} \pred_{\node} + \frac{1}{2} \frac{\wbar_{\node (1-a)} \wbar_{\node a}}{\wbar_\node} \wt y_{\node a}$
    \ENDWHILE
    \STATE \textbf{Return} $\wt y_\root$
  \end{algorithmic}
  \label{alg:predict-experiments}
\end{algorithm}

\subsection{Computational complexity}
\label{sub:complexity}

The next Proposition provides a bound on the average depth of a Mondrian tree.
This is of importance, since the computational complexities of $\MondrianUpdate$ and $\Predict$ are linear with respect to this depth, see below for a discussion.

\begin{proposition}
  \label{prop:mondrian_depth_bound}
  Assume that $x$ has a density $p$ satisfying the following property\textup: there exists a constant $M > 0$ such that, for every $x', x'' \in [0, 1]^d$ which only differ by one coordinate\textup,
  \begin{equation}
    \label{eq:ratio-density}
    \frac{p(x')}{p(x'')}
    \leq M
    \, .
  \end{equation}
  Then\textup, the depth $D_n^\mondrian (x)$ of the leaf containing a random point $x$ in the Mondrian tree restricted to the observations $x_1, \dots, x_n, x$ satisfies
  \begin{equation*}
    \E [ D_n^\mondrian (x) ]
    \leq \frac{\log n}{\log [ (2M) / (2M - 1) ]} + 2M
    \, .
  \end{equation*}
\end{proposition}

Assumption~\eqref{eq:ratio-density} is satisfied when $p$ is upper and lower bounded: $c \leq p \leq C$ with $M = C / c$, but this assumption is weaker: for instance, it only implies that $M^{-d} \leq p \leq M^d$, which is a milder property when the dimension $d$ is large.
The proof of Proposition~\ref{prop:mondrian_depth_bound} is given in Section~\ref{sec:proofs}.
Since the lower bound is also trivially $\Omega(\log n)$ (a binary tree with $n$ nodes has at least a depth $\log_2 n$), Proposition~\ref{prop:mondrian_depth_bound} entails that
 $\E[D_n^\mondrian] = \Theta(\log n)$.
 If the number of features is $d$, then the update complexity of a single tree is $\Theta(d \log n)$, which makes full online training time $\Theta(n d \log n)$ over a dataset of size $n$.
 %% Anytime
 Prediction is $\Theta(\log n)$ since it requires a downwards and upwards path on the tree (see Algorithms~\ref{alg:predict} and~\ref{alg:predict-experiments}).

\section{Numerical experiments}
\label{sec:num-experiments}

The aim of this section is two-fold. 
First, Section~\ref{sub:numerical_insights} gathers some insights about AMF based on numerical evidence obtained with simulated and real datasets, that confirm our theoretical findings.
Second, it proposes in Section~\ref{sub:comparison_datasets} a thorough comparison of AMF with several baselines on several datasets for multi-class classification.

An open source implementation of AMF (AMF with restricted Mondrian partitions described in Algorithms~\ref{alg:mondrian-update-experiments} and~\ref{alg:predict-experiments}) is available in the \texttt{onelearn} Python package, where algorithms for multi-class classification and regression are available, respectively, in the \texttt{AMFClassifier} and \texttt{AMFRegressor} classes.
The source code of \texttt{onelearn} is available at \url{https://github.com/onelearn/onelearn} and can be installed easily through the \texttt{PyPi} repository by typing \texttt{pip install onelearn} in a terminal. The documentation of \texttt{onelearn} is available at \url{https://onelearn.readthedocs.io}, which contains extra experiments and explains, among other things, how the experiments from the paper can be reproduced.
The \texttt{onelearn} package is fully implemented in Python, but is heavily accelerated thanks to \texttt{numba}\footnote{http://numba.pydata.org} and follows API conventions from \texttt{scikit-learn}, see~\cite{pedregosa2011scikit-learn}.

\subsection{Main insights about AMF}
\label{sub:numerical_insights}

First, let us provide some key observations about AMF.

\paragraph{A purely online algorithm.}

AMF is a \emph{purely online} algorithm, as illustrated in Figure~\ref{fig:omaf_iterations} on a simulated binary classification problem.
Herein, we see that the decision function of AMF evolves smoothly along the learning steps, leading to a correct AUC even in the early stages.
This is confirmed from a theoretical point of view in Section~\ref{sec:theory}, which provides regret bounds and minimax rates, but also from a computational point of view in Table~2 from Section~\ref{sec:algorithm}, where we show that at each step, both learning and prediction have $\Theta(\log n)$ complexity, where $n$ stands for the number of samples currently available. 
\begin{figure}[htbp]
  \centering
  \includegraphics[width=\textwidth]{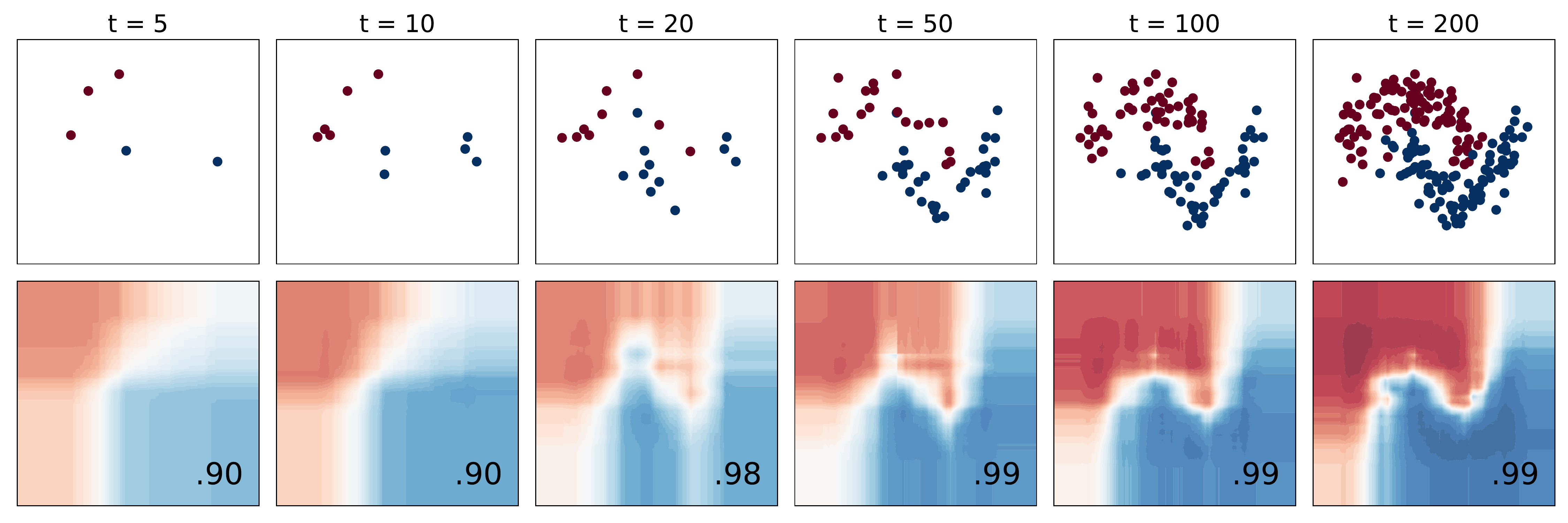}
  \caption{Evolution of the decision function of AMF along the online learning steps on a simulated binary classification problem. We observe the online property of this algorithm, which produces a smooth decision function at each iteration, and leads to a correct AUC on a test set even in the early stages (bottom right of each display).}
  \label{fig:omaf_iterations}
\end{figure}

\paragraph{AMF is adaptive.}

We know from Section~\ref{sec:theory} that a tree in AMF controls its regret with respect to the best pruning of a Mondrian partition and is consequently adaptive to the unknown smoothness of the regression function.
\begin{figure}[ht!]
  \centering
  \includegraphics[width=0.49\textwidth]{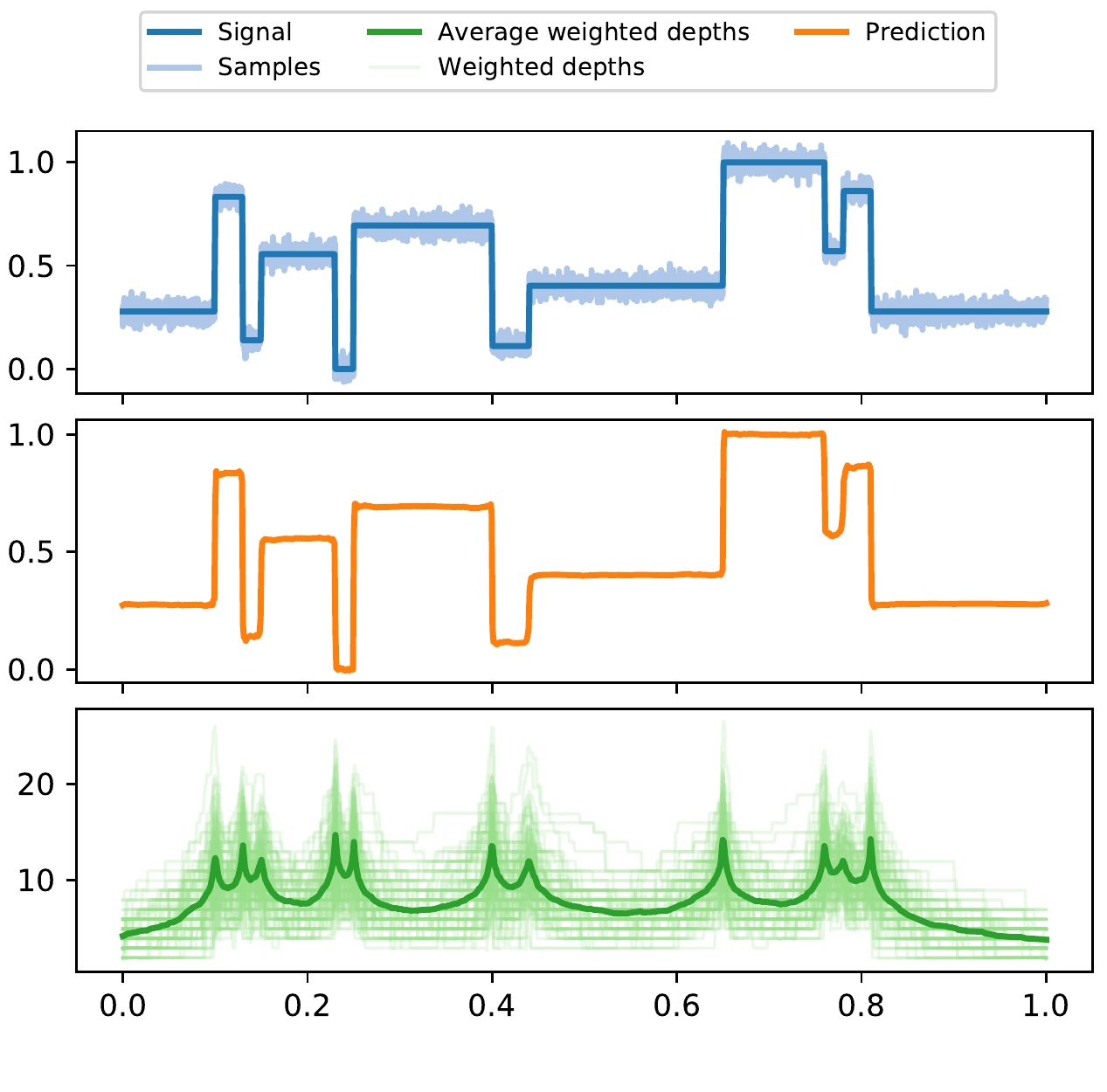}
  \includegraphics[width=0.49\textwidth]{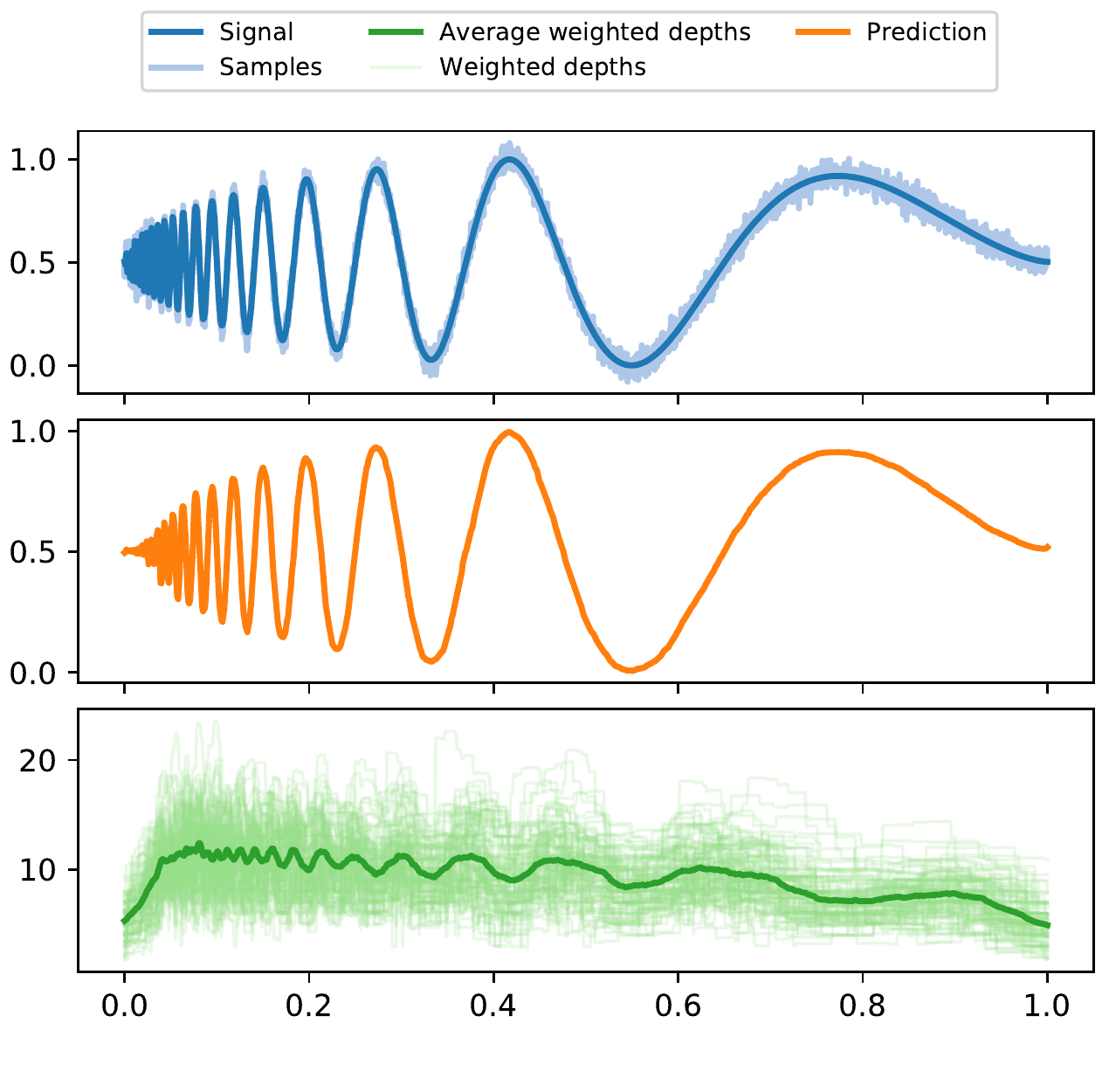}
  \caption{Two examples (left and right hand sides) of a true signal and noisy Gaussian samples (top), the reconstructed signal (middle) and the local weighted depths of 100 trees used in AMF and their average (bottom), over the interval $[0, 1]$.
  We observe that the weighted depths vary strongly over $[0, 1]$ as a function of the signal smoothness: wherever the signal is unsmooth, the weighted depths used by AMF increase and decrease where the signal is smooth.}
  \label{fig:weighted_depths}
\end{figure}
This fact is confirmed by Figure~\ref{fig:weighted_depths}, where we consider two examples of one-dimensional ($d=1$) regression problems with Gaussian noise.

In both displays, we compute the \emph{local weighted depths} denoted $\mathrm{wdepth}(x)$ of each tree in the forest as
\begin{equation*}
 \mathrm{wdepth}(x) = \sum_{\tree} w(\tree) \mathrm{depth}_\tree(x) \quad \text{ with } \quad w(\tree) = \frac{\pi (\tree) e^{-\eta L_{t-1} (\tree)} }{\sum_{\tree'} \pi (\tree') e^{-\eta L_{t-1} (\tree')}},
\end{equation*}
where the sum is over all subtrees $\tree$ of the considered tree, where the \emph{prior} $\pi$ is given by~\eqref{eq:ctw-prior}, where $x \in [0, 1]^d$ and where $\mathrm{depth}_\tree(x)$ stands for the depth of each subtree $\tree$ met along the path leading to the leaf containing $x$.
Note that the aggregation weights $w(\tree)$ are the same as in the aggregated estimator from Equation~\eqref{eq:exact-aggregation}.
When the aggregation weight $w(\tree)$ of a tree $\tree$ is large ($w(\tree) \approx 1$), it carries most of the prediction computed by~\eqref{eq:exact-aggregation}, and $\mathrm{wdepth}(x) \approx \mathrm{depth}_\tree(x)$.
In such a case, by displaying $\mathrm{wdepth}(x)$ along $x \in [0, 1]^d$, we can visualize the local complexities (measured by the weighted depth) used internally by AMF.
This leads to the displays of Figure~\ref{fig:weighted_depths}, where we can observe that AMF adapts to the local smoothness of the regression functions.
We see that the weighted depth increases at points $x \in [0, 1]$ where the signal is unsmooth.
This means that AMF tends to use deeper trees at such points, with deeper leaves, containing less samples points, leading to prediction based on narrow averages.
On the contrary, AMF decreases the weighted depth where the signal is smooth, leading to tree predictions with leaves containing more samples, hence wider averages.

\paragraph{Aggregation prevents overfitting.}

Thanks to the aggregation algorithm, each tree in AMF is adaptive and leads to smooth decision functions for classification. 
We display in Figure~\ref{fig:decisions} the decision functions of AMF, Mondrian Forest (MF), Breiman's batch random forest (RF) and batch Extra Trees (ET) on simulated datasets for binary classification (see Section~\ref{sub:comparison_datasets} for a precise description of the implementations used for each algorithm).
\begin{figure}[ht]
  \centering
  \includegraphics[width=0.95\textwidth]{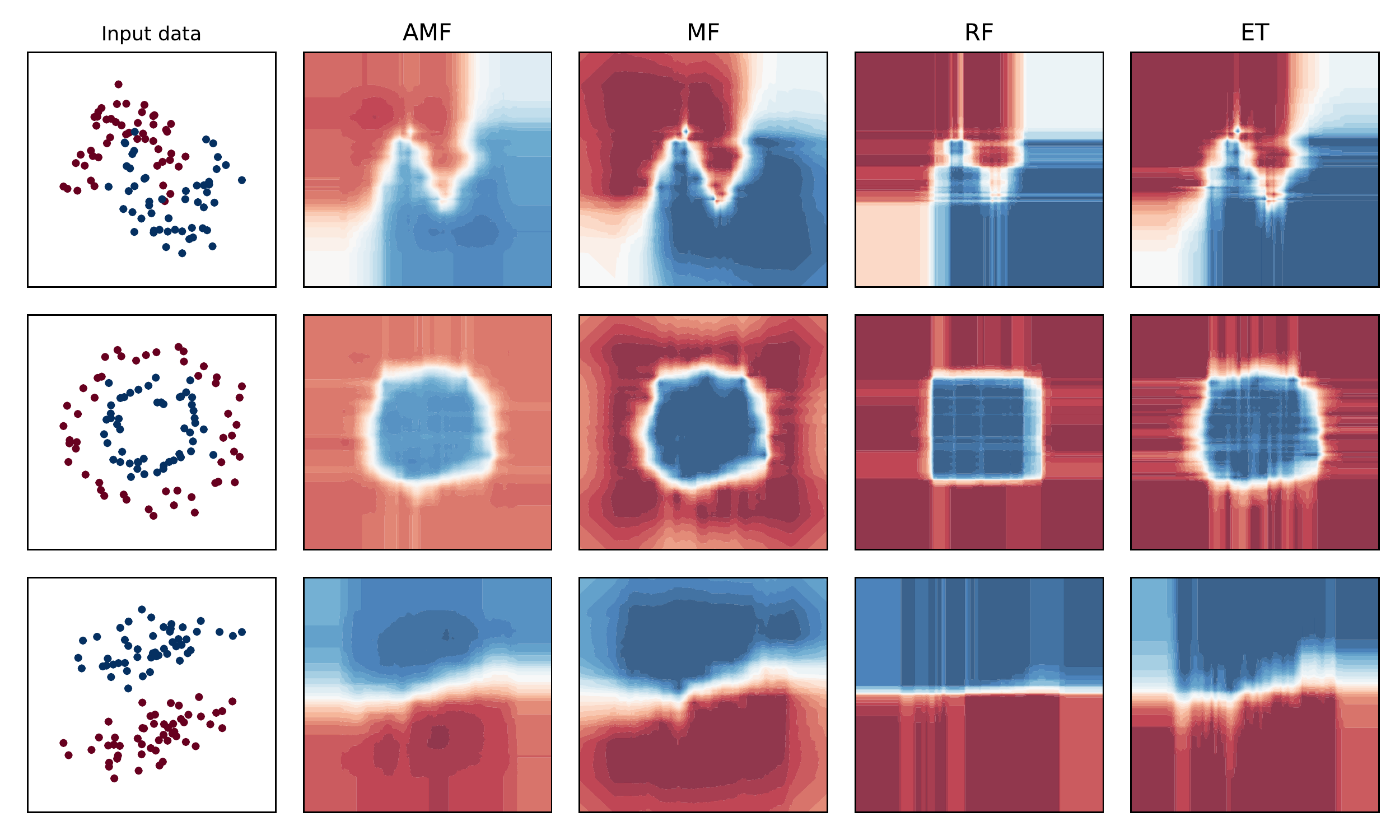}
  \caption{  Decision functions of AMF, Mondrian Forest (MF), Breiman's batch random forest (RF) and batch Extra Trees (ET), on several toy datasets for binary classification. We observe that AMF, thanks to aggregation, leads to a smooth decision function, hence with a better generalization properties. Let us stress that both AMF and MF do a \emph{single pass} on the data, while RF and ET require many passes. All algorithms use a forest containing 100 trees.}
  \label{fig:decisions}
\end{figure}
We observe that AMF produces a smooth decision function in all cases, while
all the other algorithms display rather non-smooth decision functions, which suggests that the underlying probability estimates are not well-calibrated.
Note also that, thanks to the aggregation algorithm, AMF obtains typically good performances with a small number of trees, as illustrated in Figure~\ref{fig:n_trees} from Section~\ref{sub:comparison_datasets}.

These numerical insights are confirmed in the next Section, where we propose a thorough comparison of AMF with several baselines on real datasets, and where we can observe that AMF compares favorably with respect to the considered baselines.

\subsection{Comparison of AMF with several baselines on real datasets} % (fold)
\label{sub:comparison_datasets}

We describe all the considered algorithms in Section~\ref{sub:description_of_the_considered_procedures}, including both online methods and batch methods.
A comparison of the {average losses} (assessing the online performance of algorithms) on several datasets is given in Section~\ref{sub:comparison_of_regrets} for online methods only. 
Batch and online methods are compared in Section~\ref{sub:comparisons_for_batch_learning} and an experiment comparing the sensitivity of all methods with respect to the number of trees used is given in Section~\ref{sub:sensitivity_to_the_number_of_trees}.
All these experiments are performed for multi-class classification problems, using datasets described in Section~\ref{sub:considered_datasets}.

\subsubsection{Algorithms}
\label{sub:description_of_the_considered_procedures}

In this Section, we describe precisely the procedures considered in our experiments for online and batch classification problems.

\medskip
\noindent
\emph{AMF.} We use \texttt{AMFClassifier} from the \texttt{onelearn} library, with all default parameters: we use $10$ trees, a learning rate $\eta = 1$, we don't split nodes containing only a single class of labels.

\medskip
\noindent
\emph{Dummy.} We consider a dummy baseline that only estimates the distribution of the labels (without taking into account the features) in an online manner, using \texttt{OnlineDummyClassifier} from \texttt{onelearn}.
At step $t + 1$, it simply computes the Krichevsky-Trofimov forecaster (see Equation~\eqref{eq:kt-predictor}) $\widehat y_{t+1}(k) = (n_{t}(k) + 1/2) / (t + K / 2)$ of the classes $k=1, \ldots, K$, where $n_{t}(k) = \sum_{s=1}^{t} \indic{y_s = k}$.

\medskip
\noindent
\emph{MF (Mondrian Forest).} The Mondrian Forest algorithm from~\cite{lakshminarayanan2014mondrianforests,lakshminarayanan2016mondrianuncertainty} proposed in the \texttt{scikit-garden}  library\footnote{\url{https://github.com/scikit-garden/scikit-garden}}.
We use \texttt{MondrianForestClassifier} in our experiments, with the default settings proposed with the method: 
$10$ trees are used, no depth restriction is used on the trees, all trees are trained using the entire dataset (no bootstrap).

\medskip
\noindent
\emph{SGD (Stochastic Gradient Descent).}
This is logistic regression trained with a single pass of stochastic gradient descent.
We use \verb|SGDClassifier| from the \texttt{scikit-learn} library\footnote{\url{https://scikit-learn.org}}, see~\cite{pedregosa2011scikit-learn}.
We use a constant learning rate $0.1$ and the default choice of ridge penalization with strength $0.0001$, 
since it provides good results on all the datasets.

\medskip
\noindent
\emph{RF (Random Forests).}
This is Random Forest~\citep{breiman2001randomforests} for classification.
We use the implementation available in the \texttt{scikit-learn} library, 
namely \texttt{RandomForestClassifer} from the \texttt{sklearn.ensemble} module.
This is a reference implementation, which is highly optimized and 
among the fastest implementations available in the open-source community. Details on this implementation are available in~\cite{louppe2014understanding}.
Note that this is a batch algorithm, that cannot be trained sequentially, which requires a large number of passes through the data to optimize some impurity criterion (default is the Gini index). 
We use the default parameters of the procedure, with 10 trees.

\medskip
\noindent
\emph{ET (Extra Trees).}
This is the Extra Trees algorithm~\citep{geurts2006extremely}. Once-again, we use the implementation available in the \texttt{scikit-learn} library, namely \texttt{ExtraTreesClassifier} from the \texttt{sklearn.ensemble} module. 
As for RF, it is a reference implementation from the open source community.
We use the default parameters of the procedure, with 10 trees.

\subsubsection{Considered datasets} % (fold)
\label{sub:considered_datasets}

The datasets used in this paper are from the UCI Machine Learning repository, see~\cite{Dua:2019} and are described in Table~\ref{tab:datasets_classif}.%

\begin{table}
	\caption{\label{tab:datasets_classif} \footnotesize List of datasets from the UCI Machine Learning repository considered in our experiments.}
	\centering
	\fbox{%
		\begin{tabular}{llrr}
			\textbf{dataset} &  \textbf{\#samples} &  \textbf{\#features} &  \textbf{\#classes} \\
			\midrule
			  adult &      32561 &         107 &          2 \\
			  bank &      45211 &          51 &          2 \\
			  car &       1728 &          21 &          4 \\
			  cardio &       2126 &          24 &          3 \\
			  churn &       3333 &          71 &          2 \\
			  default\_cb &      30000 &          23 &          2 \\
			  letter &      20000 &          16 &         26 \\
			  satimage &       5104 &          36 &          6 \\
			  sensorless &      58509 &          48 &         11 \\
			  spambase &       4601 &          57 &          2 \\
			\bottomrule
	\end{tabular}}
\end{table}

\subsubsection{Online learning: comparison of averaged losses} % (fold)
\label{sub:comparison_of_regrets}

We compare the curves of averaged losses over time of all the considered online algorithms.
At each round $t$, we reveal a new sample $(x_t, y_t)$ and update all algorithms using this new sample. 
Then, we ask all algorithms to give a prediction $\pred_{t+1}$ of the label $y_{t+1}$ associated to $x_{t+1}$, and compute the log-loss $\ell(\pred_{t+1}, y_{t+1})$ incurred by all algorithms.
Along the rounds $t=1, \ldots, n-1$ when the considered data has sample size $n$, we compute the average loss $\frac{1}{t-1} \sum_{s=1}^{t-1} \ell(\pred_{s+1}, y_{s+1})$.
This is what is displayed in Figure~\ref{fig:all_regrets} below, on 10 datasets for the online procedures AMF, Dummy, SGD and MF. % (we don't display Dummy on \texttt{covtype} since its performs too poorly).

\begin{figure}[htbp]
  \centering
  \includegraphics[width=0.195\textwidth]{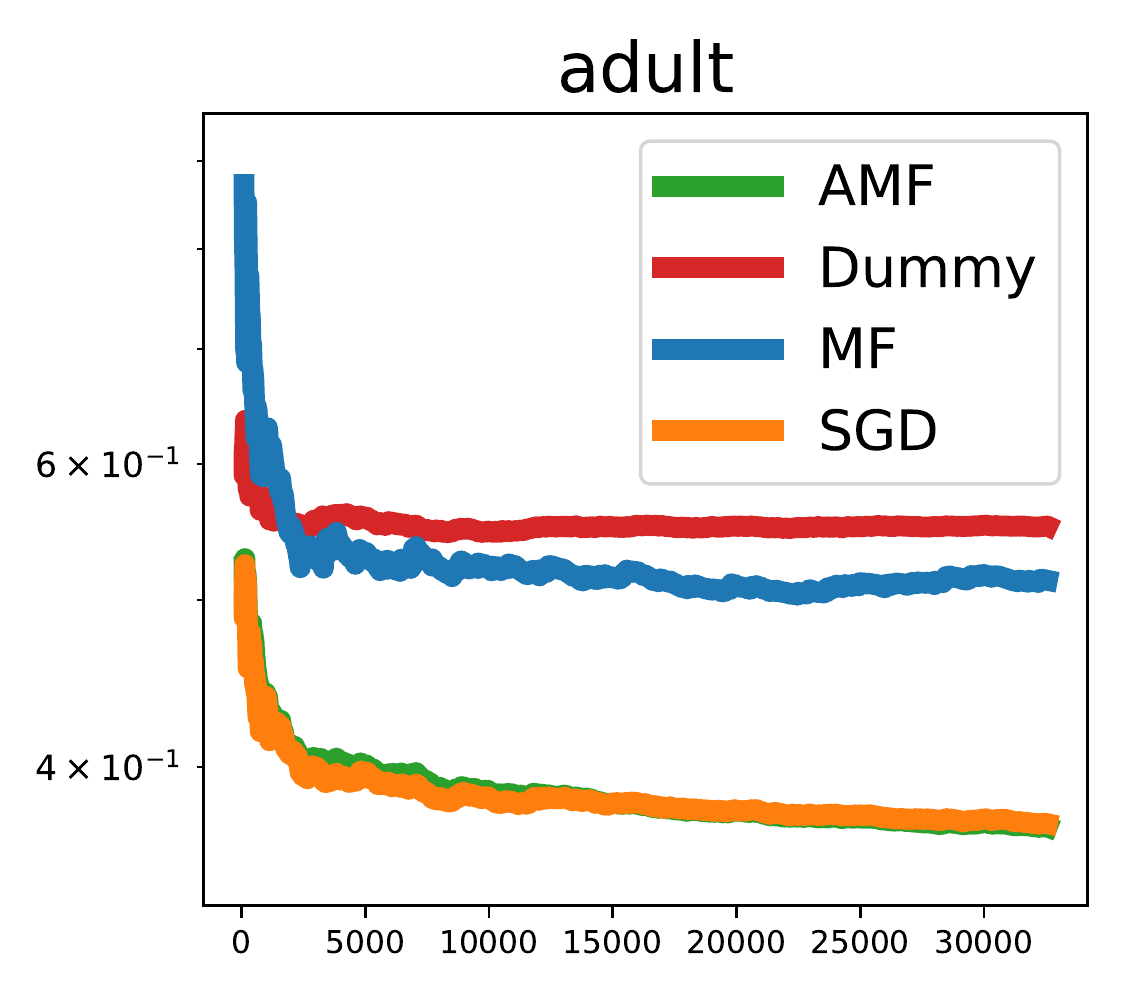}
  \includegraphics[width=0.195\textwidth]{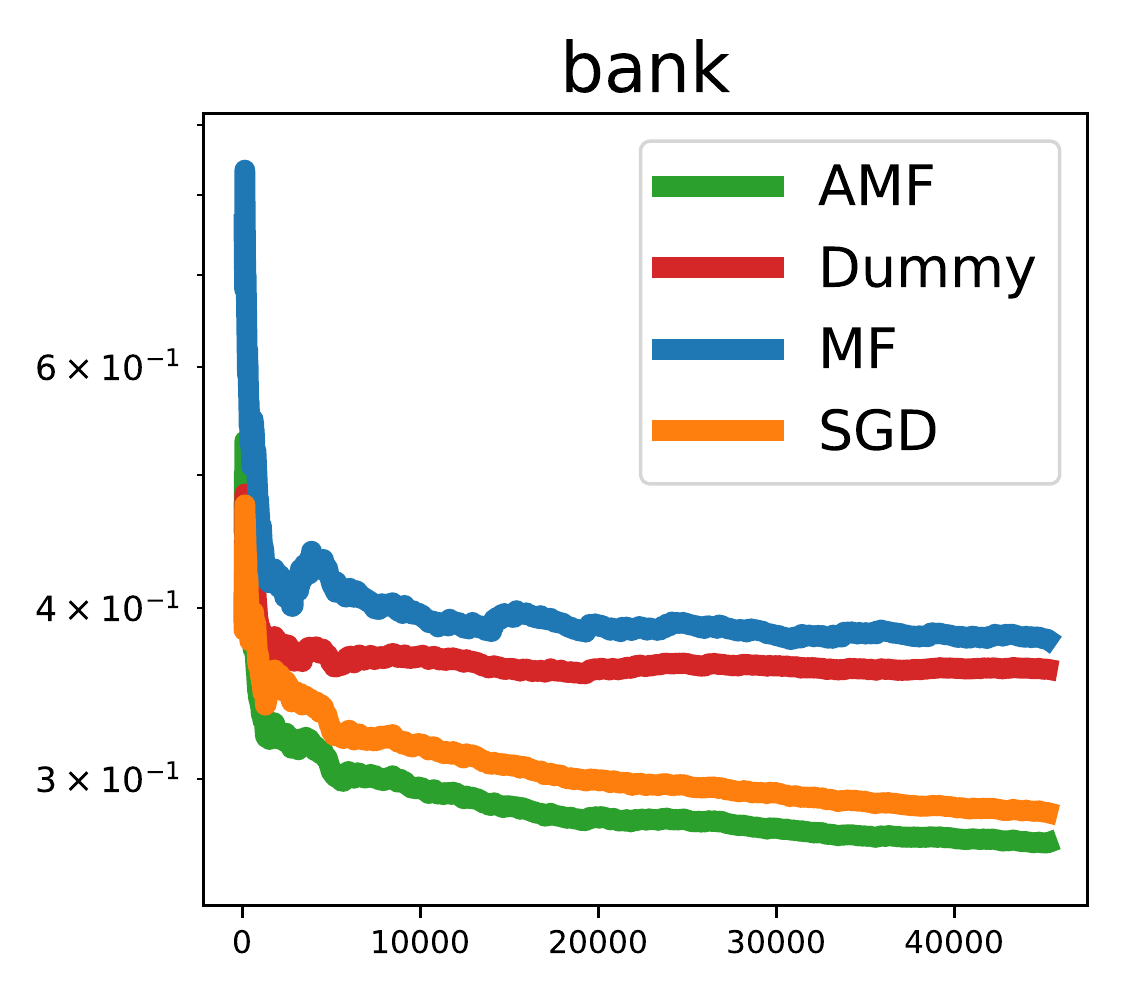}%
  \includegraphics[width=0.195\textwidth]{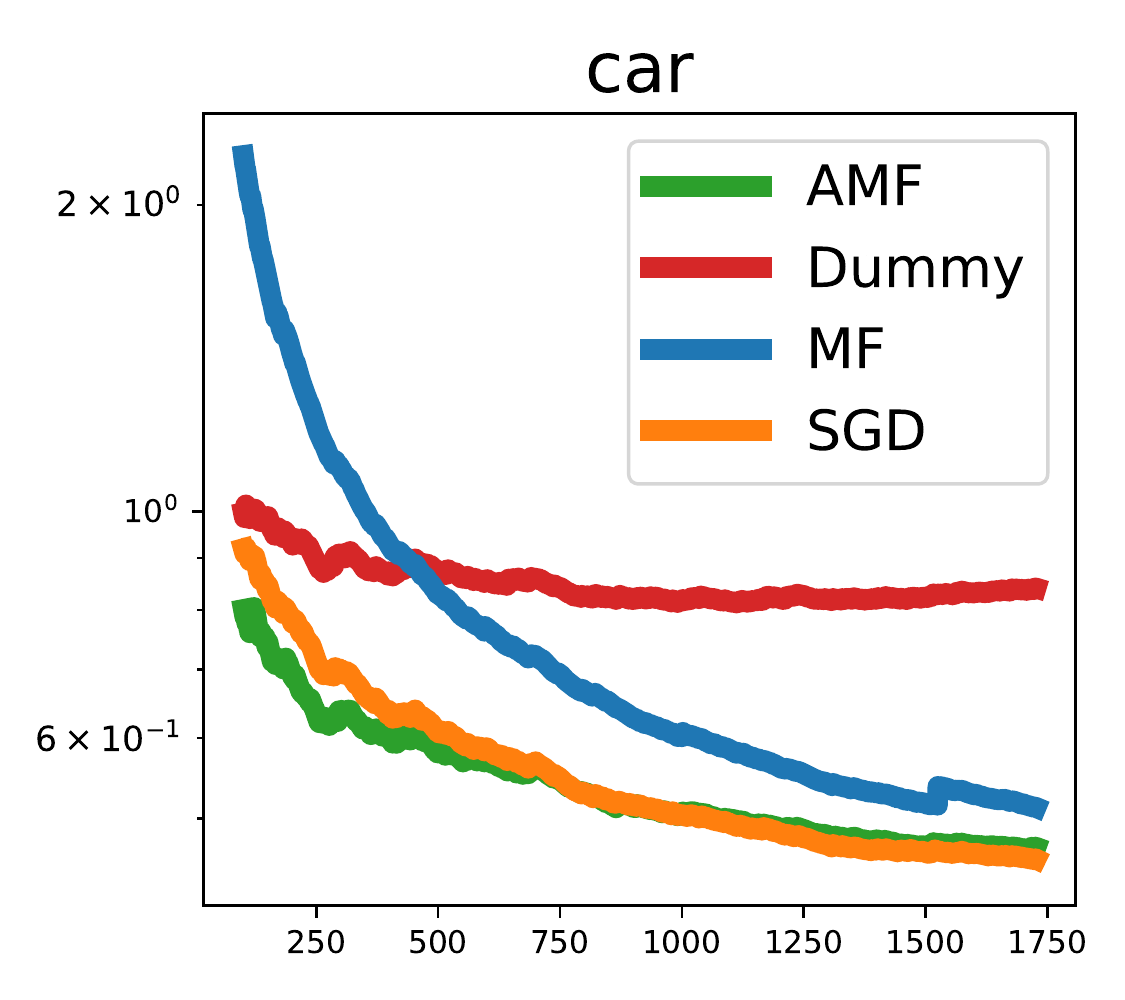}
  \includegraphics[width=0.195\textwidth]{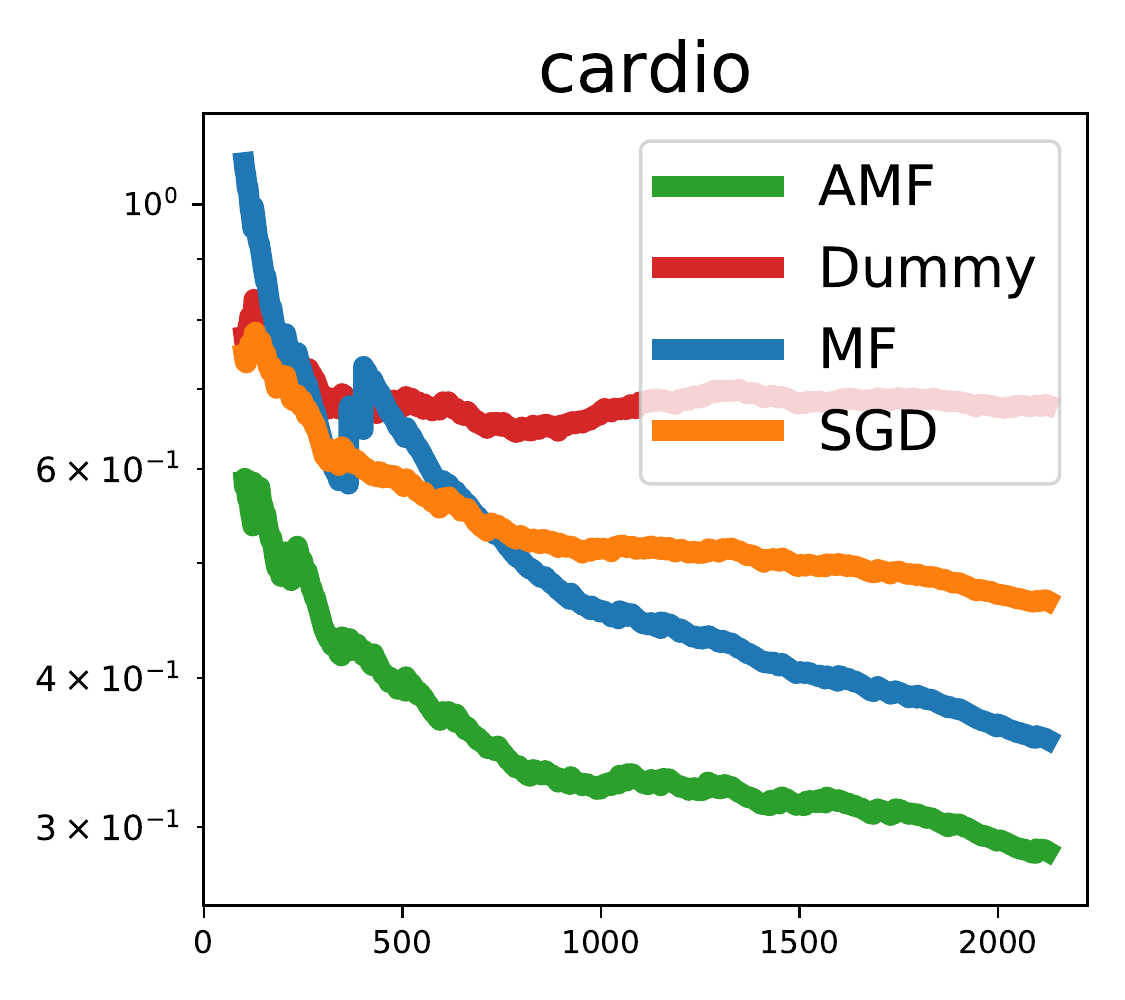}
  \includegraphics[width=0.195\textwidth]{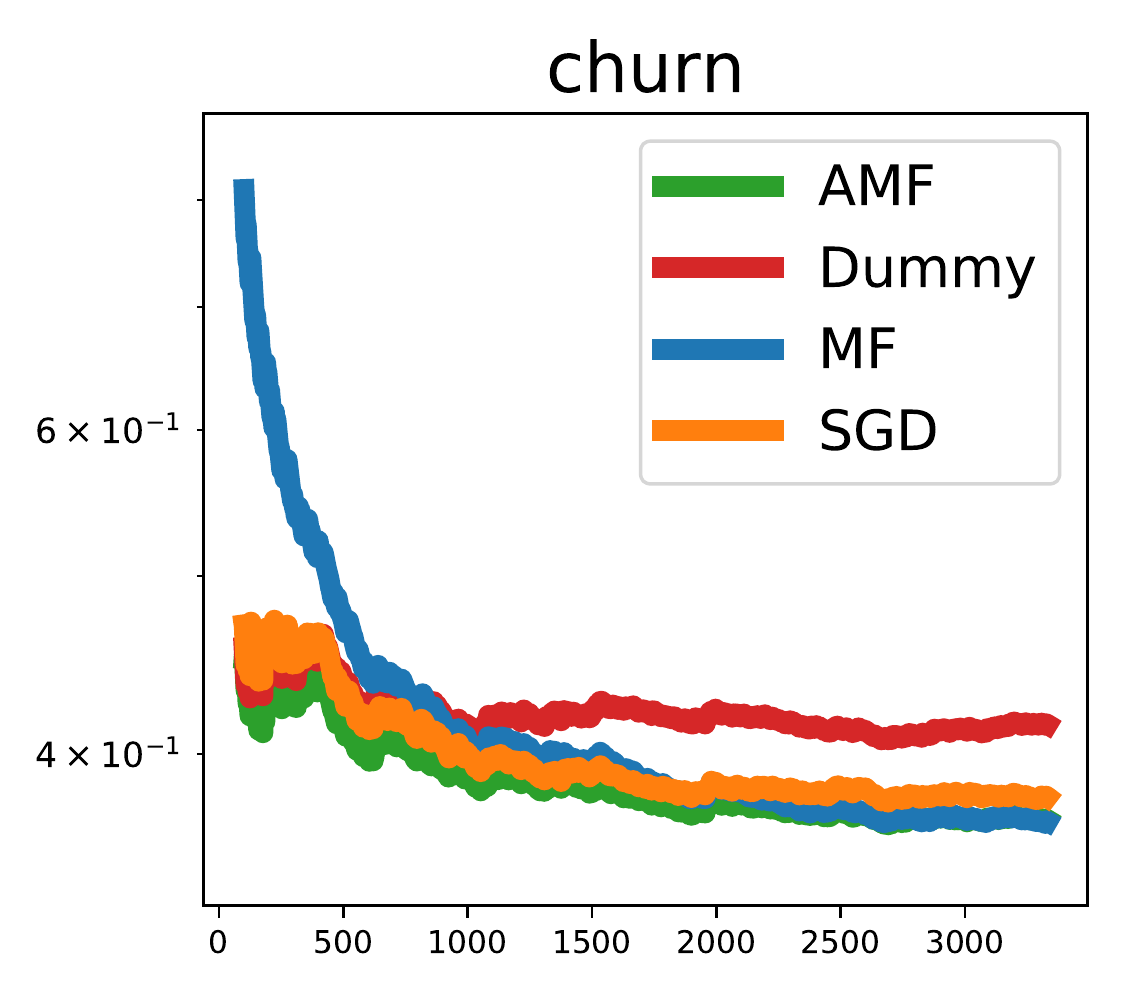}%

  \includegraphics[width=0.195\textwidth]{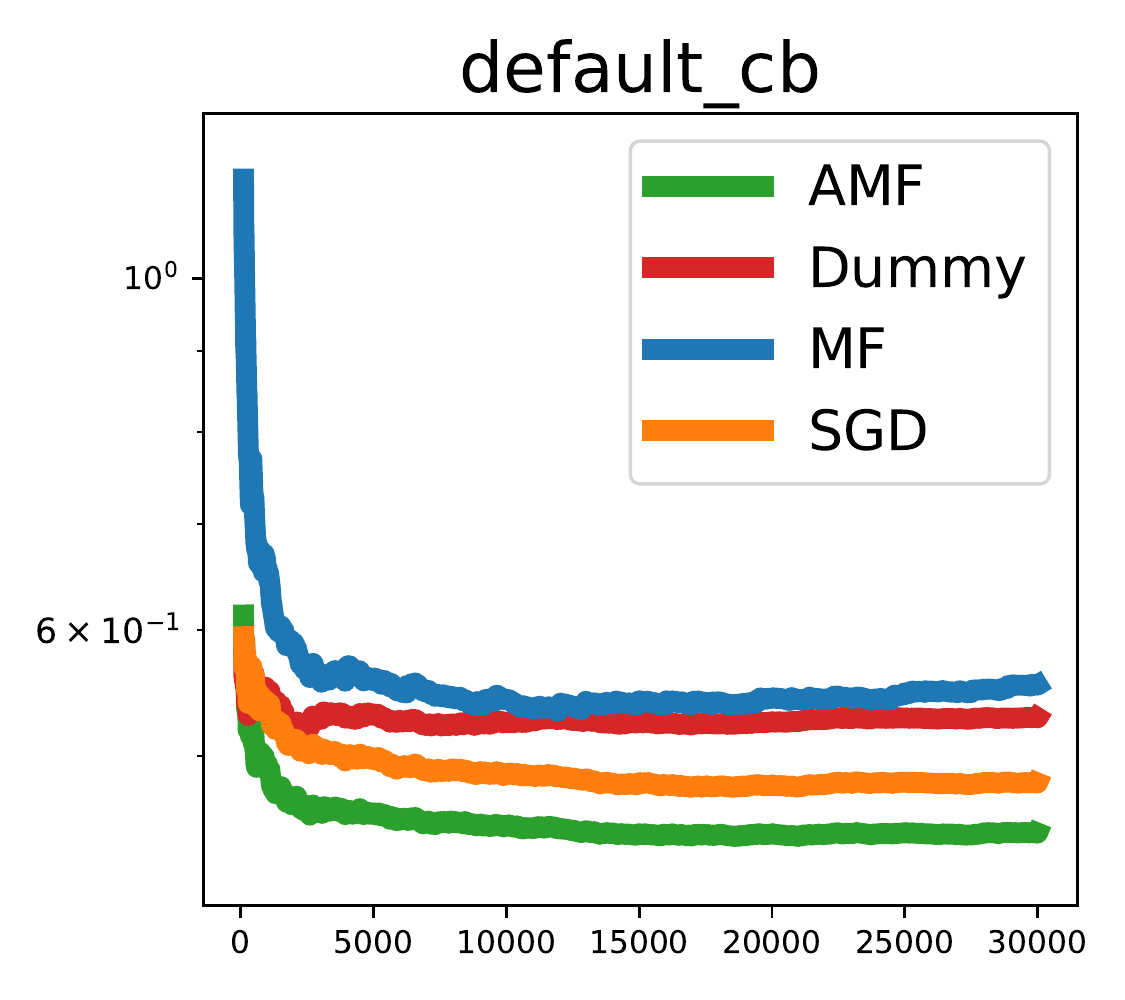}
  \includegraphics[width=0.195\textwidth]{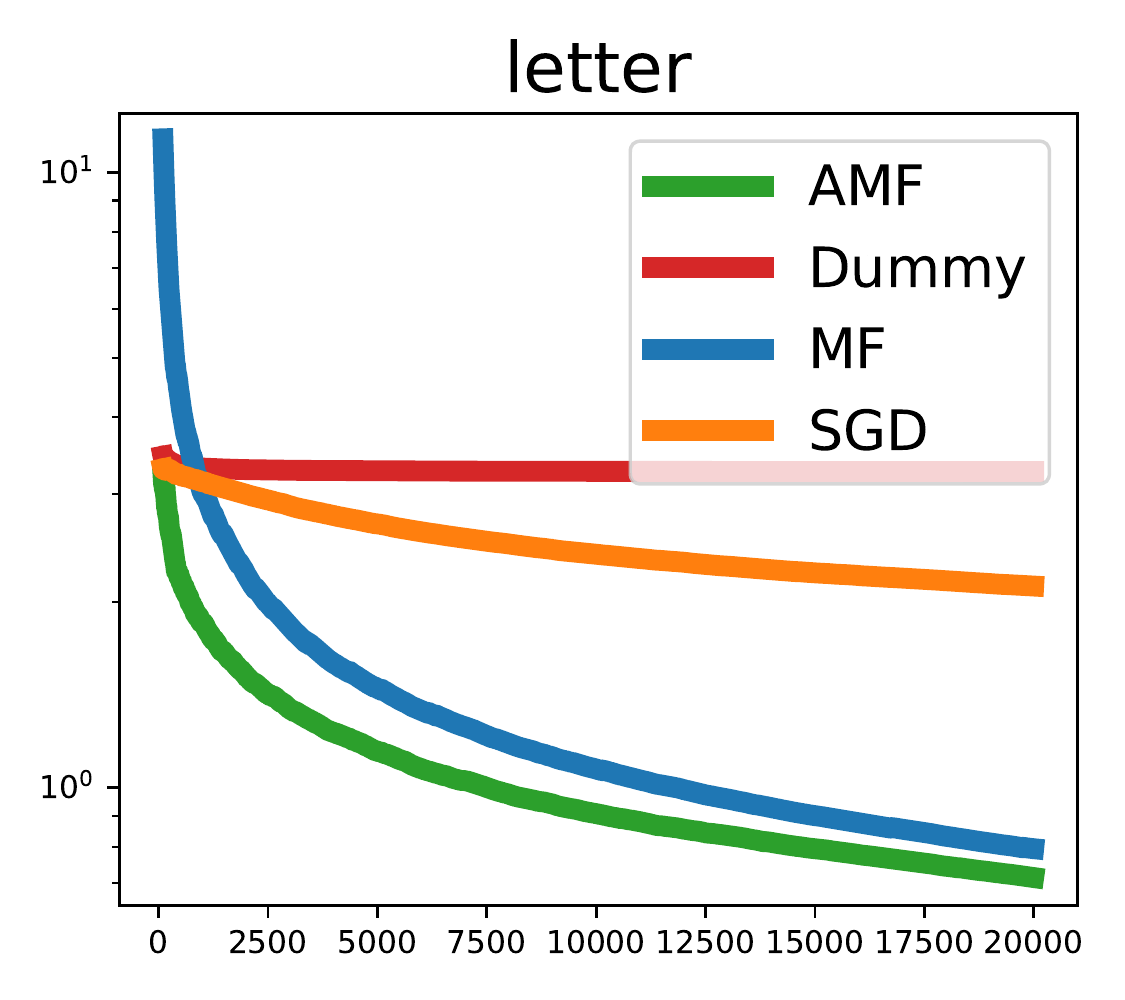}  
  \includegraphics[width=0.195\textwidth]{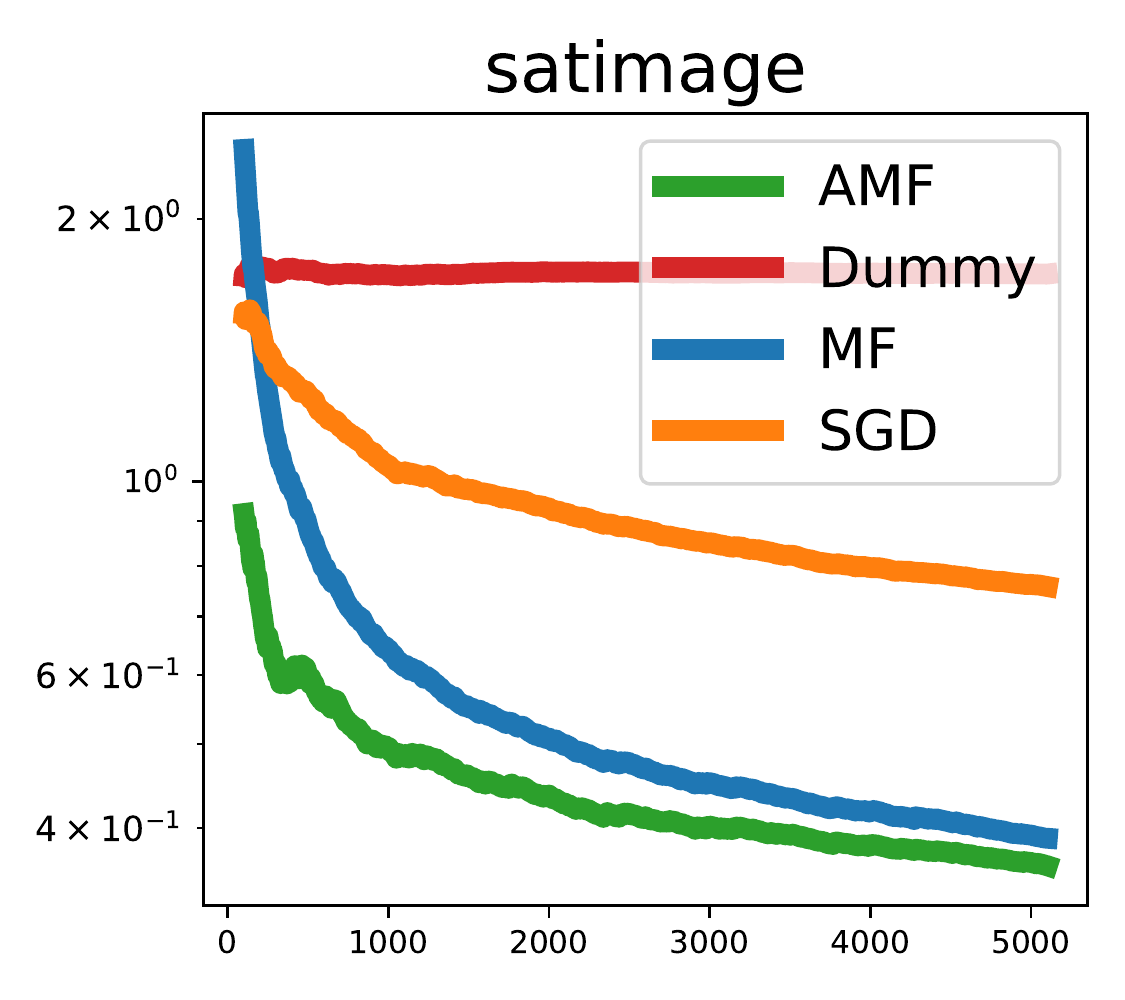}
  \includegraphics[width=0.195\textwidth]{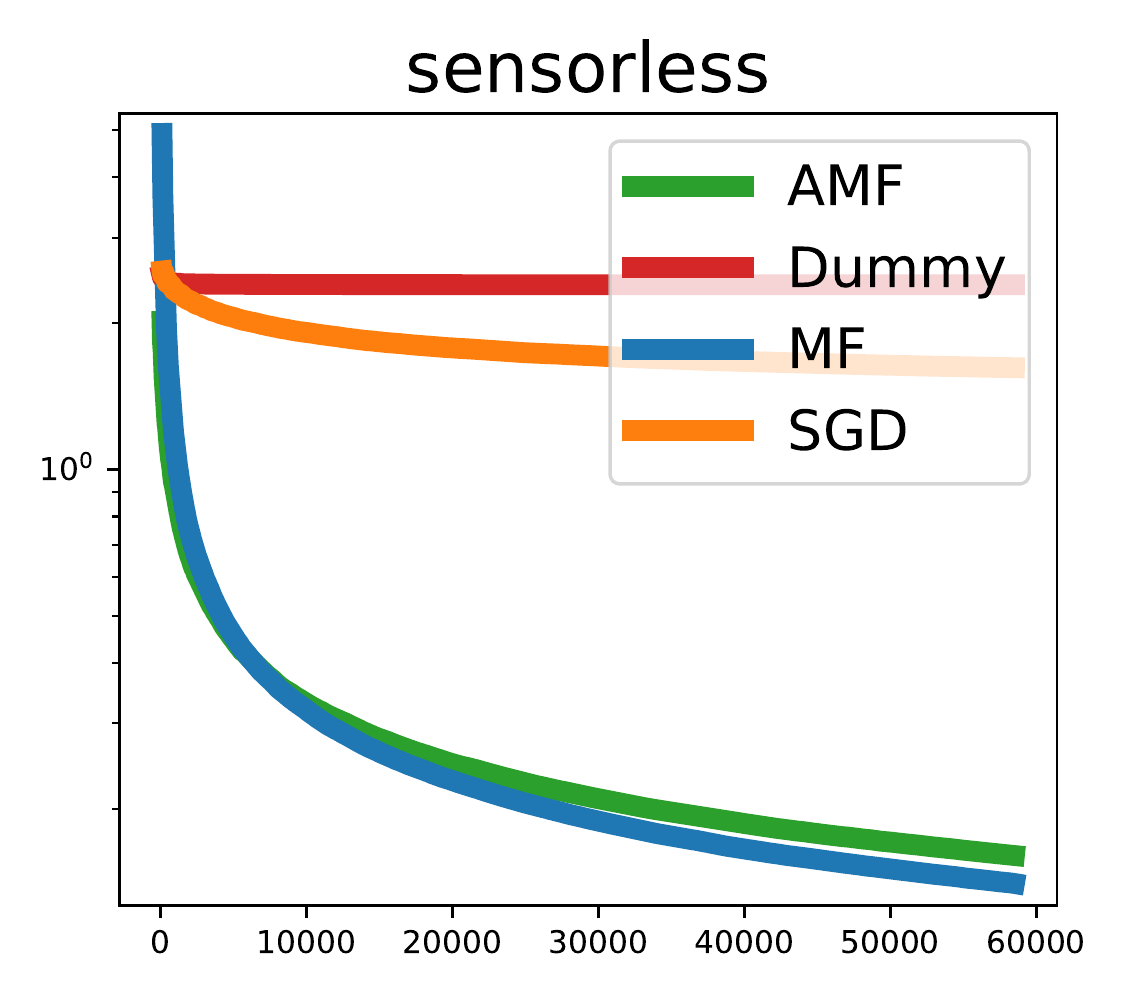}
  \includegraphics[width=0.195\textwidth]{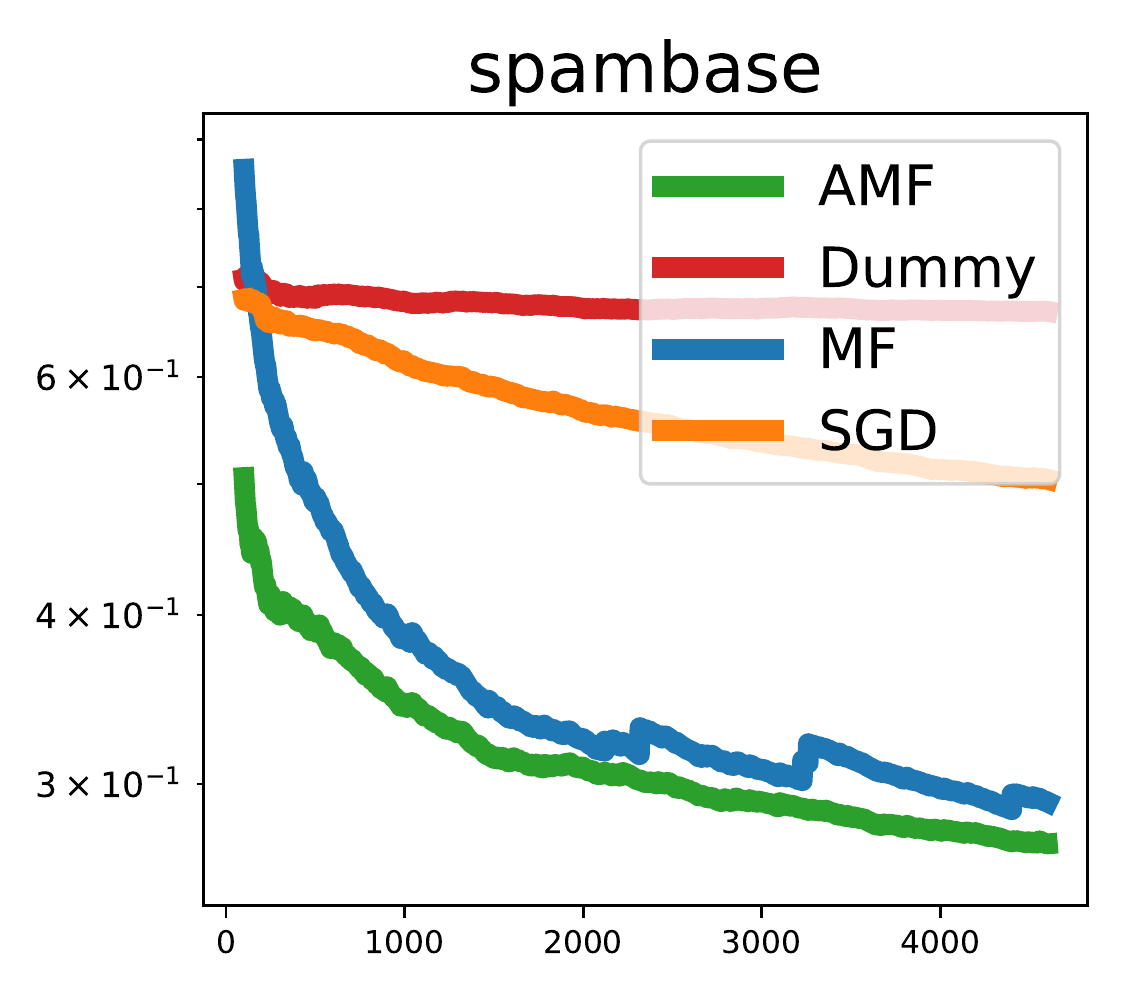}

  \caption{Average {losses} %regrets
    of all the online algorithms considered on 10 datasets for multi-class classification. The $x$-axis corresponds to the step $t$ (number of samples revealed) and the $y$-axis is the value of average regret obtained until this step (the lower the better). 
    AMF almost always exhibits the smallest average loss on all the considered datasets.}
  \label{fig:all_regrets}
\end{figure}

On most datasets, AMF exhibits the smallest average loss, and is always competitive with respect to the considered baselines.
As a comparison, the performance of SGD and MF strongly varies depending on the dataset: the ``robustness'' of AMF comes from the aggregation algorithm it uses, which always produces a non-overfitting and smooth decision function, as illustrated in Figure~\ref{fig:decisions} above, even in the early iterations.
This is confirmed by the early values of the average losses observed in all displays in Figure~\ref{fig:all_regrets}, where we see that it is always the smallest compared to all the baselines.

\subsubsection{Online versus Batch learning} % (fold)
\label{sub:comparisons_for_batch_learning}

In this Section, we consider a ``batch'' setting, where we hold out a test dataset (containing $30\%$ of the whole data), and we consider only binary classification problems, in which all methods are assessed using the area under the ROC curve (AUC) on the test dataset.
We consider the datasets \texttt{adult}, \texttt{bank}, \texttt{default}, \texttt{spambase} and all the methods (online and batch) described in Section~\ref{sub:description_of_the_considered_procedures}.
The performances of batch methods (RF and ET) are assessed only once using the test set, since these methods are not trained in an online fashion, but rather at once.
Therefore, the test AUCs of these batch methods are displayed in Figure~\ref{fig:batch_test_aucs} as a constant horizontal line along the iterations.
Online methods (AMF, Dummy, SGD and MF) are tested every $100$ iterations: each time $100$ samples are revealed, we produce predictions on the full test dataset, and report the corresponding test AUCs in Figure~\ref{fig:batch_test_aucs}.
We observe that, as more samples are revealed, the online methods improve their test AUCs.

\begin{figure}[htbp]
  \centering
  \includegraphics[width=0.24\textwidth]{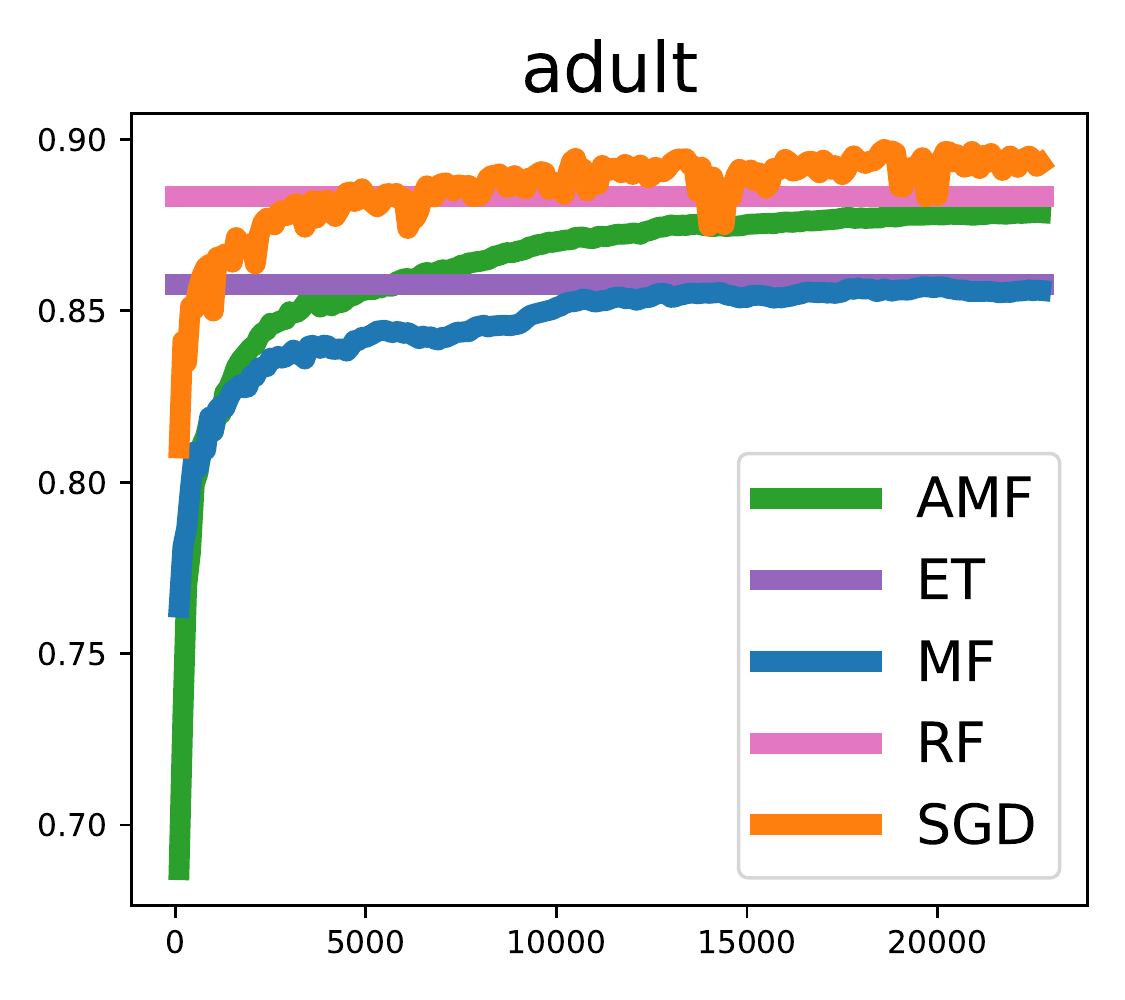}
  \includegraphics[width=0.24\textwidth]{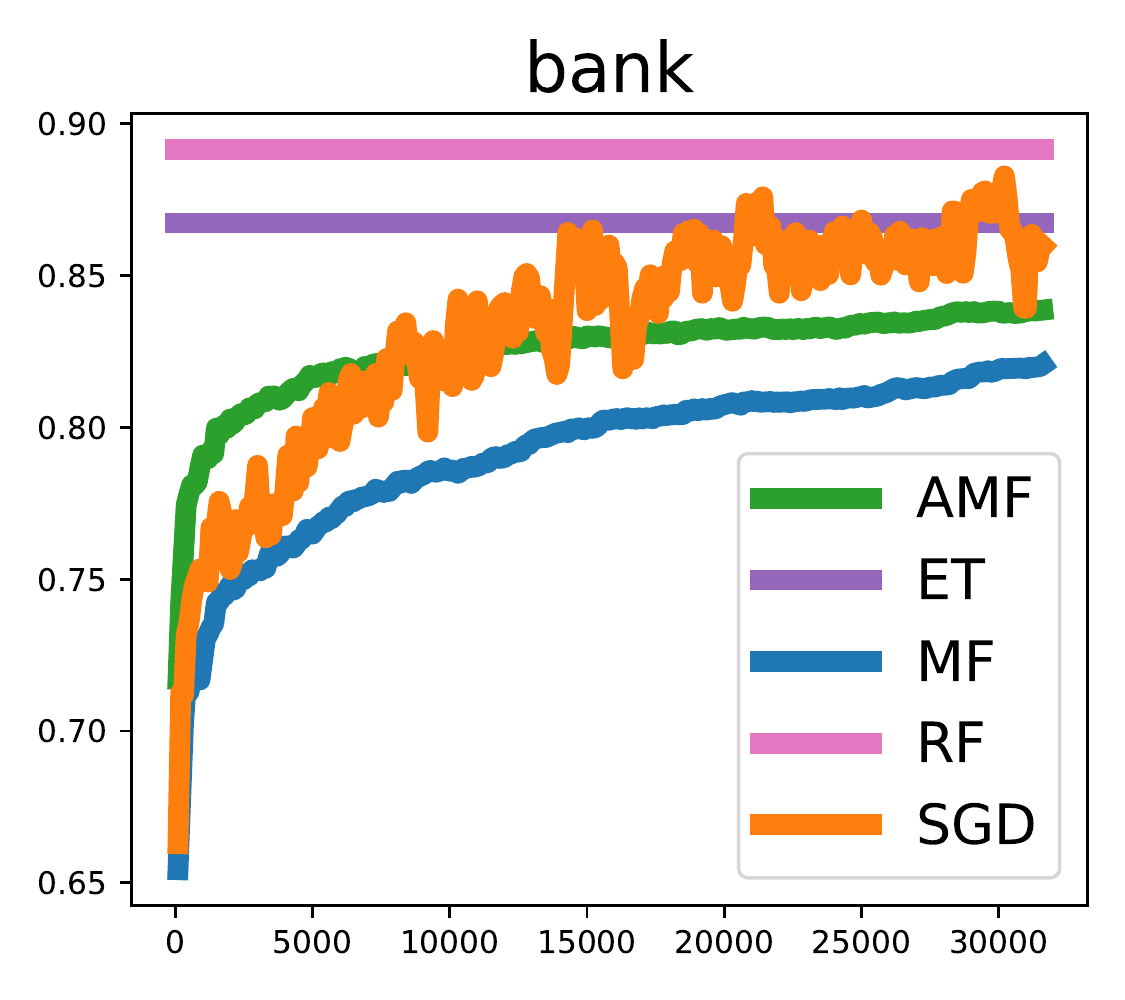}
  \includegraphics[width=0.24\textwidth]{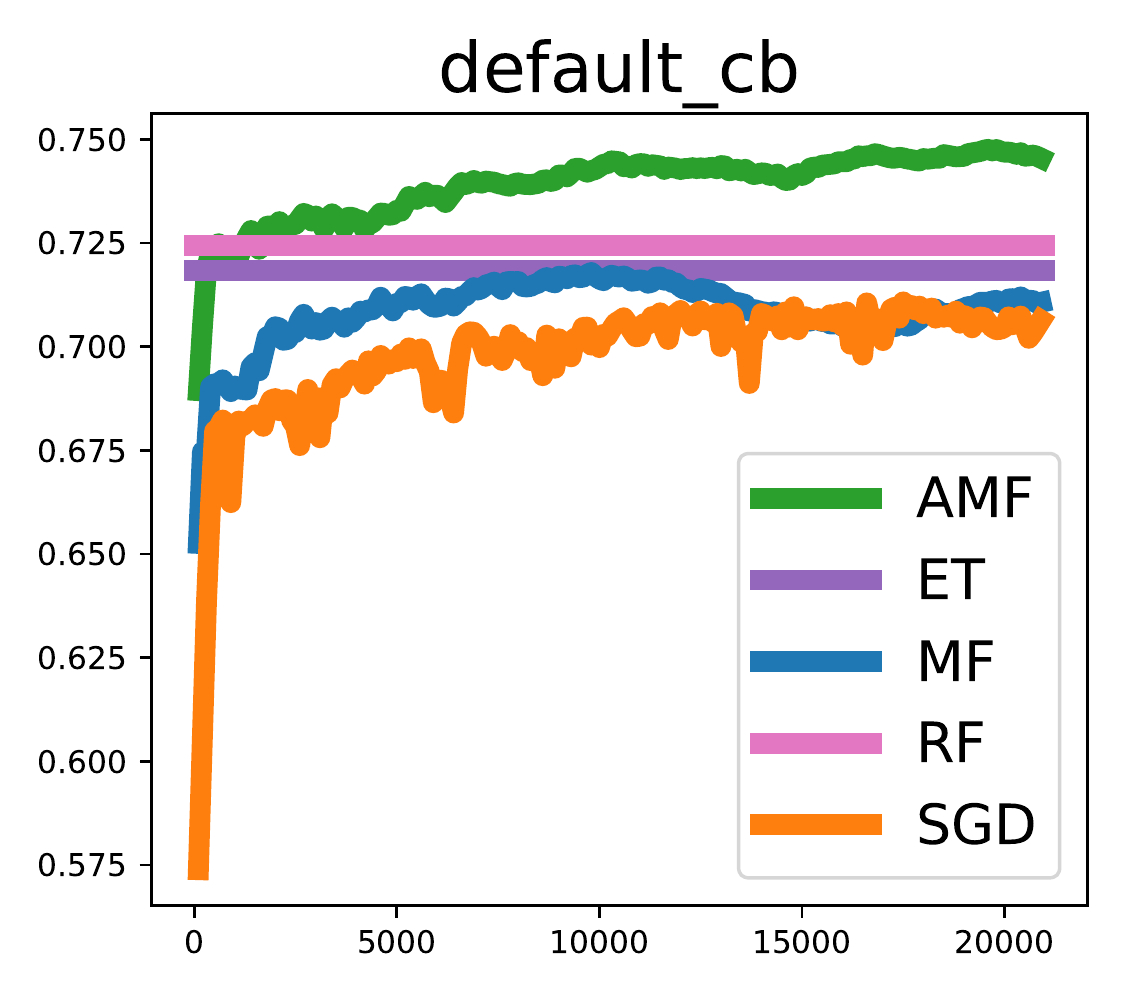}
  \includegraphics[width=0.24\textwidth]{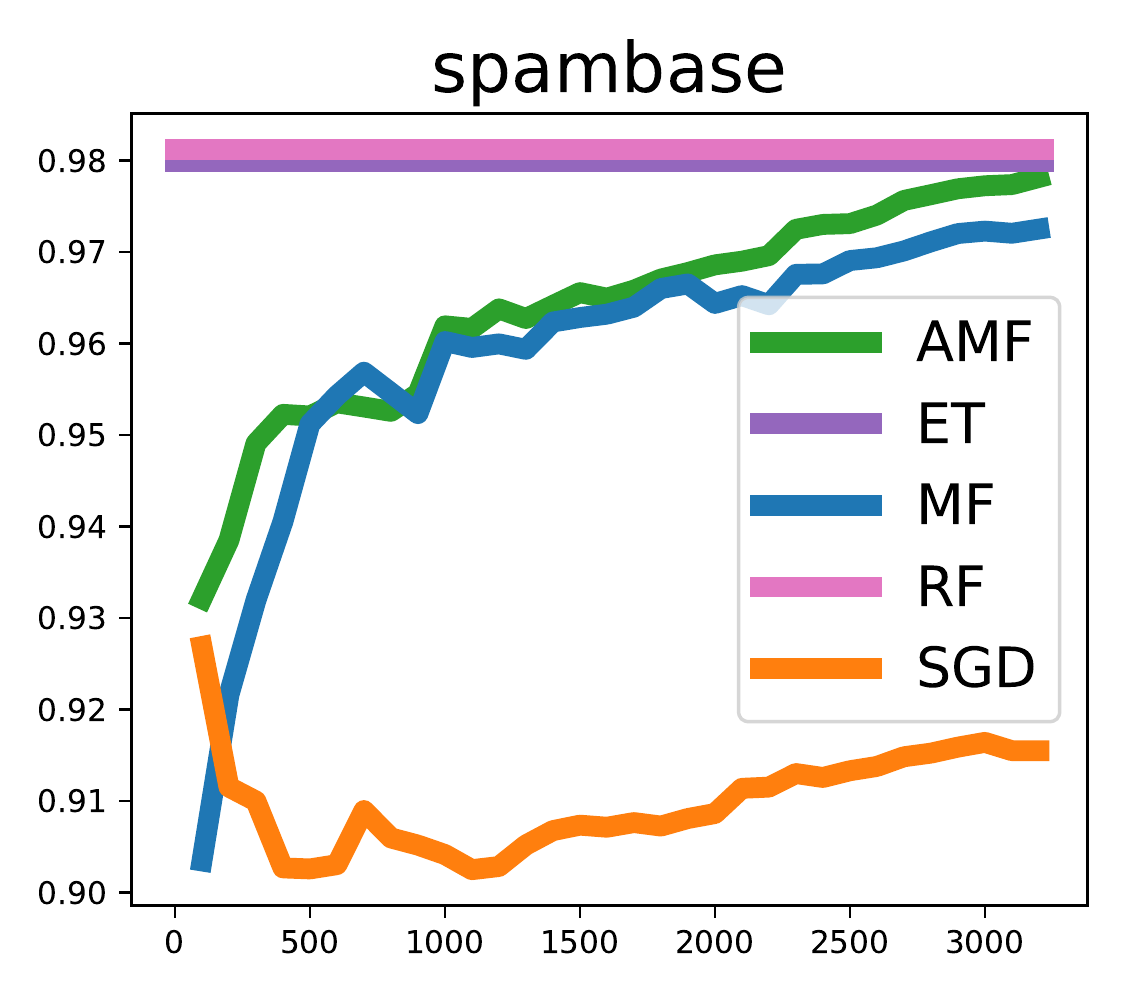}

  \caption{Area under the ROC curve (AUC) obtained on a held-out testing dataset (30\% of the whole data) obtained by batch methods (RF and ET) and online methods (SGD, MF and AMF) on four binary classification datasets. The $x$-axis corresponds to the online steps (number of samples seen) over the train dataset. AMF is very competitive and always achieve a good AUC after a few steps. 
  In the \texttt{defaultcb} dataset, AMF even improves the test AUC of RF and ET.}
  \label{fig:batch_test_aucs}
\end{figure}

We observe that the batch methods RF and ET generally perform best; this ought to be expected, since their splits are optimized using training data, while those of AMF and MF are chosen on-the-fly.
However, the performance of AMF is very competitive against all baselines.
In particular, it performs better than MF and even improves upon ET and RF on the \verb|default_cv| dataset.

\subsubsection{Sensitivity to the number of trees} % (fold)
\label{sub:sensitivity_to_the_number_of_trees}

The aim of this Section is to exhibit another positive effect of the aggregation algorithm used in AMF.
Indeed, we illustrate in Figure~\ref{fig:n_trees} below the fact that AMF can achieve good performances using less trees than MF, RF and ET.
This comes from the fact that even a single tree in AMF can be a good classifier, since the aggregation algorithm used in it (see Section~\ref{sub:agg-ctw}) aggregates all the prunings of the Mondrian tree.
This allows to avoid overfitting, even when a single tree is used, as opposed to the other tree-based methods considered here.
We consider in Figure~\ref{fig:n_trees} the same experimental setting as in Section~\ref{sub:comparisons_for_batch_learning}, and compare the test AUCs obtained on four binary classification problems for an increasing number of trees in all methods.
The test AUCs obtained by all algorithms with $1, 2, 5, 10, 20$ and $50$ trees are displayed in Figure~\ref{fig:n_trees}, where the $x$-axis corresponds to the number of trees and the $y$-axis corresponds to the test AUC.

\begin{figure}[htbp]
  \centering
  \includegraphics[width=\textwidth]{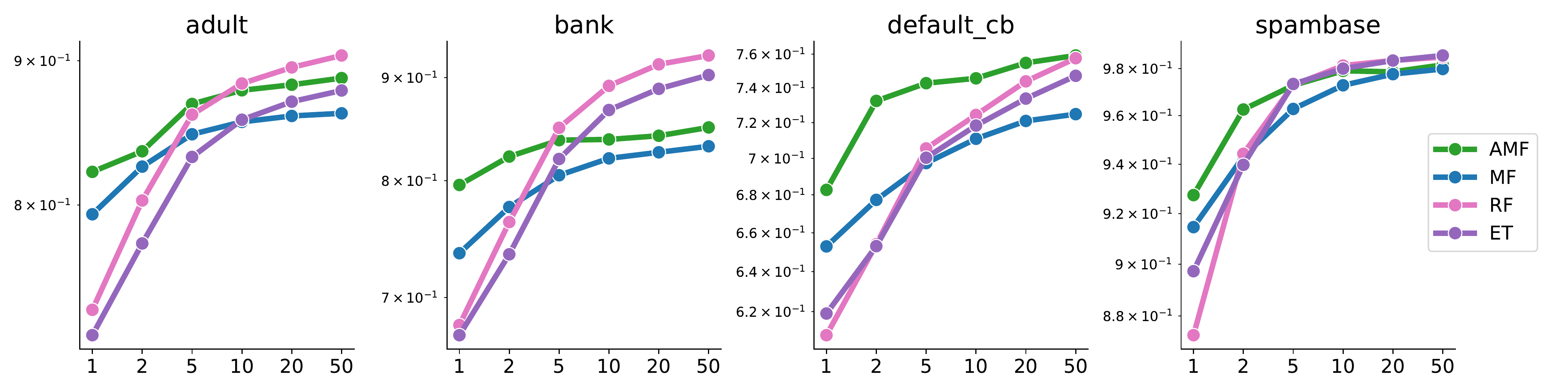}

  \caption{Area under the ROC curve (AUC) obtained on a held-out testing dataset (30\% of the whole data) obtained by AMF, MF, RF and ET as a function of the number of trees used.
    We observe that AMF is less sensitive to the number of trees used in the forest than all the baselines, and that it has good performances even when using one or two trees.}
  \label{fig:n_trees}
\end{figure}

We observe that when using one or two trees, AMF performs better than all the baselines.
The performance of RF strongly increases with an increasing number of trees, and ends up with the best performances with 50 trees.
The performance of AMF also improves when more trees are used (averaging over several realizations of the Mondrian partition certainly helps prediction), but the aggregation algorithm makes AMF somehow less sensitive to the number of trees in the forest.

\section{Conclusion}
\label{sec:conclusion}

In this paper we introduced AMF, an online random forest algorithm based on a combination of Mondrian forests and an efficient implementation of an aggregation algorithm over all the prunings of the Mondrian partition.
This algorithm is almost parameter-free, and has strong theoretical guarantees expressed in terms of regret with respect to the optimal time-pruning of the Mondrian partition, and in terms of adaptation with respect to the smoothness of the regression function.  
We illustrated on a large number of datasets the performances of AMF compared to strong baselines, where AMF appears as an interesting procedure for online learning.

A limitation of AMF, however, is that it does not perform feature selection.
It would be interesting to develop an online feature selection procedure that could indicate along which coordinates the splits should be sampled in Mondrian trees, and prove that such a procedure performs dimension reduction in some sense.
This is a challenging question in the context of online learning which deserves future  investigations.

\section{Proofs}
\label{sec:proofs}

This Section gathers the proofs of all the results of the paper, following their order of appearance, namely the proofs of Proposition~\ref{prop:algorithm-implements-aggregation}, Lemma~\ref{lem:regret-pruning}, Corollaries~\ref{cor:regret-best-pruning-log},~\ref{cor:regret-best-pruning-square} and~\ref{cor:regret-lambda-regression}, Lemma~\ref{lem:online-to-batch} and Theorems~\ref{thm:oracle-best-lambda} and~\ref{thm:minimax-adaptive}.

\subsection{Proof of Proposition~\ref{prop:algorithm-implements-aggregation}}
\label{sec:proof-proposition}

Consider a realization $\mondrian = (\tree^\mondrian, \Sigma^\mondrian) \sim \MP$ of the infinite Mondrian partition, and assume that we are at step $t \geq 1$, namely we observed $(x_1, y_1), \ldots, (x_{t-1}, y_{t-1})$ and performed the updates described in Algorithm~\ref{alg:mondrian-update} on each sample.
Given $x \in [0, 1]^d$, we want to predict the label (or its distribution) using
\begin{equation}
  \label{eq:ftx_proof}
 \widehat f_t (x) = \frac{\sum_{\tree \subset \tree^\mondrian} \pi (\tree) e^{-\eta L_{t-1} (\tree)} \pred_{\tree, t} (x)}{\sum_{\tree \subset \tree^\mondrian} \pi (\tree) 
 e^{-\eta L_{t-1} (\tree)}},
\end{equation}
where we recall that $\pi (\tree) = 2^{- | \tree |}$ with $|\tree|$ the number of nodes in $\tree$ and where we recall that the sum in~\eqref{eq:ftx_proof} is an infinite sum over all subtrees $\tree$ of $\tree^\mondrian$.

\paragraph{Reduction to a finite sum.}

Let $\globaltree$ denote the minimal subtree of $\tree^\mondrian$ that separates the elements of $\{ x_1, \dots, x_{t-1}, x \}$ (if $x = x_{t}$ then
 $\globaltree = \tree_{t+1}$).
Also, for every finite tree $\tree$, denote $\tree|_\globaltree := \tree \cap \globaltree$.
For any subtree $\tree$ of $\globaltree$, we have
\begin{equation}
  \label{eq:sum-subtree-restrict}
  \sum_{\tree' : \tree'|_\globaltree = \tree} \pi (\tree') = 2^{- \| \tree \|}
  =: \pi_\globaltree (\tree),
\end{equation}
where $\|\tree \|$ denotes the number of nodes of $\tree$ which are not leaves of $\globaltree$; note that $\pi_\globaltree$ is a probability distribution on the subtrees of $\globaltree$, since $\pi$ is a probability distribution on finite subtrees of $\{ 0, 1\}^*$.
To see why Equation~\eqref{eq:sum-subtree-restrict} is true, consider the following representation of $\pi$: let $(B_\node)_{\node \in \{ 0, 1 \}^*}$ be an \iid family of Bernoulli random variables with parameter $1 / 2$; a node $\node$ is said to be \emph{open} if $B_\node = 1$, and \emph{closed} otherwise.
Then, denote $\tree'$ the subtree of $\{ 0, 1\}^*$ all of whose interior nodes are open, and all of whose leaves are closed; clearly, $\tree' \sim \pi$.
Now, $\tree'|_\globaltree = \tree$ if and only if all interior nodes of $\tree$ are open and all leaves of $\tree$ except leaves of $\globaltree$ are closed.
By independence of the $B_\node$, this happens with probability $2^{-\| \tree \|}$.

In addition, note that if $\tree'$ is a finite subtree of $\{ 0, 1\}^*$ and $\tree =\tree'|_\globaltree$, then $\pred_{\tree', t} (x) = \pred_{\tree, t} (x)$.
Indeed, let $\node_{\tree'}(x)$ be the leaf of $\tree'$ that contains $x$; if $\node_{\tree'} (x) \in \globaltree$, then $\node_{\tree'} (x) = \node_{\tree}(x)$ and hence $\pred_{\tree', t} (x) = \pred_{\node_{\tree'}(x), t} = \pred_{\node_{\tree}(x), t} = \pred_{\tree, t}(x)$; otherwise, by definition of $\globaltree$, both $\node_{\tree'}(x)$ and $\node_\tree(x)$ only contain the $x_s$ ($s \leq t-1$) such that $x_s = x$, so that again $\pred_{\node_{\tree'}(x), t} = \pred_{\node_{\tree}(x), t}$.
Similarly, this result for $x = x_t$ also holds for $x_s$, $s \leq t-1$, so that $L_{t-1} (\tree') = L_{t-1} (\tree)$.
From the points above, it follows that
\begin{align}
  \label{eq:infinite-to-finite-aggregation}
  \wh f_t (x)
  &= \frac{\sum_{\tree \subset \tree^\mondrian} \pi (\tree) e^{-\eta L_{t-1} (\tree)} \pred_{\tree, t} (x)}{\sum_{\tree \subset \tree^\mondrian} \pi (\tree) e^{-\eta L_{t-1} (\tree)}} \nonumber \\
  &= \frac{\sum_{\tree \subset \tree^\mondrian} \sum_{\tree' : \tree'|_\globaltree = \tree} \pi (\tree') e^{-\eta L_{t-1} (\tree')} \pred_{\tree', t} (x)}{\sum_{\tree \subset \tree^\mondrian} \sum_{\tree' : \tree'|_\globaltree = \tree} \pi (\tree') e^{-\eta L_{t-1} (\tree')}} \nonumber \\
  &= \frac{\sum_{\tree \subset \tree^\mondrian} \sum_{\tree' : \tree'|_\globaltree = \tree} \pi (\tree') e^{-\eta L_{t-1} (\tree)} \pred_{\tree, t} (x)}{\sum_{\tree \subset \tree^\mondrian} \sum_{\tree' : \tree'|_\globaltree = \tree} \pi (\tree') e^{-\eta L_{t-1} (\tree)}} \nonumber \\
  &= \frac{\sum_{\tree \subset \globaltree} \ \pi_\globaltree (\tree) e^{-\eta L_{t-1} (\tree)} \pred_{\tree, t} (x)}{\sum_{\tree \subset \globaltree} \pi_\globaltree (\tree) e^{-\eta L_{t-1} (\tree)}}
    \, .
\end{align}

%\ste{?? Ce n'est pas terrible...manque qq explications ??}

\paragraph{Computation for the finite tree $\globaltree$.}
The expression in Equation~\eqref{eq:infinite-to-finite-aggregation} involves finite sums, over all subtrees of $\globaltree$ (involving an exponential in the number of leaves of $\globaltree$, namely $t$, terms).
However, it can be computed efficiently because of the specific choice of the prior $\pi$.
More precisely, we will use the following lemma \cite[Lemma~1]{helmbold1997pruning} several times to efficiently compute sums of products.
Let us recall that $\nodes(\globaltree)$ stands for the set of nodes of~$\globaltree$.

\begin{lemma}
  \label{lem:ctw-sum-prod}
  Let $g: \nodes(\globaltree) \to \R$ be an arbitrary function and define $G: \nodes(\globaltree) \to \R$ as
  \begin{equation}
  \label{eq:ctw-sum-prod-def}
  G (\node) = \sum_{\tree_\node} 2^{- \| \tree_\node \|} 
  \prod_{\othernode \in \leaves(\tree_\node)} g(\othernode),
  \end{equation}
  where the sum is over all subtrees $\tree_\node$ of $\globaltree$ rooted 
  at $\node$. 
  Then\textup, $G (\node)$ can be computed recursively as follows\textup:
  \begin{equation*}
    G(\node) = 
    \begin{cases}
      g(\node) & \text{ if } \node \in \leaves(\globaltree) \\
      \frac{1}{2} g(\node) + \frac{1}{2} G(\node 0) G(\node 1) & \text{ otherwise,}
    \end{cases}
  \end{equation*}
  for each node $\node \in \nodes(\globaltree)$.
\end{lemma}

Let us introduce 
\begin{equation*}
w_t (\tree) = \pi_\globaltree (\tree) \exp (- \eta \cumloss_{t-1} (\tree)),
\end{equation*}
so that Equation~\eqref{eq:ftx_proof} writes
\begin{equation}
\label{eq:ewa-tree}
\wh f_t(x) = \frac{\sum_{\tree \subset \globaltree} w_{t} (\tree) \pred_{\tree, t} (x)}{\sum_{\tree \subset \globaltree} w_{t} (\tree)},
\end{equation}
where the sums hold over all subtrees $\tree$ of $\globaltree$.
We will show how to efficiently compute and update the numerator and denominator in Equation~\eqref{eq:ewa-tree}.
Note that $w_t (\tree)$ may be written as
\begin{equation}
\label{eq:weight-tree-decomp}
w_t (\tree) = 2^{-\|\tree\|} \prod_{\node \in \leaves(\tree)} w_{\node, t}
\end{equation}
with $w_{\node, t} = \exp(- \eta \cumloss_{\node, t-1})$, where $\cumloss_{\node, t} := \sum_{s \leq t \pp X_t \in \cell_\node} \ell (\prednode_{\node, s}, y_s)$.

\paragraph{Denominator of Equation~\eqref{eq:ewa-tree}.}
For each node $\node \in \nodes (\globaltree)$ and every $t\geq 1$, denote
\begin{equation}
\label{eq:avg-weights-def}
\wbar_{\node, t} = \sum_{\tree_\node} 2^{-\|\tree_\node\|} \prod_{\node' \in \leaves (\tree_\node)}
w_{\node', t}
\end{equation}
so that~\eqref{eq:weight-tree-decomp} entails
\begin{equation}
\label{eq:wbar-root}
\wbar_{\root, t} = \sum_{\tree} w_{t} (\tree) \, .
\end{equation}
Using Equation~\eqref{eq:avg-weights-def}, the weights $\wbar_{\node, t}$ can be computed recursively using Lemma~\ref{lem:ctw-sum-prod}.
We denote by $\pathpoint(x_t)$ the path from $\root$ to $\node_\globaltree(x_t)$ (from the root to the leaf containing $x_t$).
Note that, by definition of $w_{\node, t}$, if $\node \not \in \pathpoint(x_t)$ (namely $x_t \not \in \cell_\node$), we have $w_{\node, t+1} = w_{\node, t}$.
In addition, if $\node \not \in \pathpoint(x_t)$, so are all its descendants, so that (by induction, and using the above recursive formula) $\wbar_{\node, t+1} = \wbar_{\node, t}$.
In other words, \emph{only the nodes of $\pathpoint(x_t)$ have updated weights}.

As a result, at each round $t \geq 1$, after seeing $(x_t, y_t) \in [0, 1]^d \times \Y$, the weights $w_{\node, t}$ and $\wbar_{\node, t}$ are updated for $\node \in \pathpoint(x_t)$ as follows (note that they are all initialized at $w_{\node, 1} = \wbar_{\node, 1} =1$):
\begin{itemize}
  \item for every $\node \not \in \pathpoint(x_t)$, $w_{\node, t+1} = w_{\node, t}$ and $\wbar_{\node, t+1} = \wbar_{\node, t}$;
  \item for every $\node \in \pathpoint(x_t)$, $w_{\node, t+1} = w_{\node, t} \exp (- \eta \ell (\pred_{\node, t}, y_t))$; 
  \item for every $\node \in \pathpoint(x_t)$, we have
  \begin{equation*}
    \wbar_{\node, t+1} =
    \begin{cases}
      w_{\node, t+1} & \text{ if } \node \in \leaves (\globaltree) \;\;\; \text{ (namely } \node = \node_\globaltree (x_t) ), \\
      \frac{1}{2} w_{\node, t+1} + \frac{1}{2} \wbar_{\node 0, t+1} \wbar_{\node 1, t+1} &\text{ otherwise}.
    \end{cases}
  \end{equation*}
\end{itemize}
The weights $w_{\node, t},\wbar_{\node, t}$ as well as the predictions $\pred_{\node, t}$ are updated recursively in an ``upwards'' traversal of $\pathpoint(x_t)$ in $\globaltree$ (from $\node_\globaltree (x_t)$ to $\root$), as indicated in Algorithm~\ref{alg:mondrian-update}.

Note that when updating the structure of the tree, the weights $w_{\node, t+1}, \wbar_{\node, t+1}$ and predictions $\pred_{\node, t+1}$ for the newly created nodes in $\tree_{t+1} \setminus \tree_t$ (which are offsprings of $\node_{\tree_t}(x_t)$ created from the splits necessary to separate $x_t$ from the other point $x_s \in C_{\node_{\tree_t}(x_t)}$) can be set depending on whether these nodes contain $x_s$ or $x_t$.
This does not affect the values of $w_{\node, t}$ and $\pred_{\node, t}$ at other nodes, but only the values of $\wbar_{\node, t}$ for $\node \in \pathpoint(x_t)$ that are computed in the upwards recursion.

\paragraph{Numerator of Equation~\eqref{eq:ewa-tree}.}

The numerator of Equation~\eqref{eq:ewa-tree} can be computed in the same fashion as the denominator.
Let $w'_{\node, t} = w_{\node, t} \prednode_{\node, t}$ if $\node \in \pathpoint(x)$, and $w'_{\node, t} = w_{\node, t}$ otherwise.
Additionally, let
\begin{equation*}
\wpred_{\node, t}
= \sum_{\tree_\node} 2^{-\|\tree_\node\|}
\prod_{\node' \in \leaves (\tree_\node)} w'_{\node', t} \, .
\end{equation*}
%%% NB: $\wpred$ depends on $x$
Note that we have
\begin{equation}
\label{eq:wpred-root}
\wpred_{\root, t}
= \sum_{\tree} 2^{-\|\tree\|}
\prod_{\node' \in \leaves (\tree)} w'_{\node', t}
= \sum_{\tree} w_t (\tree) \prednode_{\node_{\tree} (x), t}
= \sum_{\tree} w_t (\tree) \pred_{t} (\tree) \, .
\end{equation}
Lemma~\ref{lem:ctw-sum-prod} with $g(\node) = w'_{\node, t}$ (so that $G(\node) = \wpred_{\node, t}$) enables to recursively compute $\wpred_{\node, t}$ from $w'_{\node, t}$.
First, note that $w'_{\node, t} = w_{\node, t}$ for every $\node \not \in \pathpoint(x)$.
Since every descendant $\node'$ of $\node$ is also outside of $\pathpoint(x)$, it follows by induction that $\wpred_{\node, t} = \wbar_{\node, t}$ for every $\node \not \in \pathpoint(x)$.
It then remains to show how to compute $\wpred_{\node, t}$ for $\node \in \pathpoint(x)$.
This is done again recursively, starting from the leaf $\node_\globaltree (x)$ up to the root $\root$:
\begin{equation*}
  \wpred_{\node, t} =
  \begin{cases}
      w_{\node, t} \pred_{\node, t} & \text{ if } \node = \node_\globaltree (x) \\
      \frac{1}{2} w_{\node, t} \pred_{\node, t} + \frac{1}{2} \wbar_{\node (1-a), t} \wpred_{\node a, t} &\text{ otherwise, where } a \in \{ 0, 1 \} \text{ is such that } \node a \in \pathpoint(x)
  \end{cases}
\end{equation*}
Finally, we define 
\begin{equation*}
  \wt y_{\node, t} (x) = \frac{\wpred_{\node, t}}{\wbar_{\node, t}}
\end{equation*}
for each node $\node \in \globaltree$.
It follows from Equations~\eqref{eq:ewa-tree},~\eqref{eq:wbar-root} and~\eqref{eq:wpred-root} that $\wh f_t (x) = \wt y_{\root, t} (x)$.
Additionally, the recursive expression for $\wbar_{\node, t}$ and $\wpred_{\node, t}$ imply that $\wt y_{\node, t}$ can be computed recursively as well, in the upwards traversal from $\node_\globaltree (x)$ to $\root$: we set
\begin{equation*}
  \wt y_{\node, t} (x) = \pred_{\node, t}
\end{equation*}
for $\node = \node_\globaltree (x)$, otherwise we set
\begin{equation*}
  \wt y_{\node, t} (x)
  = \frac{1}{2} \frac{w_{\node, t}}{\wbar_{\node, t}} \pred_{\node, t} + \frac{1}{2} \frac{\wbar_{\node a, t} \wbar_{\node (1-a), t}}{\wbar_{\node, t}} \wt y_{\node a, t} (x)
  = \frac{1}{2} \frac{w_{\node, t}}{\wbar_{\node, t}} \pred_{\node, t} + \left( 1 - \frac{1}{2} \frac{w_{\node, t}}{\wbar_{\node, t}} \right) \wt y_{\node a, t} (x),
\end{equation*}
where $a \in \{ 0, 1 \}$ is such that $\node a \in \pathpoint(x)$.
The recursions constructed above are precisely the ones describing AMF in Algorithms~\ref{alg:mondrian-update} and~\ref{alg:predict}, so that this concludes the proof of Proposition~\ref{prop:algorithm-implements-aggregation}. $\hfill \square$

\subsection{Proofs of Lemma~\ref{lem:regret-pruning}, Corollaries~\ref{cor:regret-best-pruning-log},~\ref{cor:regret-best-pruning-square},~\ref{cor:regret-lambda-regression}, Lemma~\ref{lem:online-to-batch}, Theorems~\ref{thm:oracle-best-lambda},~\ref{thm:minimax-adaptive} and Proposition~\ref{prop:mondrian_depth_bound}}

We start with some well-known lemmas that are used to bound the regret: Lemma~\ref{lem:exp-regret} controls the regret with respect to each tree forecaster, while Lemmas~\ref{lem:regret-kt-log} and~\ref{lem:regret-average-square} bound the regret of each tree forecaster with respect to the optimal labeling of its leaves.

\begin{lemma}[\citealp{vovk1998mixability}]%[Regret of the Exponential Weights algorithm]
  \label{lem:exp-regret}
  Let $\Experts$ be a countable set of experts and $\prior = (\prior_i)_{i\in \Experts}$ be a probability measure on $\Experts$.
  Assume that $\ell$ is $\eta$-\emph{exp-concave}.
  For every $t \geq 1$\textup, let $y_t \in \Y$\textup, $\pred_{i,t} \in \predspace$ be the prediction of expert $i \in \Experts$ and $L_{i,t} = \sum_{s=1}^t \ell (\pred_{i,s}, y_s)$ be its cumulative loss.
  Consider the predictions defined as
  \begin{equation}
    \label{eq:exponential-weights}
    \pred_t = \frac{\sum_{i \in \Experts} \pi_i \, e^{- \eta L_{i, t-1}} \pred_{i,t} }{\sum_{i \in \Experts} \pi_i \, e^{-\eta L_{i, t-1}}}.
  \end{equation}
  Then\textup, irrespective of the values of $y_t \in \Y$ and $\pred_{i,t} \in \predspace,$ we have the following regret bound
  \begin{equation}
    \label{eq:exp-regret}
    \sum_{t=1}^n \ell (\pred_t, y_t) - \sum_{t=1}^n \ell (\pred_{i,t}, y_t)
    \leq \frac{1}{\eta} \log \frac{1}{\prior_i}
  \end{equation}
  for each $i \in \Experts$ and $n \geq 1$.
\end{lemma}

\begin{lemma}[\citealp{tjalkens1993sequential}]
  \label{lem:regret-kt-log}
  Let $\ell$ be the logarithmic loss on the finite set $\Y,$ and let $y_t \in \Y$ for every $t \geq 1$.
  The \emph{Krichevsky-Trofimov (KT)}  forecaster\textup, which predicts
  \begin{equation}
    \label{eq:kt}
    \pred_{t} ( y ) = \frac{n_{t-1} (y) + 1/2}{(t - 1) + |\Y |/2}
    \, ,
  \end{equation}
  with $n_{t-1} (y) = | \{ 1 \leq s \leq t -1 : y_s = y \} |,$ 
  satisfies the following regret bound with respect to the class $\probas(\Y)$ of constant experts \textup(which always predict the same probability distribution on $\Y$\textup)\textup: 
  \begin{equation}
    \label{eq:regret-kt}
    \sum_{t=1}^n \ell (\pred_t, y_t) - \inf_{p \in \probas (\Y)} \sum_{t=1}^n \ell (p, y_t)
    \leq\frac{|\Y| - 1}{2} \log (4n)
  \end{equation}
  for each $n \geq 1.$
\end{lemma}

%%% and enable to control the regret at each leaf
%%% put in the appendix or proof section ?
\begin{lemma}[\citealp{cesabianchi2006PLG}, p.~43]
  \label{lem:regret-average-square}
  Consider the square loss $\ell (\pred, y) = (\pred - y)^2$ on $\Y = \predspace = [-B, B],$ with $B >0$. 
  For every $t \geq 1,$ let $y_t \in [-B, B]$.
  Consider the strategy defined by $\pred_1 = 0$, and for each $t \geq 2,$
  \begin{equation}
    \label{eq:average-online}
    \pred_t = \frac{1}{t-1} \sum_{s=1}^{t-1} y_s \, .
  \end{equation}
  The regret of this strategy with respect to the class of constant experts \textup(which always predict some $b\in [-B, B]$\textup) is upper bounded as follows\textup:
  \begin{equation}
    \label{eq:regret-average-square}
    \sum_{t=1}^n \ell (\pred_t, y_t) - \inf_{b \in [-B, B]} \sum_{t=1}^n \ell (b, y_t)
    \leq 8 B^2 (1+ \log n)
  \end{equation}
  for each $n \geq 1.$
\end{lemma}

\begin{proof}[Lemma~\ref{lem:regret-pruning}]
  This follows from Proposition~\ref{prop:algorithm-implements-aggregation} and Lemma~\ref{lem:exp-regret}.  
\end{proof}

\begin{proof}[Corollary~\ref{cor:regret-best-pruning-log}]
  Since the logarithmic loss is $1$-exp-concave, Lemma~\ref{lem:regret-pruning} implies
  \begin{equation}
    \label{eq:proof-regret-log-1}
    \sum_{t=1}^n \ell (\wh f_t (x_t) , y_t)    
    - \sum_{t=1}^n \ell (\pred_{\tree, t} (x_t), y_t)
    \leq |\tree| \log 2
  \end{equation}
  for every subtree $\tree$.
  It now remains to bound the regret of the tree forecaster $\tree$ with respect to the optimal labeling of its leaves.
  By Lemma~\ref{lem:regret-kt-log}, for every leaf $\node$ of $\tree$,
  % and every $p_\node \in \probas(\Y)$,
  \begin{equation*}
    %\label{eq:proof-regret-log-2}
    \sum_{1\leq t \leq n \pp x_t \in \cell_\node} \ell (\pred_{\tree, t} (x_t), y_t)
    - \inf_{p_\node \in \probas(\Y)} \sum_{1\leq t \leq n \pp x_t \in \cell_\node} \ell (p_\node, y_t)
    \leq \frac{|\Y| - 1}{2} \log (4 N_{\node, n})
  \end{equation*}
  where $N_{\node, n} = | \{ 1\leq t \leq n : x_t \in \cell_\node \} |$ (assuming that $N_{\node, n} \geq 1$).
  Summing the above inequality over the leaves $\node$ of $\tree$ such that $N_{\node, n} \geq 1$
  yields
  \begin{equation}
    \label{eq:proof-regret-log-3}
    \sum_{t=1}^n \ell (\pred_{\tree, t} (x_t), y_t)
    - \inf_{g_\tree} \sum_{t=1}^n \ell (g_\tree (x_t), y_t)
    \leq \frac{|\Y| - 1}{2} \sum_{\node \in \leaves (\tree) \pp N_{\node, n} \geq 1} \log (4 N_{\node, n})
  \end{equation}
  where $g_\tree$ is any function constant on the leaves of $\tree$.
  Now, letting $L = | \{ \node \in \leaves(\tree) :  N_{\node, n} \geq 1 \} | \leq |\leaves(\tree)| = \frac{|\tree| + 1}{2}$, we have by concavity of the $\log$ 
  \begin{align*}
    \sum_{\node \in \leaves (\tree) \pp N_{\node, n} \geq 1} \log (4 N_{\node, n})
    &\leq L \log \bigg( \frac{\sum_{\node \in \leaves (\tree) \pp N_{\node, n} \geq 1} 4 N_{\node, n}}{L} \bigg) \\
    & = L \log \Big( \frac{4n}{L} \Big) \leq \frac{|\tree|+1}{2} \log (4n) \, .
  \end{align*}
  Plugging this in~\eqref{eq:proof-regret-log-3} and combining with Equation~\eqref{eq:proof-regret-log-1} leads to the desired bound~\eqref{eq:regret-best-pruning-log}.
\end{proof}

\begin{proof}[Corollary~\ref{cor:regret-best-pruning-square}]
  The proof proceeds similarly to that of Corollary~\ref{cor:regret-best-pruning-log}, by combining Lemmas~\ref{lem:regret-pruning} and~\ref{lem:regret-average-square} and using the fact that the square loss is $\eta = 1/ ( 8 B^2 )$-exp-concave on $[-B, B]$.
\end{proof}

\begin{proof}[Corollary~\ref{cor:regret-lambda-regression}]
  First, we reason conditionally on the Mondrian process $\mondrian$.
  By applying Corollary~\ref{cor:regret-best-pruning-square} to $\tree = \mondrian_\lambda$, we obtain, since the number of nodes of $\mondrian_\lambda$ is $2 | \leaves(\mondrian_\lambda) | - 1$:
  \begin{equation}    
    \label{eq:proof-regret-lambda}    
    \sum_{t=1}^n \ell (\wh f_t (x_t) , y_t)
    - \inf_g \sum_{t=1}^n \ell (g (x_t), y_t)
    \leq 8 B^2 | \leaves (\mondrian_\lambda)| \log n \, ,
  \end{equation}
  where the infimum spans over all functions $g : [0, 1]^d \to \predspace$ which are constant on the cells of $\mondrian_\lambda$.
  Corollary~\ref{cor:regret-lambda-regression} follows by taking the expectation over $\mondrian$ and using the fact that  $\mondrian_\lambda \sim \MP (\lambda)$ implies $\E [ |\leaves (\mondrian_\lambda) | ] = (1+\lambda)^d$ \citep[Corollary~1]{mourtada2018mondrian}.
\end{proof}

\begin{proof}[Lemma~\ref{lem:online-to-batch}]
  For every $t=1, \dots, n$, $\wh f_t$ is $\F_{t-1} := \sigma(x_1, y_1, \dots, x_{t-1}, y_{t-1})$-measurable and since $(x_t, y_t)$ is independent of $\F_t$:
  % $\E [ \wh f_t () ]$
  \begin{equation*}
    \E [ \ell ( \wh f_t (x_t), y_t ) ]
    = \E [ \E [ \ell ( \wh f_t (x_t) , y_t ) \cond \F_{t-1} ]  ]
    = \E [ R (\wh f_t) ]
    \, ,
  \end{equation*}
  so that, for every $g \in \G$,
  \begin{equation*}
    \frac{1}{n} \E \Big[ \sum_{t=1}^n \big( \ell (\wh f_t (x_t), y_t) - \ell (g(x_t), y_t) \big) \Big] = \frac{1}{n} \sum_{t=1}^n \E [ R (\wh f_t) ] - R(g) = \E [ R (\widetilde f_n ) ] - R(g) \, .
  \end{equation*}
\end{proof}

\begin{proof}[Theorem~\ref{thm:oracle-best-lambda}]
  This is a direct consequence of Lemma~\ref{lem:online-to-batch} and Corollary~\ref{cor:regret-best-pruning-square}.
\end{proof}

\begin{proof}[Theorem~\ref{thm:minimax-adaptive}]  
  Recall that the sequence $(x_1, y_1), \ldots, (x_n, y_n)$ is i.i.d and distributed as a generic pair $(x, y) \in [0, 1]^d \times \Y$.
  Since $f^*(\cdot) = \E [y \cond x = \cdot]$, we have 
  \begin{equation}
    \label{eq:proof-minimax-regression}
    R (f) = \E [ (f(x) - f^*(x))^2 ] + R(f^*)
  \end{equation}
  for every function $f : [0, 1]^d \to \R$.
  Now, let $\lambda > 0$ be arbitrary.
  Consider the estimator $\widetilde f_n$ defined in Lemma~\ref{lem:online-to-batch}, and the function $h_\lambda^*$ constant on the cells of a random partition $\mondrian_\lambda \sim \MP (\lambda)$, with optimal predictions on the leaves given by $h^*_\lambda(u) = \E [y | x \in \cell_\node]$ for $u \in \cell_\node$, for every leaf $\node$ of $\mondrian_\lambda$.
  Since $R(\widetilde f_n) - R(f^*) = R(\widetilde f_n) - R(h^*_\lambda) + R(h^*_\lambda) - R(f^*)$, Equation~\eqref{eq:proof-minimax-regression} gives, after taking the expectation over the random sampling of the Mondrian process $\mondrian_\lambda$,
  \begin{equation}
    \label{eq:proof-minimax-regression-2}
    \E [ (\widetilde f_n (x) - f^* (x))^2 ]
    = \E [ R (\widetilde f_n) ] - \E [ R (h^*_\lambda) ] + \E [ (h^*_\lambda (x) - f^* (x))^2 ]
    \, .
  \end{equation}
  Let $D_\lambda (u)$ denote the diameter of the cell $\cell_\lambda(u)$ of $u \in [0, 1]^d$ in the Mondrian partition $\mondrian_\lambda$ used to define~$h^*_\lambda$.
  Assume that $f^*$ is $\beta$-Holder with constant $L > 0$, namely $| f^*(u) - f^*(v) | \leq L |u - v|^\beta$ for any $u, v \in [0, 1]^d$.
  Since $h_\lambda^*(u) = \E [f^*(x) | x \in \cell_\lambda (u)]$, we have $| h^*_\lambda(u) 
  - f^*(u) | \leq L D_\lambda (u)^\beta$, so that
  \begin{equation}
    \label{eq:proof-minimax-regression-3}
    \E [ (h^*_\lambda (x) - f^* (x))^2 ] \leq L^2 \E [ D_\lambda (x)^{2\beta} ] \, .
  \end{equation}
  Now, since $u \mapsto u^\beta$ is concave,   
  \begin{equation}
    \label{eq:proof-minimax-regression-4}
    \E [D_\lambda (x)^{2\beta}]
    \leq \E [D_\lambda (x)^{2}]^\beta    
    \leq \Big( \frac{4 d}{\lambda^2} \Big)^\beta
  \end{equation}
  where the last inequality comes from Proposition~1 in~\citet{mourtada2018mondrian}.
  Integrating with the distribution of $x$ and using~\eqref{eq:proof-minimax-regression-3} gives $\E [ (h^*_\lambda (x) - f^* (x))^2 ] \leq (4 d)^{\beta} L^2 /\lambda^{2\beta}$.
  In addition, Theorem~\ref{thm:oracle-best-lambda} gives
  $\E [ R(\widetilde f_n) ] - \E [ R(h^*_\lambda) ] \leq 8 B^2 (1+\lambda)^d {(\log n)}/{n} $.
  Combining these inequalities with~\eqref{eq:proof-minimax-regression-2} leads to
  \begin{equation}
    \label{eq:proof-minimax-regression-5}
    \E [ (\widetilde f_n (x) - f^* (x))^2 ]
    \leq \frac{(4 d)^{\beta} L^2}{\lambda^{2\beta}} + \frac{8 B^2 (1+\lambda)^d \log n}{n}
    \, .
  \end{equation}
  Note that the bound~\eqref{eq:proof-minimax-regression-5} holds for every value of $\lambda >0$.
  In particular, for $\lambda \asymp (n / \log n)^{1/(d + 2 \beta)}$, it yields the
   $O\big( ( \log (n) / n)^{2 \beta / (d+2\beta)} \big)$ bound on the estimation risk of Theorem~\ref{thm:minimax-adaptive}.
\end{proof}

\begin{proof}[Proposition~\ref{prop:mondrian_depth_bound}]
  First, we reason conditionally on the realization of an infinite Mondrian partition $\mondrian$, by considering the randomness with respect to the sampling of the feature points $x_1, \dots, x_n, x$.
  For every depth $j \geq 0$, denote by $N_j$ the number of points among $x_1, \dots, x_n$ that belong to the cell of depth $j$ of the unrestricted Mondrian partition containing $x$, and $\node_j \in \{ 0, 1\}^j$ the corresponding node.
  In addition, for every $\node \in \{ 0, 1\}^*$, denote $p_\node = \P ( x \in C_{\node 0} \cond x \in C_\node )$.  
  In addition, for $j \geq 0$, since conditionally on $C_{\node_j}, N_j$, the points $x$ and $\{ x_i : x_i \in C_{\node_j} \}$ are distributed \iid following the conditional distribution of $x$ given $\{ x \in C_{\node_j} \}$:
  \begin{align}    
    \E [N_{j+1} \cond C_{\node_j}, N_j, \mondrian] 
    &= \P (\node_{j+1} = \node_j 0 \cond \node_j) \times N_j \P (x_1 \in C_{\node_j 0} \cond x_1 \in C_{\node_j}) \nonumber \\
      &\quad + \P (\node_{j+1} = \node_j 1 \cond \node_j) \times N_j \P (x_1 \in C_{\node_j 1} \cond x_1 \in C_{\node_j}) \nonumber \\
    &= N_j \big( p_{\node_j}^2 + (1 - p_{\node_j})^2 \big) \nonumber \\
    &= N_j \big( 1 - 2 p_{\node_j} (1 - p_{\node_j}) \big) \label{eq:proof-depth-mondrian-rec}
      \, .
  \end{align}
  Now, note that $p_{\node_j} (1 - p_{\node_j})$ is determined by $C_{\node_j}$ and its split in $\mondrian$, while $N_j$ is determined by $C_{\node_j}$ and $x_1, \dots, x_n$.
  % Hence, by induction on $k \geq 0$, using the fact that by definition $N_0 = n$,
  % \begin{equation}
  %   \label{eq:proof-depth-mondrian-1}
  %   \E [N_k \cond \mondrian]
  %   = n \prod_{j=0}^k
  % \end{equation}
  Now, let $U_j$ be the ratio of the volume of $C_{\node_{j} 0}$ by that of $C_{\node_j}$; by construction of the Mondrian process, $U_j \sim \uniformdist ([0, 1])$ conditionally on $C_{\node_j}$.
  In addition, the assumption~\eqref{eq:ratio-density} implies (by integrating over the coordinate of the split, fixing the other coordinates) that $p_{\node_j} \geq M^{-1} U_j$, $1 - p_{\node_j} \geq M^{-1} (1 - U_j)$.
  It follows that
  \begin{equation*}
    p_{\node_j} (1 - p_{\node_j})
    \geq \frac{1}{2} \{ p_{\node_j} \wedge (1 - p_{\node_j}) \}
    \geq \frac{1}{2 M} \{ U_j \wedge (1 - U_j) \}
  \end{equation*}
  so that
  \begin{equation*}
    \E [ p_{\node_j} (1 - p_{\node_j}) \cond C_{\node_j}]
    \geq \frac{1}{2 M} \E [ U_j \wedge (1 - U_j) \cond C_{\node_j} ]
    = \frac{1}{4 M}
    \, .
  \end{equation*}
  Using the fact that $N_j$ and $p_{\node_j}$ are independent conditionally on $C_{\node_j}$, it follows from~\eqref{eq:proof-depth-mondrian-rec} that 
  %%% the former being determined by $x_1, \dots, x_n$, the latter by the next split
  \begin{equation*}
    \E [N_{j+1} \cond C_{\node_j}]
    = \E [N_j \cond C_{\node_j}] \big( 1 - 2 \E [ p_{\node_j} (1 - p_{\node_j})] \big)
    \leq \Big( 1 - \frac{1}{2M} \Big) \E [N_j \cond C_{\node_j}] .
  \end{equation*}
  By induction on $k \geq 0$, using the fact that by definition $N_0 = n$,
  \begin{equation}
    \label{eq:proof-depth-mondrian-1}
    \E [N_k ]
    \leq n \Big( 1 - \frac{1}{2M} \Big)^k
    \, .
  \end{equation}
  Now, note that if $N_k = 0$, then the depth $D_n^\mondrian(x)$ of $x$ in the Mondrian partition $\mondrian$ restricted to $x_1, \dots, x_n,x$ is at most $k$.
  % the depth $D_n^\mondrian(x)$ of $x$ in the Mondrian partition $\mondrian$ restricted to $x_1, \dots, x_n,x$ is smaller than the smallest depth $k \geq 0$ at which the cell of $x$ contains no point among $x_1, \dots, x_n$, \ie $N_k = 0$.
  Thus, inequality~\eqref{eq:proof-depth-mondrian-1} implies:
  \begin{align}
    \E [ D_n^\mondrian (x) ]
    &= \sum_{k \geq 1} \P (D_n^\mondrian (x) \geq k) \nonumber \\
    &\leq \sum_{k \geq 1} \P (N_k \geq 1) \nonumber \\
    &\leq \sum_{k \geq 1} \E [N_k ] \wedge 1 \nonumber \\
    &\leq \sum_{k \geq 1} \left\{ n \left( 1 - \frac{1}{2M} \right)^k \right\} \wedge 1
      \label{eq:proof-depth-mondrian-2}
    % \\
  \end{align}
  Now, let $k_0$ be the smallest $k \geq 1$ such that $n ( 1 - {1}/(2M) )^{k_0} \leq 1$.
  We have 
  \begin{equation*}
    k_0 = \Big\lceil \frac{\log n}{\log \{ (2M)/(2M - 1) \} } \Big\rceil,
  \end{equation*}
  so that $k_0 - 1 \leq \log (n) / \log \{ (2M)/(2M - 1) \}$.
  Hence, inequality~\eqref{eq:proof-depth-mondrian-2} becomes:
  \begin{align*}
    \E [ D_n^\mondrian (x) ]
    &\leq (k_0 - 1) + \sum_{k \geq 0} \underbrace{n \Big( 1 - \frac{1}{2M} \Big)^{k_0}}_{\leq 1} \Big( 1 - \frac{1}{2M} \Big)^k \\
    &\leq \frac{\log n}{\log [ (2M)/(2M-1) ]} + 2M
  \end{align*}
  which establishes Proposition~\ref{prop:mondrian_depth_bound}.
\end{proof}

% \bibliography{../biblio-journal}

\end{document}